\documentclass{article}

\usepackage{natbib}
\setcitestyle{numbers, compress}

\usepackage{amsmath}
\usepackage{amssymb}
\usepackage{mathtools}
\usepackage{amsthm}
\theoremstyle{plain}
\newtheorem{theorem}{Theorem}[section]

\newtheorem{lemma}[theorem]{Lemma}
\newtheorem{corollary}[theorem]{Corollary}
\theoremstyle{definition}

\newtheorem{assumption}[theorem]{Assumption}
\newtheorem{remark}[theorem]{Remark}
\newtheorem*{main result 1}{Main Theorem 1}
\newtheorem*{main result 2}{Main Theorem 2}
\allowdisplaybreaks[4]

    \usepackage[final]{neurips_2022}

\usepackage[utf8]{inputenc} 
\usepackage[T1]{fontenc}    
\usepackage{hyperref}       
\usepackage{url}            
\usepackage{booktabs}       
\usepackage{amsfonts}       
\usepackage{nicefrac}       
\usepackage{microtype}      
\usepackage{xcolor}         

\hypersetup{colorlinks=true,linkcolor=blue,linktocpage=false,citebordercolor=blue,citecolor=blue,anchorcolor=blue}

\usepackage{subfigure}
\usepackage{etoc}

\title{Early Stage Convergence and Global Convergence of Training Mildly Parameterized Neural Networks}

\author{%
  Mingze Wang \\
  School of Mathematical Sciences \\
  Peking University\\
  Beijing, 100081, P.R. China \\
  \texttt{mingzewang@stu.pku.edu.cn} \\
  \And
  Chao Ma\\
  Department of Mathematics \\
  Stanford University\\
  Stanford, CA 94305 \\
  \texttt{chaoma@stanford.edu} \\
}

\begin{document}

\maketitle

\begin{abstract}
The convergence of GD and SGD when training mildly parameterized neural networks starting from random initialization is studied. For a broad range of models and loss functions, including the most commonly used square loss and cross entropy loss, we prove an ``early stage convergence'' result. We show that the loss is decreased by a significant amount in the early stage of the training, and this decrease is fast. Furthurmore, for exponential type loss functions, and under some assumptions on the training data, we show global convergence of GD. Instead of relying on extreme over-parameterization, our study is based on a microscopic analysis of the activation patterns for the neurons, which helps us derive more powerful lower bounds for the gradient. The results on activation patterns, which we call ``neuron partition'', help build intuitions for understanding the behavior of neural networks' training dynamics, and may be of independent interest. 
\end{abstract}

\section{Introduction}
As deep learning shows its capability in various fields of applications, extensive researches are done to theoretically understand and explain its great success. Among the topics covered by these theoretical studies, the optimization of deep neural networks is one of the most crucial one, especially given the fact that simple optimization algorithms such as Gradient Descent (GD) and Stochastic Gradient Descent (SGD) can easily achieve zero training loss \citep{zhang2017understanding}, although the loss landscape of training neural networks is highly non-convex \citep{li2018visualizing, sun2020global}. 
Existing works attempting to answer the surprising convergence ability usually work in settings that do not align well with realistic practices. For example, the notable Neural Tangent Kernel (NTK) theory~\citep{jacot2018neural, du2018gradient, du2019gradient, allen2019convergence, zou2018stochastic} shows the global convergence of SGD in highly over-parameterized regime, in which the loss landscape is approximated by a quadratic function. In practice, though, the loss landscape is indeed highly non-convex with spurious local minima and saddle points~\citep{safran2018spurious, safran2021effects}, and the dynamics of neurons show nonlinear behaviors \citep{ma2020quenching}.


Nevertheless, fast decreasing of the loss value always happens when the network trained is not highly over-parameterized, at least in the early stage of training. It is common that the loss experiences a drastic decreasing at the beginning of the training. In many cases, this decrease of loss even continues until the loss achieves zero, i.e. the optimization algorithm fully converges. In this paper, we study the early stage or full convergence of GD and SGD when under-parameterized or mildly over-parameterized models are trained. Specifically, we answer the following two theoretical questions:
\begin{gather*}
\textit{When we train practical-size neural networks by GD or SGD,}
\\
\textit{1. Does the fast convergence in the early stage of the training provably exist? If so, how}
\\\textit{long will the phenomenon last and how much will loss descend in the early stage?}
\\\textit{2. Can the global convergence be proved under some special conditions}
\\\textit{on loss function and training data?}
\end{gather*}

Our answers to the questions above are roughly summarized in the following main theorems:

\begin{main result 1}[Informal statement of Theorem \ref{thm: binary quadratic} and \ref{thm: one-hot}]\textit{For mildly over-parameterized or under-parameterized two-layer neural networks with quadratic loss or general classification loss (Assumption \ref{def: general loss}), let parameters be trained by GD or SGD started with random initialization.
If the learning rate $\eta\leq 0.01$, then
in the first $\Theta(1/\eta)$ iterations, the loss will descend $\Omega(1)$.}
\end{main result 1}

\begin{main result 2}[Informal statement of Theorem \ref{thm: binary global GD} and \ref{thm: binary global GD linear}]\textit{For mildly over-parameterized or under-parameterized two-layer neural networks with exponential-type loss (Assumption \ref{def: exponential type loss}), let parameters be trained by GD with random initialization and proper learning rate. If data is well-separated (satisfies Assumption \ref{ass: data separation}), the loss converges to $0$ at exponential rate or arbitrary polynomial rate, depending on the conditions.}
\end{main result 2}

In the first main theorem, we demonstrate that the fast convergence in the early stage of training happens under weak conditions. The model does not need to be highly over-parameterized, and the loss function can take quite general forms.
In the second main theorem, we further prove the full convergence of GD for exponential-type loss and well-separated dataset. These assumptions on the loss function and training data are close to practice and widely picked by previous theoretical studies~~\citep{ji2018gradient, ji2020directional, phuong2020inductive, chatterji2021doesA, chatterji2021doesB, lyu2019gradient, lyu2021gradient}. ``Full convergence'' means that our analysis covers all the stages of the training process, starting from initialization to the convergence to zero loss. This is different from existing works which focus on the convergence process after the training accuracy hits 100\%~\citep{lyu2019gradient, ji2020directional, chatterji2021doesA}.
Moreover, we provide the convergence rate of GD for two classes of neural networks.

During the theoretical analysis, we study the dynamics of neurons in detail and capture the effect of each sample on each neuron. We call our results ``neuron partition'', which provides a accurate description to the behavior of neurons. With the neuron partition results, we derive a novel gradient lower bound for mildly parameterized neural networks, which is vital for the convergence results. 
The neuron partition also provides rich intuition for how the convergence happens. When GD or SGD is used with small initialization, the network is initially close to the saddle point at $0$, which might be hard for GD to escape. However, neurons will adjust their directions rapidly and enter a good region which contains neither spurious local minima nor saddle points. Then, neurons will keep moving through the right directions for $\Theta(1/\eta)$ iterations, during which the loss descends quickly and significantly. For exponential-type loss and well-separated training data, the second stage will proceed uninterruptedly until training is stopped.



\section{Related Work}
Due to the highly non-convex loss landscape \citep{li2018visualizing, sun2020global}, classical optimization theories such as convex optimization fail to characterize the convergence of training neural networks. Researchers have proposed theories specifically for neural networks. We list some below.

A popular line of works focus on the highly over-parameterized neural networks and derive the ``Neural tanget kernel (NTK)'' theory. In the NTK regime \citep{jacot2018neural, du2018gradient, du2019gradient, allen2019learning, allen2019convergence, arora2019fine, arora2019exact, daniely2017sgd, zou2018stochastic, li2018learning, chizat2019lazy, weinan2019analysis, weinan2020comparative}, the neural network model is close to a kernel method, leading to the nearly convex optimization landscape. Global convergence is proven in this regime. Besides, for highly over-parameterized neural networks, the mean-field approach \citep{chizat2018global, mei2018mean, mei2019mean} is another line, which analyze the training dynamics by Wasserstein gradient flow.

However,
the NTK regime is different from real practical neural networks in several aspects. 
First, practical neural networks are often mildly over-parameterized rather than highly over-parameterized \citep{livni2014computational}. Second, the landscape for practical neural networks is usually more complicated, containing local minima and saddle points \citep{safran2018spurious, safran2021effects, ding2019sub}.
Finally, the empirical superiority of neural networks over kernel methods is obvious. For example, neural networks can learn single neuron efficiently \citep{yehudai2020learning, wu2022learning}, while kernel methods fail unless the network size is exponentially large with respect to the input dimension \citep{yehudai2019power}. 

Furthermore, optimization theories for neural networks beyond the NTK regime has also been studied. \cite{safran2018spurious, ding2019sub} proved that spurious local minima are common in the loss landscape. \cite{safran2021effects} pointed out that neither one-point convexity nor Polyak-Łojasiewicz condition holds near the global minimum. In \citep{lyu2019gradient, ji2020directional, chatterji2021doesA, chatterji2021doesB, zhou2021local}, local convergence results are given with cross-entropy loss or some special networks at the late stage of training, i.e. when training loss is small enough. In \citep{li2020learning, phuong2020inductive, lyu2021gradient, jentzen2021proof, cheridito2022proof}, global convergence results are given for some special networks, special data distribution or special target functions. 

\section{Preliminaries}
\label{section: pre}

\subsection{Notations}
We use bold letters for vectors or matrices and lowercase letters for scalars, e.g. $\boldsymbol{x}=(x_1,\cdots,x_d)^\top\in\mathbb{R}^d$ and $\mathbf{P}=(P_{i j})_{m_1\times m_2}\in\mathbb{R}^{m_1\times m_2}$.
We use $\left<\cdot,\cdot\right>$ for the standard Euclidean inner product between
two vectors, and $\left\|\cdot\right\|$ for the $l_2$ norm of a vector or the spectral norm of a matrix. 
We use $a\lesssim b$ to indicate that there exists an absolute constant $c>0$ such that $a\leq cb$, and $a\gtrsim b$ is similarly defined.
We use standard progressive representation $\mathcal{O},\Omega, \Theta$ to hide absolute constants.
For any positive integer $n$, let $[n]=\{1,\cdots,n\}$. Denote by $\mathcal{N}(\boldsymbol{\mu},\mathbf{\Sigma})$ the Gaussian distribution with mean $\boldsymbol{\mu}$ and covariance matrix $\mathbf{\Sigma}$, $\mathbb{U}(S)$ the uniform distribution on a set $S$. Denote by $\mathbb{I}\{E\}$ the indicator function for an event $E$. For a square matrix $P$, we use $\lambda_{\min}(\mathbf{P})$ to denote its smallest singular value.

\subsection{Problem settings}
In this paper, we consider supervised learning problems. Let $\mathcal{P}$ be a data distribution on $\mathbb{R}^d\times\mathbb{R}^C$.
We are given $n$ training data $\mathcal{S}=\{(\boldsymbol{x}_i,\boldsymbol{y}_i)\}_{i=1}^n\subset\mathbb{R}^d\times\mathbb{R}^C$ drawn i.i.d. from $\mathcal{P}$. Without loss of generality, we assume $\left\|\boldsymbol{x}\right\|\leq1$ for any $(\boldsymbol{x},\boldsymbol{y})$ sampled from $\mathcal{P}$.

We consider the empirical risk minimization (ERM) problem, which tries to minimize the empirical risk $\mathcal{L}(\cdot)$ with loss function $\ell(\cdot,\cdot)$:
\begin{equation}\label{equ: problem ERM}
        \mathcal{L}(\boldsymbol{\theta})=\frac{1}{n}\sum_{i=1}^n \ell(\boldsymbol{y}_i,\boldsymbol{f}(\boldsymbol{x}_i;\boldsymbol{\theta})),
\end{equation}
where $\boldsymbol{f}(\boldsymbol{x};\boldsymbol{\theta})\in\mathbb{R}^C$ is the model and $\boldsymbol{\theta}$ represents all parameters of the model. We will give specific forms to the loss function and the model in the analysis in later sections.

We use Gradient Descent (GD) or Stochastic Gradient Descent (SGD) starting from random initialization to solve the ERM problem. The update rules of GD or batch SGD can be written as:
\begin{equation}\label{equ: alg GD}
    {\rm \textbf{GD}}:\quad\boldsymbol{\theta}(t+1)=\boldsymbol{\theta}(t)-\eta_t\nabla \mathcal{L}(\boldsymbol{\theta}(t)),
\end{equation}
\begin{equation}\label{equ: alg SGD}
    {\rm \textbf{SGD}}:\quad\boldsymbol{\theta}(t+1)=\boldsymbol{\theta}(t)-\frac{\eta_t}{B}\sum_{i\in\mathcal{B}_t}\nabla\ell(\boldsymbol{y}_i,f(\boldsymbol{x}_i;\boldsymbol{\theta}(t))),
\end{equation}
where $\mathcal{B}_t=\{\gamma_{t,1},\cdots,\gamma_{t,B}\}$ is a batch, and $\gamma_{t,1},\cdots,\gamma_{t,B}\overset{\text{i.i.d.}}{\sim}\mathbb{U}([n])$ and are independent with $\boldsymbol{\theta}(t)$. The random initialization will be specified in each theorem.

\section{Early Stage Convergence}
\label{section: early stage}

In this section, we state and discuss our main results on the fast convergence in the early stage of training. The models we focus on are mildly over-parameterized or under-parameterized neural networks. As a warm up, in Section \ref{section: thm1}, we study a simple case with binary classification problems and quadratic loss. Then, in Section \ref{section: thm2}, we extend our results to multi-class classification problems with general loss and one-hot labels. We provide discussions to our results in Section~\ref{section: main results discuss}.

\subsection{Binary classification with quadratic loss}\label{section: thm1}

In this subsection, we study the ERM problem (\ref{equ: problem ERM}) with quadratic loss $\ell(y_1,y_2)=\frac{1}{2}(y_1-y_2)^2$ and the two-layer ReLU neural network model without bias:
\begin{equation}\label{equ: model 2NN nobias}
f(\boldsymbol{x};\boldsymbol{\theta})=\sum_{k=1}^m a_k\sigma(\boldsymbol{b}_k^\top\boldsymbol{x}),
\end{equation}
where $\sigma(z)=\text{ReLU}(z)$, $\boldsymbol{\theta}=(a_1,\cdots,a_m,\boldsymbol{b}_1^\top,\cdots,\boldsymbol{b}_m^\top)^\top\in\mathbb{R}^{d m+m}$. We consider the following random initialization
\begin{equation}\label{equ: random initialization}
    \begin{gathered}
        a_k(0)\overset{\text{i.i.d.}}{\sim} \mathbb{U}(\pm{1}/{\sqrt{m}}),\quad \boldsymbol{b}_k(0)\overset{\text{i.i.d.}}{\sim}\mathcal{N}(\mathbf{0},\frac{\kappa^2}{md}\mathbf{I}_d),\quad\text{for\ } k\in[n],
    \end{gathered}
\end{equation}
where $\kappa$ is a constant that controls the initial scale of the first layer.

We focus on the training data with two classes with good separation given by the following assumption.

\begin{assumption}\label{ass: data separation}
(i)
$n$ is even. $y_i=1$ for $i\in[n/2]$; $y_i=-1$ for $i\in[n]-[n/2]$. $\boldsymbol{x}_i^\top\boldsymbol{x}_j\geq 0$  for $i,j$ in the same class; $\boldsymbol{x}_i^\top\boldsymbol{x}_j\leq 0$  for $i,j$ in different classes. 
(ii) There exists a constant $\mu_0>0$, s.t. $\min\limits_{\{i\in[n],\boldsymbol{v}\in\mathbb{R}^d|\boldsymbol{v}^\top\boldsymbol{x}_i\leq0\}}\max\limits_{\{j\in[n]|\boldsymbol{v}^\top\boldsymbol{x}_j>0\}}y_i\boldsymbol{x}_i^\top\boldsymbol{x}_j y_j\geq\mu_0$.
\end{assumption}

For Assumption \ref{ass: data separation} (i), similar assumption has been used in prior works~\citep{phuong2020inductive}.
For image classification problems, this assumption can be ensured by a simple transformation on the data due to the non-negative pixels of images. Specifically, given an image dataset consisted of two classes, we just need to replace $\boldsymbol{x}$ with $-\boldsymbol{x}$ for any $(\boldsymbol{x},y)$ in the second class. Assumption 4.1 (ii) is a technical assumption used in the analysis. It is not a strong addition over Assumption 4.1 (i). For example, if there exists a $(x_0,y_0)$ in the dataset $S$ such that $(-x_0,-y_0)$ is also in $S$, then the dataset $S$ satisfies Assumption 4.1 (ii) with $\mu_0=1$ (the upper bound of $\mu_0$ is $1$). Admittedly, this technical assumptions is a limitation of our theory---it restricts its applicability. We will search for the relaxation of this assumption in future works. 


Now we state the following result.
\begin{theorem}\label{thm: binary quadratic}
Suppose Assumption \ref{ass: data separation} holds. Let $\boldsymbol{\theta}(t)$ be the parameters of model \eqref{equ: model 2NN nobias} trained by Gradient Descent \eqref{equ: alg GD} with quadratic loss. If the width $m={\Omega}(\log(n/\delta))$, the input dimension $d=\Omega(\log m)$, the learning rate $\eta_t=\eta\leq0.01$ and the initialization scale $\kappa=\mathcal{O}(\eta\mu_0/n)$ in \eqref{equ: random initialization}, then, with probability at least $1-\delta-2me^{-2d}$, the loss will descend ${\Omega}(1)$ in $T=\Theta(\frac{1}{\eta})$ iterations.
\end{theorem}

From our proof in Appendix \ref{appendix: proof thm1}, we have the following corollary with fixed numbers about $\eta,\kappa,m$ instead of progressive expressions.
\begin{corollary}
Suppose Assumption \ref{ass: data separation} holds. Let $\boldsymbol{\theta}(t)$ be the parameters of model \eqref{equ: model 2NN nobias} trained by Gradient Descent \eqref{equ: alg GD} with quadratic loss. Let $\eta=\boldsymbol{0.01}$, $
    \kappa\leq\min\{\eta/{2000}, {\eta\mu_0}/{3n}\}$, $m\geq\max\{ 144\log(2n^2/\delta),4\}$, and $\gamma_1,\gamma_2=\Theta(1)$ be constants related with the data distribution defined in \eqref{equ: gamma definition} in Appendix \ref{appendix: proof thm1}. Then, with probability at least $1-\delta-2me^{-2d}$, in $T=\boldsymbol{44}$ iterations, loss will descend at least $0.193(\gamma_1-\gamma_2\sqrt{8\log(n^2/\delta)/m})- 0.0111=\Omega(1-\sqrt{\log(n/\delta)/m})$.
\end{corollary}

Conclusions in Theorem~\ref{thm: binary quadratic} show the meaning of {\bf ``mildly parameterized''} in our title. The network does not need to be over-parameterized, i.e. the number of parameters may be smaller than the number of training data. Hence, we pick the term ``mildly parameterized'' instead of the more widely use ``mildly over-parameterized''. The former is weaker than the latter.


\subsection{Multi-class classification with general loss and one-hot labels}\label{section: thm2}

In this subsection, we extend the early stage convergence results above to more practical settings. Specifically, we consider classification problem with one-hot labels and the general loss functions that satisfy the following conditions. 

\begin{assumption}\label{def: general loss}
The loss function
$\ell(\cdot,\cdot)$ can be expressed as  $\ell(\boldsymbol{y}_1,\boldsymbol{y}_2)=\tilde{\ell}(\boldsymbol{y}_1^\top\boldsymbol{y}_2)$ such that:
(i) $\tilde{\ell}(\cdot)$ is twice differentiable in $\mathbb{R}$;
(ii) $\tilde{\ell}(\cdot)$ is non-negative and non-increasing in $\mathbb{R}$;
(iii) there exist $z_0\in(0,1]$, $g_{\min},g_{\max}>0,h_{\max}\geq0$ such that $g_{\min}\leq-\tilde{\ell}'(z)\leq g_{\max}$ and $0\leq\tilde{\ell}''(z)\leq h_{\max}$ hold for any $z\in[0,z_0]$.
\end{assumption}

It is easy to verify that exponential loss $\ell(\boldsymbol{y}_1,\boldsymbol{y}_2)=e^{-\boldsymbol{y}_1^\top\boldsymbol{y}_2}$ ($z_0=1$, $g_{\min}=\frac{1}{e}$, $g_{\max}=1$, $h_{\max}=1$), logistic loss $\ell(\boldsymbol{y}_1,\boldsymbol{y}_2)=\log(1+e^{-\boldsymbol{y}_1^\top\boldsymbol{y}_2})$ ($z_0=1$, $g_{\min}=\frac{1}{e+1}$, $g_{\max}=\frac{1}{2}$, $h_{\max}=\frac{1}{4}$) and hinge loss $\ell(\boldsymbol{y}_1,\boldsymbol{y}_2)=\max\{0, 1-\boldsymbol{y}_1^\top\boldsymbol{y}_2\}$ ($z_0=1$, $g_{\min}=1$, $g_{\max}=1$, $h_{\max}=0$) all satisfy the conditions in Assumption \ref{def: general loss}.

On the training data, we impose the following assumption, which is milder than Assumption~\ref{ass: data separation}.
\begin{assumption}\label{ass: data concentration} 
There exists $s>-1$ s.t. ${\left<\boldsymbol{x}_i,\boldsymbol{x}_j\right>}\geq s$ for any $i,j\in[n]$.
\end{assumption}

Note that now we do not need any separability of training data, but only require there is no pair of data going into opposite directions. 
(We can even normalize the data and let $\left\|\boldsymbol{x}\right\|<1-\epsilon$ for some $\epsilon>0$, rather than $\left\|\boldsymbol{x}\right\|\leq1$. Then, the above assumption naturally holds.)
Notably, Assumption \ref{ass: data concentration} is more widely applicable than Assumption \ref{ass: data separation} above. It holds for normalized image datasets such as MNIST~\citep{lecun1998gradient} and CIFAR-10~\citep{krizhevsky2009learning}.

For the model, we use the two-layer ReLU neural network model with bias as our prediction model:
\begin{equation}\label{equ: model 2NN bias}
\boldsymbol{f}(\boldsymbol{x};\boldsymbol{\theta})=\sum_{k=1}^m \boldsymbol{a}_k\sigma(\boldsymbol{b}_k^\top\boldsymbol{x}+c_k),
\end{equation}
where $\sigma(z)=\text{ReLU}(z)$, $\boldsymbol{a}_k\in\mathbb{R}^C$ and
$\boldsymbol{\theta}=(\boldsymbol{a}_1^\top,\cdots,\boldsymbol{a}_m^\top,\boldsymbol{b}_1^\top,\cdots,\boldsymbol{b}_m^\top,c_1,\cdots,c_m)^\top\in\mathbb{R}^{(d+C+1)m}$. 
We consider random initialization $\boldsymbol{b}_k(0)\overset{\text{i.i.d.}}{\sim}\mathcal{N}(\mathbf{0},\frac{\kappa^2}{m(d+1)}\mathbf{I}_d)$, $c_k(0)=\frac{\kappa}{\sqrt{m(d+1)}}$, and $\boldsymbol{a}_k(0)=(1/\sqrt{m},\cdots,1/\sqrt{m})^\top$ for $k\in[m]$, where $\kappa$ is again a constant that controls the initial scale of the first layer. 

We show the following convergence result:
\begin{theorem}\label{thm: one-hot}
Under Assumption \ref{def: general loss} and \ref{ass: data concentration}, let $\boldsymbol{\theta}(t)$ be the parameters of model \eqref{equ: model 2NN bias} trained by Stochastic Gradient Descent \eqref{equ: alg SGD}. If the width $m={\Omega}(1)$, the input dimension $d=\Omega(\log m)$, the batch size $B=\Omega(\log m)$, the learning rate $\eta_t=\eta\leq 0.01$ and the initialization scale $\kappa=\mathcal{O}(\eta/B)$, then with probability at least $1-\delta-\mathcal{O}(me^{-d})-\mathcal{O}(m0.17^{B})$, the loss will descend ${\Omega}(1)$ in $T=\Theta(\frac{1}{\eta})$ iterations.
\end{theorem}

By our analysis in Appendix \ref{appendix: proof thm2}, we have the following corollary which gives specific numbers about $\eta,\kappa,m$ instead of progressive expression.

\begin{corollary}
Under the same assumptions as Theorem~\ref{thm: one-hot}, let $\boldsymbol{\theta}(t)$ be the parameters of the model \eqref{equ: model 2NN bias} trained by SGD \eqref{equ: alg SGD}. Let $\eta=\boldsymbol{0.01}$, $\kappa\leq\min\{\eta/10,\eta/3B\}$, $m\geq6$; $z_0=1$, $g_{\min}=\frac{1}{2}$, $g_{\max}=1$, $h_{\max}=1$ (in Assumption \ref{def: general loss}) and $s=0$ (in Assumption \ref{ass: data concentration}). Then, with probability at least $1-\delta-4me^{-\frac{d+1}{2}}-m0.17^B$, in $T=\boldsymbol{34}$ iterations, loss will descend at least $0.262533=\Omega(1)$.
\end{corollary}

We remark here that the weaker Assumption \ref{ass: data concentration} in this subsection compared with the Assumption \ref{ass: data separation} before is made possible by the one-hot labels. For one-hot labels, each label component is non-negative. Hence, neural networks can quickly learn a positive bias, allowing the loss to drop significantly.


\subsection{Discussion}
\label{section: main results discuss}

A large number of classical convergence analysis focus on the late stage of training. For neural networks' optimization, many works focus on the period of training after the loss is smaller enough~\citep{lyu2019gradient, ji2020directional, chatterji2021doesA, zhou2021local}. However, the early stage is at least as important as the late stage, especially for non-convex problems. 

The theoretical results we present in this section demonstrate the fast convergence during the early stage of optimization, i.e. the (short) period of time after the initialization. The two main results, Theorem \ref{thm: binary quadratic} and \ref{thm: one-hot}, imply that if we train two-layer neural networks by GD or SGD with learning rate $\eta\leq0.01$, then the loss will descend significantly ($\Omega(1)$) in the first $T=\Theta(1/\eta)$ iterations.

Our results work under realistic settings.
For models, our results hold for mildly over-parameterized and under-parameterized neural networks,  because we only need the width $m=\Omega(\log (n/\delta))$. For loss functions, our results apply for practical losses, such as quadratic loss and cross entropy loss. On the algorithm side, the results work for both GD (\ref{equ: alg GD}) and SGD (\ref{equ: alg SGD}).

Though both concerning convergence of GD and SGD, our results are essentially different from the NTK theory, especially on the requirement of over-parameterization. Our results only need $m=\Omega(\log(n/\delta))$, while in the NTK analysis usually assumes that $m$ is polynomial in $n,1/\kappa$ or etc, excluding the practical settings. 
Moreover, from the analysis side, we establish fine-grained analysis of the training dynamics of each neuron, which may be of independent interested. See Section \ref{section: proof sketch} for a summary of our techniques, and Appendix \ref{appendix: proof thm1} and \ref{appendix: proof thm2} for the detailed proof.
Our analysis helps understand how the convergence in the early stage happens.
When GD or SGD is used with small initialization, the initial network is close to the saddle point at $0$. However, neurons will adjust directions rapidly, and the iterator will enter a good region which contains neither spurious local minima nor saddle points. Then, neurons will keep going towards the right directions for a period of time, during which process the loss will descend fast and significantly.

\section{Global Convergence}
\label{section: global}
When stronger conditions are imposed on the data distribution and the loss function, we can show global convergence of GD, i.e. starting from any initialization, the GD descends the loss to 0.

For loss functions, we consider the following exponential-type classification loss: 
\begin{assumption}\label{def: exponential type loss}
The loss function
$\ell(\cdot,\cdot)$ can be expressed as  $\ell(\boldsymbol{y}_1,\boldsymbol{y}_2)=\tilde{\ell}(\boldsymbol{y}_1^\top\boldsymbol{y}_2)$ such that:
(i) $\tilde{\ell}(\cdot)$ is twice continuously differentiable in $\mathbb{R}$;
(ii) $\tilde{\ell}(\cdot)$ is positive and non-increasing in $\mathbb{R}$;
(iii) there exist $g_b>0,h\geq0$ s.t. $-\frac{\tilde{\ell}'(z)}{\tilde{\ell}(z)}\leq g_b$ and $0\leq\frac{\tilde{\ell}''(z)}{\tilde{\ell}(z)}\leq h$ for any $z\in\mathbb{R}$; (iv) there exists $g_a>0$, s.t. $g_{a}\leq-\frac{\tilde{\ell}'(z)}{\tilde{\ell}(z)}$ for any $z\geq0$.
\end{assumption}
The similar assumption on exponential-type loss has been used in prior works~\citep{lyu2019gradient}.
It is easy to verify that both exponential loss $\ell(\boldsymbol{y}_1,\boldsymbol{y}_2)=e^{-\boldsymbol{y}_1^\top\boldsymbol{y}_2}$ ($g_a=g_b=h=1$) and logistic loss $\ell(\boldsymbol{y}_1,\boldsymbol{y}_2)=\log(1+e^{-\boldsymbol{y}_1^\top\boldsymbol{y}_2})$ ($g_a=\frac{1}{2}$, $g_b=1$, $h=1$)
satisfy the conditions in Assumption \ref{def: exponential type loss}.

Let $V$ be a constant defined as 
\begin{equation}\label{equ: V}
V=\frac{1}{16}\Big(\frac{1}{2}-
\sqrt{\frac{8\log(n^2/\delta)}{m}}\Big)\max\Big\{\frac{2}{n}+\frac{n-2}{n}\gamma,\lambda_{\min}(\mathbf{X}_{+}^\top\mathbf{X}_+)\wedge\lambda_{\min}(\mathbf{X}_{-}^\top\mathbf{X}_-)\Big\},
\end{equation}
where $\mathbf{X}_+:=(\boldsymbol{x}_1,\cdots,\boldsymbol{x}_\frac{n}{2})\in\mathbb{R}^{d\times\frac{n}{2}}$, $\mathbf{X}_-:=(\boldsymbol{x}_{\frac{n}{2}+1},\cdots,\boldsymbol{x}_n)\in\mathbb{R}^{d\times\frac{n}{2}}$, and $\gamma:=\min\limits_{i,j\text{in the same class}}\boldsymbol{x}_i^\top\boldsymbol{x}_j\geq0$. We have the following convergence theorems in which $V$ controls the convergence speed.



\begin{theorem}\label{thm: binary global GD}
Under Assumption \ref{def: exponential type loss} and Assumption \ref{ass: data separation}, let $\boldsymbol{\theta}(t)$ be the parameters of model \eqref{equ: model 2NN nobias} trained by Gradient Descent \eqref{equ: alg GD} starting from random initialization~\eqref{equ: random initialization}. Let the width $m={\Omega}(\log(n/\delta))$, the initialization scale $\kappa=\mathcal{O}(\eta_0{\mu_0}/n)$ in \eqref{equ: random initialization}, $V$ be defined in \eqref{equ: V}, $c$ be a constant in $(0,\frac{1}{6(1+2\eta_0)^2+2}]$, the constant $c'>0$ be sufficiently small, the hitting time $T_0=\lceil(n\log2)^{\frac{2}{Vc}}\rceil$, the parameter $r\in [1,+\infty)$ and the learning rate satisfy
\[
\left \{ \begin{array}{ll}
\eta_0\leq1/2\sqrt{2},& \text{for}\ t=0 \\
\eta_t=\frac{c}{t\mathcal{L}(\boldsymbol{\theta}(t))},& \text{for}\ 1\leq t< T_0 \\
\eta_t=\frac{c'}{\mathcal{L}(\boldsymbol{\theta}(t))^{1-\frac{1}{2r}}}, & \text{for}\ t\geq T_0\\
\end{array}\right..
\]
Then, with probability at least $1-\delta-2me^{-2d}$, GD will converge at \textbf{polynomial} rate:
\[
\left \{ \begin{array}{ll}
\mathcal{L}(\boldsymbol{\theta}(t))\leq\frac{\mathcal{L}(\boldsymbol{\theta}(1))}{t^\frac{V c}{2}},& 1\leq t< T_0 \\
\mathcal{L}(\boldsymbol{\theta}(t))=\mathcal{O}\big(\frac{1}{t^r}\big), & t\geq T_0\\
\end{array}\right..
\]
\end{theorem}

\begin{theorem}\label{thm: binary global GD linear}
Under Assumption \ref{def: exponential type loss} and Assumption \ref{ass: data separation}, let $\{\boldsymbol{b}_k(t)\}_{k\in[m]}$ be the input layer parameters of model \eqref{equ: model 2NN nobias}, and we only train this layer by Gradient Descent \eqref{equ: alg GD} starting from random initialization~\eqref{equ: random initialization}. Let the width $m={\Omega}(\log(n/\delta))$, the initialization scale $\kappa=\mathcal{O}(\eta_0\mu_0/n)$ in \eqref{equ: random initialization}, the constant $V$ be defined in \eqref{equ: V}, the constant $c\leq\frac{1}{2}$ and the learning rate satisfy
\[\left \{ \begin{array}{ll}
\eta_0\leq 1/2\sqrt{2},& \text{for}\ t=0 \\
\eta_t=\frac{c}{\mathcal{L}(\boldsymbol{\theta}(t))},& \text{for}\ t\geq1 \\
\end{array}\right..\]

Then with probability at least $1-\delta-2me^{-2d}$, GD will converge at \textbf{exponential} rate:
\[
\mathcal{L}(\boldsymbol{\theta}(t))\leq\Big(1-\frac{V c}{2}\Big)^{t-1}\mathcal{L}(\boldsymbol{\theta}(1)),\quad t\geq1.
\]

\end{theorem}

Theorem \ref{thm: binary global GD} and \ref{thm: binary global GD linear} describe the whole training process starting from random initialization to convergence.
If all parameters in the network are trained, we obtain the global convergence with \textit{arbitrary polynomial} rate.
On the other hand, if only the input layer parameters $\{\boldsymbol{b}_k\}_{k\in[m]}$ are trained, we can obtain \textit{exponential} convergence rate. Fixing some layers of the neural network is a common practice for theoretical studies~\citep{chatterji2021doesA}.
The results show that adaptively increasing learning rate helps GD achieve faster convergence. Similar idea has been explored in previous works such as~\citep{lyu2019gradient, chatterji2021doesA}.

A similar convergence result was proven in~\citep{phuong2020inductive}. where the authors showed a global convergence result of two-layer neural networks trained by Gradient Flow (GF) with cross entropy loss and orthogonal separable data.
The main difference between our results and analysis with that in~\citep{phuong2020inductive} is that we study GD instead of GF. Due to the discretization, GD is more complicated than GF. For example, the balanced property of two layers ($\sum_{k\in[m]}|a_k(t)|^2-\sum_{k\in[m]}\left\|\boldsymbol{b}_k(t)\right\|^2\equiv0,\ \forall t\geq0$) holds for GF throughout the training, but it does not hold for GD. Dealing with GD requires new analysis techniques and finer analysis, such as the neuron partition results we derive.
Another related work is ~\citep{lakshminarayanan2020deep}. The neural partition approach in our article shares some similarity with the view of gating networks. In their work, the authors use gating networks to understand the role of depth in training deep ReLU networks. By comparison, we focus more on the characterization of training dynamics of two-layer ReLU networks. 


\section{Proof Sketch and Techniques}
\label{section: proof sketch}

\subsection{Main techniques in the proof of Theorem \ref{thm: binary quadratic}}\label{subsection: proof sketch thm1}


Inspired by the empirical work \citep{ma2020quenching}, we study the nonlinear behavior of each neuron in the early stage of training. 
First, we define the hitting time $T$, and we will study the convergence in the early stage $0\leq t\leq T$:
\begin{align*}
  T:= \sup\big\{t\in\mathbb{N}:\ 
    \max_{i\in[n]}\left|f(\boldsymbol{x}_i;\boldsymbol{\theta}(s))\right|\leq1,
    \ a_k(s)a_k(0)>0,\ \forall k\in[m],\ \forall 0\leq s\leq t+1\big\}.
\end{align*}

\subsubsection{Neuron partition: fine-grained dynamical analysis of each neuron and each sample}
During our proof, we conduct fine-grained analysis to the interaction of neurons and samples. Specifically, we characterize the impact of each sample on each neuron, and for each sample classify neurons into four categories according to their effect. We call this analysis ``neuron partition''.

For the $i$-th training data $(\boldsymbol{x}_i,y_i)$, we divide all neurons into four categories during training 0$\leq t\leq T$ and study them separately.
We define the true-living neurons $\mathcal{TL}_i(t)$, the true-dead neurons $\mathcal{TD}_i(t)$, the false-living neurons $\mathcal{FL}_i(t)$ and the false-dead neurons $\mathcal{FD}_i(t)$ at time $t$ as:
\begin{gather*}
    \mathcal{TL}_{i}(t):=\big\{k\in[m]:y_ia_k(t)>0, \boldsymbol{b}_k(t)^\top\boldsymbol{x}_i> 0\big\},
    \\
    \mathcal{TD}_{i}(t):=\big\{k\in[m]:y_ia_k(t)>0,\boldsymbol{b}_k(t)^\top\boldsymbol{x}_i\leq 0\big\},
    \\
    \mathcal{FL}_{i}(t):=\big\{k\in[m]:y_ia_k(t)<0, \boldsymbol{b}_k(t)^\top\boldsymbol{x}_i> 0\big\},
    \\ 
    \mathcal{FD}_{i}(t):=\big\{k\in[m]:y_ia_k(t)<0,\boldsymbol{b}_k(t)^\top\boldsymbol{x}_i\leq 0\big\}.
\end{gather*}
It is easy to verify
$[m]=\mathcal{TL}_{i}(t)\bigcup\mathcal{TD}_{i}(t)\bigcup\mathcal{FL}_{i}(t)\bigcup\mathcal{FD}_{i}(t)$ (see Lemma \ref{lemma: neuron partition 1}).


Under a weak assumption on the width of the network, we have the following results for the neuron partition at initialization:
\begin{lemma}[Informal statement of Lemma \ref{lemma: estimate of the number of initial neural partition 1}]\label{lemma: proof sketch: estimate of the number of initial neural partition}
If $m=\Omega(\log(n/\delta))$, then with high probability, if data $i$ and data $j$ are in the same class, we have $\text{card}\big(\mathcal{TL}_i(0)\cap\mathcal{TL}_j(0)\big)\approx\frac{\pi-\arccos(\boldsymbol{x}_i^\top\boldsymbol{x}_j)}{4\pi}m$.

\end{lemma}


The next important lemma characterizes the evolution of the neuron partition under GD dynamics.

\begin{lemma}[Informal Lemma \ref{lemma: correct neurons remain correct 1}]\label{lemma: proof sketch: correct neurons remain correct}
With high probability, for any $i\in[n]$ and $t\leq T$ we have:\\
(S1) True-living neurons remain true-living:  $\mathcal{TL}_i(t)\subset\mathcal{TL}_i(t+1)$.\\
(S2) False-dead neurons remain false-dead:  $\mathcal{FD}_i(t)\subset\mathcal{FD}_i(t+1)$.\\
(S3) True-dead neurons turn true-living in the firt step ($t=0$): $\mathcal{TD}_i(0)\subset\mathcal{TL}_i(1)$.\\
(S4) False-living neurons turn false-dead in the firt step ($t=0$): $\mathcal{FL}_i(0)\subset\mathcal{FD}_i(1)$.\\
(S5) For any $i\in[n]$, $1\leq t\leq T$, $k\in[m]$ and $\boldsymbol{b}\in\overline{\boldsymbol{b}_k(t)\boldsymbol{b}_k(t+1)}$, we have $\text{sgn}(\boldsymbol{b}^\top\boldsymbol{x}_i)\equiv\text{sgn}(\boldsymbol{b}_k^\top(1)\boldsymbol{x}_i)\ne0$.
\end{lemma}

To interpret Lemma \ref{lemma: proof sketch: correct neurons remain correct}, we display the first-step dynamics in Figure \ref{fig: proof sketch}. 
\begin{figure}[ht!]
\begin{center}
\subfigure[$t=0$.]{
\label{fig: t=0}
\includegraphics[width=0.50\textwidth]{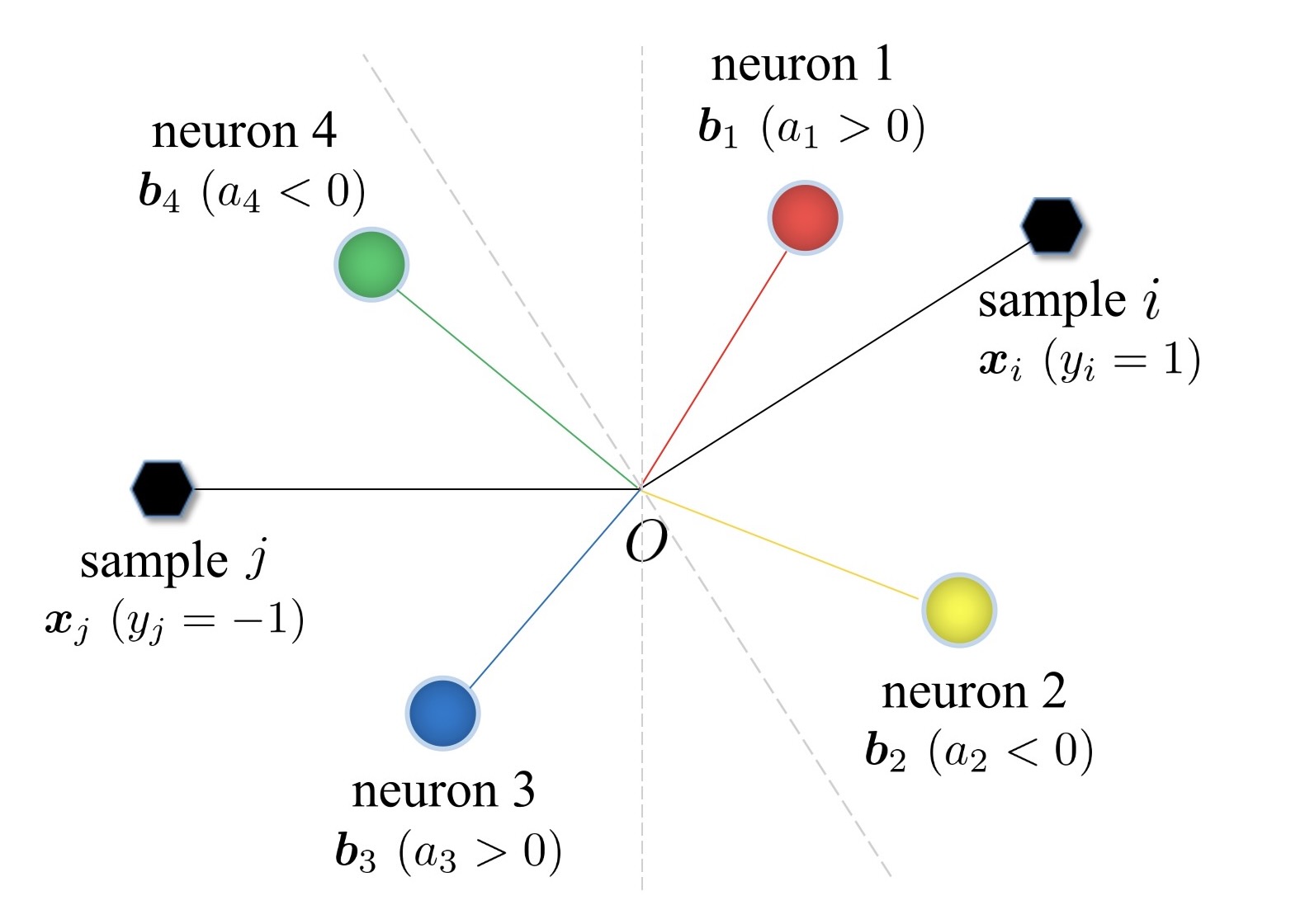}}
\subfigure[$t=1$.]{
\label{fig: t=1}
\includegraphics[width=0.48\textwidth]{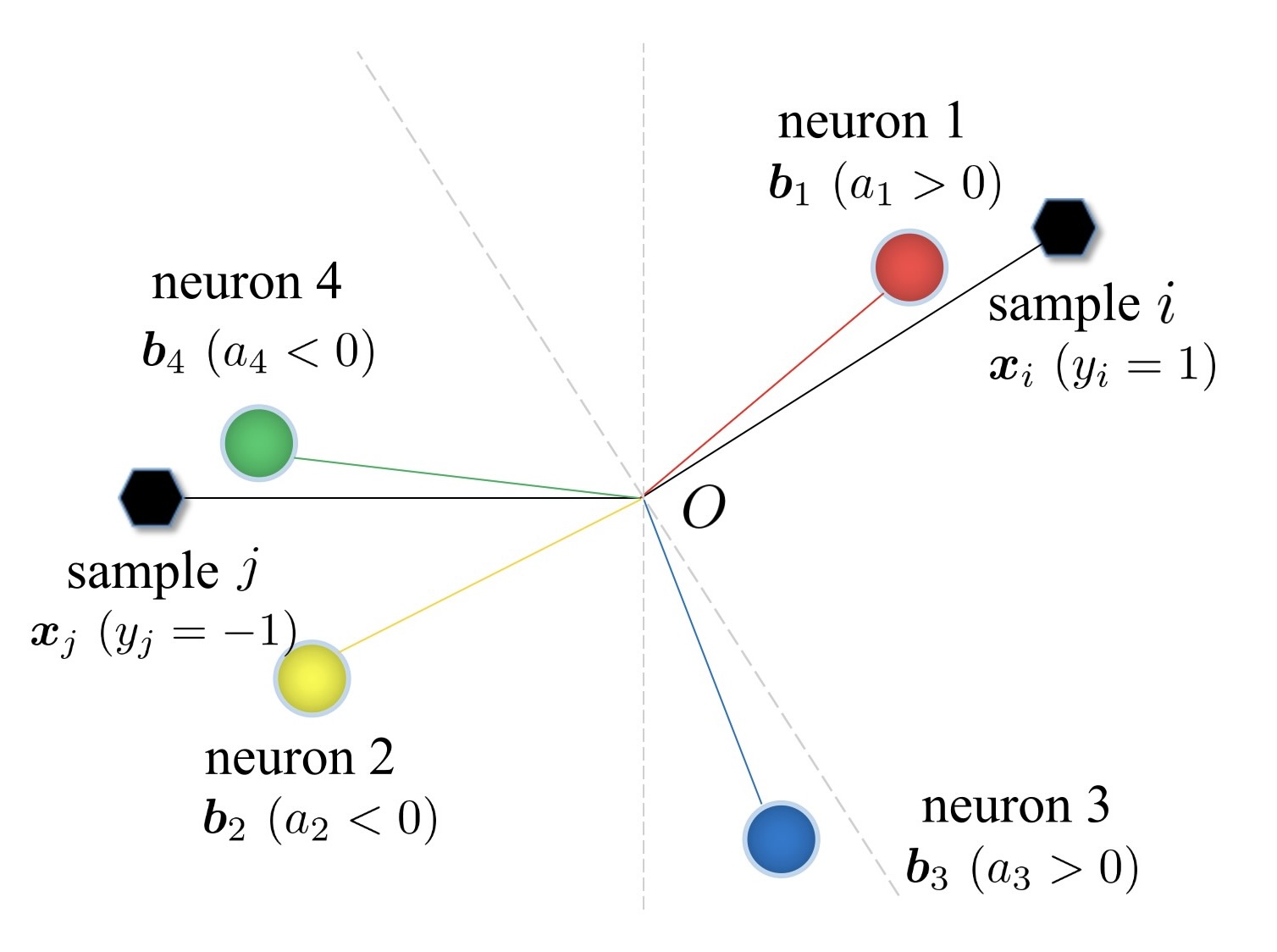}}
\end{center}
\caption{The first-step dynamics of four neurons which belong to different neuron partitions at $t=0$ for
sample $i$ (with label $1$) and sample $j$ (with label $-1$). Fig \ref{fig: t=0} and Fig \ref{fig: t=1} depict the directions of each neuron at $t=0$ and $t=1$, respectively.
Results in Lemma \ref{lemma: proof sketch: correct neurons remain correct} can be visualized: (S1) neuron $1\in\mathcal{TL}_i(0)$ and $1\in\mathcal{TL}_i(1)$. (S2) neuron $4\in\mathcal{FD}_i(0)$ and $4\in\mathcal{FD}_i(1)$. (S3) neuron $2\in\mathcal{TL}_j(0)$ and $2\in\mathcal{TL}_j(1)$. (S4) neuron $2\in\mathcal{FL}_i(0)$ and $2\in\mathcal{FD}_i(1)$. (Counterexample) By our Lemma, the first-step dynamics of neuron $3$ cannot happen.
}
\label{fig: proof sketch}
\end{figure}

\subsubsection{Estimate of hitting time and parameters}
Before the analysis of loss descent, we need to estimate the hitting time $T$ and the
speed of change for each $a_k,\boldsymbol{b}_k$ in the early stage. Our tool is comparing the hitting time $T$ with an exponentially growing hitting time $T_e$ (see \eqref{equ: exponential hitting time def 1} in Appendix \ref{appendix: proof thm1}). 
\[
  T_e:= \sup\Big\{t\in\mathbb{N}:\ (1+2\eta)^{2(t+1)}-(1-2\eta)^{2(t+1)}\lesssim1;\ \ (1+2\eta)^{t+1}\lesssim 1\Big\},
\]
where the order of magnitude of the two absolute constants hidden by $\lesssim$ is $10^0$.

Generally speaking, we can prove the following facts. 

\begin{lemma}[Informal Lemma \ref{lemma: speed separation 1} and \ref{lemma: hitting time estimate 1}]\label{lemma: proof sketch hitting time analysis} For any sample $i\in[n]$ and $t\leq T_e$, we have:
\\(P1) For any $k\in\mathcal{TL}_i(t)$, $a_k(t)\text{sgn}(a_k(0))$ is non-decreasing, and $a_k(t)$, $\boldsymbol{b}_k(t)$ have an exponential upper bound.
\\(P2) For any $k\in[m]$, $a_k(t)$ has an exponential lower bound.
\\(P3) $f\big(\boldsymbol{x}_i;\boldsymbol{\theta}(t))$ has an exponential upper bound.
\\\textbf{(P4)} $T\geq T_e\geq T^*=\Theta(\frac{1}{\eta})$.
\end{lemma}

\subsubsection{Gram matrix and gradient lower bound}
The gradient lower bound is usually an important technical component in the analysis of loss descent. 
Our analysis 
establishes a novel lower bound for the gradient for mildly parameterized neural networks based on the neuron partition.


If we define the error for sample $i$ at $t$ as $e_i(t)=f(\boldsymbol{x}_i;\boldsymbol{\theta}(t))-y_i$, and the Gram matrix at $t$ as
$\mathbf{G}(t)=(\nabla f(\boldsymbol{x}_i;\boldsymbol{\theta}(t))^\top\nabla f(\boldsymbol{x}_j;\boldsymbol{\theta}(t)))_{(i,j)\in[n]\times[n]}$. For the analysis of GD, loss descent is determined by $\left\|\nabla \mathcal{L}(\boldsymbol{\theta}(t))\right\|^2=\sum_{i=1}^n\sum_{j=1}^n e_i(t)\text{G}_{i,j}(t) e_j(t)$. Qualitatively, bigger gradient leads to the faster descent of the loss.
To derive a tight lower bound for the gradient, we first derive the following bound for the entries of the Gram matrix:

\begin{lemma}[Informal Lemma \ref{lemma: positive Gram time t 1}]\label{lemma: proof sketch: positive Gram time t} Under the condition of Theorem \ref{thm: binary quadratic}, with high probability, the Gram matrix $\mathbf{G}(t)=\begin{pmatrix}
  \mathbf{G}_{+}(t) & \mathbf{0}_{\frac{n}{2}\times\frac{n}{2}}\\\mathbf{0}_{\frac{n}{2}\times\frac{n}{2}}&\mathbf{G}_{-}(t)
\end{pmatrix}$ has the following form for any $1\leq t\leq T$ and $i,j\in[\frac{n}{2}]$:
$\text{G}_{+}(t)_{i,j},\ \text{G}_{-}(t)_{i,j}\gtrsim\boldsymbol{x}_i^\top\boldsymbol{x}_j\Big(\frac{\pi-\arccos(\boldsymbol{x}_i^\top\boldsymbol{x}_j)}{4\pi}-\mathcal{O}(\sqrt{\frac{\log(n/\delta)}{m}}\big)\Big)>0$,
where the order of magnitude of the absolute constant hidden by $\gtrsim$ is $10^0$.
\end{lemma}

Combining the estimate above as well as estimates on parameters and the prediction function, we establish our gradient lower bound.

\begin{lemma}[Informal Lemma \ref{lemma: gradient lower bound 1}]\label{lemma: proof sketch: gradient lower bound} For any $1\leq t\leq T^*$, $\left\|\nabla \mathcal{L}(\boldsymbol{\theta}(t))\right\|=\Omega\big(1-(1+2\eta)^{2t}\big)$.
\end{lemma}

Besides, since we consider GD, the following per-step estimate of loss function~\citep{bubeck2015convex} as well as the Hessian upper bound is also useful.


\begin{lemma}[Informal Lemma \ref{lemma: Hessian upper bound ERM 1} and \ref{lemma: loss quadratic upper bound 1}]\label{lemma: proof sketch: quadratic upper bound} For $1\leq t\leq T^*-1$, we have $H:=\sup\limits_{\boldsymbol{\theta}\text{ along GD trajectory}}\left\|\nabla^2 \mathcal{L}(\boldsymbol{\theta})\right\|=\mathcal{O}(1)$ and $\mathcal{L}(\boldsymbol{\theta}(t+1))\leq\mathcal{L}(\boldsymbol{\theta}(t))+\left<\nabla \mathcal{L}(\boldsymbol{\theta}(t)),\boldsymbol{\theta}(t+1)-\boldsymbol{\theta}(t)\right>+\frac{1}{2}H\left\|\boldsymbol{\theta}(t+1)-\boldsymbol{\theta}(t)\right\|^2$.
\end{lemma}

\subsection{Proof outline of Theorem \ref{thm: binary quadratic}}

First, we can estimate the initial neuron partition (Lemma \ref{lemma: proof sketch: estimate of the number of initial neural partition}) at the initialization. Then, we study the dynamics of neuron partition trained by GD (Lemma \ref{lemma: proof sketch: correct neurons remain correct}), which gives a precise directional characterization about each neuron for each sample. Next, by comparing with an exponentially growing hitting time, we can estimate the hitting time $T$ and the parameters (Lemma \ref{lemma: proof sketch hitting time analysis}).
Combining Lemma \ref{lemma: proof sketch: estimate of the number of initial neural partition}, \ref{lemma: proof sketch: correct neurons remain correct} and \ref{lemma: proof sketch hitting time analysis}, we derive a dynamical lower bound of each element of Gram matrix (Lemma \ref{lemma: proof sketch: positive Gram time t}), which induces our gradient lower bound (Lemma \ref{lemma: proof sketch: gradient lower bound}). Finally, by combining the gradient lower bound (Lemma \ref{lemma: proof sketch: gradient lower bound}) and the loss upper bound (Lemma \ref{lemma: proof sketch: quadratic upper bound}), we complete our proof. For more details, please refer to Appendix \ref{appendix: proof thm1}.

The proof of Theorem \ref{thm: one-hot} is similar to Theorem \ref{thm: binary quadratic} and put into Appendix \ref{appendix: proof thm2}. 

\subsection{Proof outline of Theorem \ref{thm: binary global GD} and \ref{thm: binary global GD linear}}

The proofs of Theorem \ref{thm: binary global GD} and \ref{thm: binary global GD linear} still rely on the dynamical analysis of neuron partition for each sample. 
First, the estimate of initial neuron partition (Lemma \ref{lemma: proof sketch: estimate of the number of initial neural partition}) also holds, and the dynamical analysis of neuron partition in Lemma \ref{lemma: proof sketch: correct neurons remain correct} holds for any $t\geq1$ (Lemma \ref{lemma: correct neurons remain correct stage I} and \ref{lemma: correct neurons remain correct stage II}).
Second, we can prove that the structure and lower bound of Gram matrix in Lemma \ref{lemma: proof sketch: positive Gram time t} hold for any $t\geq1$ (Lemma \ref{lemma: lower bound of Gram binary}), which induces the gradient lower bound below. 
Finally, combining the estimate of gradient lower bound, Hessian upper bound and loss upper bound, we complete our proof. For more details, please refer to Appendix \ref{appendix: proof global convergence} and \ref{appendix: proof global convergence linear}. 

\begin{lemma}[Informal Lemma \ref{lemma: gradient lower bound global GD}]\label{lemma: proof sketch: gradient lower bound global convergence}
Under the condition of Theorem \ref{thm: binary global GD linear}, with high probability, we have the following lower bound for any $t\geq1$: $\left\|\nabla \mathcal{L}(\boldsymbol{\theta}(t))\right\|\gtrsim\mathcal{L}(\boldsymbol{\theta}(t))$, where the absolute constant hidden by $\gtrsim$ is $\sqrt{V}$, defined in \eqref{equ: V}.
\end{lemma}

In \cite{lyu2019gradient, chatterji2021doesA}, gradient lower bounds are derived for the late stage of training when $\mathcal{L}(\boldsymbol{\theta}(t))<1/n$, and take the form $\left\|\nabla \mathcal{L}(\boldsymbol{\theta}(t))\right\|\gtrsim\mathcal{L}(\boldsymbol{\theta}(t))\log\big(1/\mathcal{L}(\boldsymbol{\theta}(t))\big)/\left\|\boldsymbol{\theta}(t)\right\|$. 
Our lower bound, instead, works on the whole training process. With this stronger gradient lower bound, we can show the exponential convergence in Theorem \ref{thm: binary global GD linear}. 

Another important technical detail is the stability of GD, i.e., the learning rate should be upper bounded by the quantity related with the top eigenvalue of the Hessian. Our analysis also considered this factor, but implicitly.
For global convergence results, the loss function we consider is exponential-type, whose Hessian can be controlled by the loss value. For example, for Theorem \ref{thm: binary global GD}, there exists an absolute constant $C>0$, s.t. $\left\|\nabla^2\mathcal{L}(\boldsymbol{\theta}(t))\right\|\leq Ct\mathcal{L}(\boldsymbol{\theta}(t))$ for any $t\geq1$, please refer to Lemma \ref{lemma: hessian upper bound GD trajectory} in the appendix. Therefore, we can transform the Hessian-controlled learning rate condition into the loss-controlled learning rate condition in Theorem \ref{thm: binary global GD}.

\section{Experiments}\label{section: exps appendix}

 {\bf MNIST and CIFAR-10 experiments.} 
 As we mentioned in Section 4.2, our theory about early stage convergence applies to a wide range of dataset (Assumption \ref{ass: data concentration}) such as MNIST and CIFAR-10. We use the two datasets (with normalization) to compare the experimental result with Theorem \ref{thm: one-hot}. Specifically, we use the first 1000 data in MNIST dataset and the first 1000 data in CIFAR-10 dataset, separately. (The two datasets with normalization $\left\|\boldsymbol{x}\right\|=1$ both satisfy Assumption \ref{ass: data concentration}.) And we use the two-layer ReLU network with the logistic loss $\ell(\boldsymbol{y}_1,\boldsymbol{y}_2)=\log(1+\exp(-\boldsymbol{y}_1^\top\boldsymbol{y}_2))$. $m,\kappa,\eta$ are choosen by Theorem \ref{thm: one-hot}. We study the change of our bounds under different network width $m$ and learning rate $\eta$. All experiments are conducted on a MacBook pro 13" only using CPU. See the code at \href{https://github.com/wmz9/Early_Stage_Convergence_NeurIPS2022}{\texttt{https://github.com/wmz9/Early\_Stage\_Convergence\_NeurIPS2022}}.

\begin{table}[ht]
\begin{center}
\label{table: MNIST}
\caption{Results of MNIST and CIFAR-10 experiments. In the first table, we fix $\eta=0.01$ and change $m$; in the second one, we fix $m=200$ and change $\eta$.}
\begin{tabular}{c|c|c|c|c}
      \hline
     $m$ & 100 & 200 & 500 & 1000 \\  \hline
     Our hitting iteration & \multicolumn{4}{c}{34}
     \\ \hline	
     Realistic loss descent (CIFAR-10) & $3.45\times10^{-1}$ & $5.34\times10^{-1}$ & $6.11\times10^{-1}$ & $6.53\times10^{-1}$ \\ \hline	
     Realistic loss descent (MINST) & $3.89\times10^{-1}$ & $4.78\times10^{-1}$ & $5.97\times10^{-1}$ & $6.45\times10^{-1}$ 
     \\ \hline	
    Our loss descent bound & \multicolumn{4}{c}{$2.63\times10^{-1}$} \\ \hline	
\end{tabular}
\begin{tabular}{c|c|c|c|c}
      \hline
     $\eta$ & 0.01 & 0.005 & 0.002 & 0.001 \\  \hline
     Our hitting iteration & 34 & 69 & 173 & 346 
     \\ \hline
     Realistic loss descent (CIFAR-10) & $4.89\times10^{-1}$ & $5.34\times10^{-1}$ & $5.43\times10^{-1}$ & $5.37\times10^{-1}$ \\\hline
     Realistic loss descent (MINST) & $4.78\times10^{-1}$ & $4.98\times10^{-1}$ & $4.95\times10^{-1}$ & $4.92\times10^{-1}$ 
     \\\hline	
    Our loss descent bound & \multicolumn{4}{c}{$2.63\times10^{-1}$} \\ \hline	
\end{tabular}
\end{center}
\end{table}
    
The results in Table \ref{table: MNIST} show that our loss descent bound (5-th row in each table) is {\bf relatively tight}, basically in the same order of magnitude as the realistic loss descent (3-rd row and 4-th row in each table). And our loss descent bound does not change with the choice of different datasets.

\section{Conclusion and Future Work}

In this paper, we study the convergence of GD and SGD when training mildly parameterized neural networks. On one hand, we show early stage convergence for a wide range of loss functions, optimization algorithms, and training data distributions. On the other hand, under some assumptions on the loss function and data distribution, we show global convergence of GD. Our analysis can be extended to the minimization of population risk (See Appendix \ref{appendix: results PRM}). 

The theoretical understanding of the training of neural networks, especially for neural networks with practical sizes, still has a long way to go. For instance, although the late stage training for exponential loss function is simple and clear, that for other losses is much more complicated. Phenomena like unstable convergence~ \citep{wu2018sgd, cohen2021gradient, ahn2022understanding, ma2022multiscale} happen. Better understanding of these phenomena during training is an important direction of future work.




\section*{Checklist}


\begin{enumerate}

\item For all authors...
\begin{enumerate}
  \item Do the main claims made in the abstract and introduction accurately reflect the paper's contributions and scope?
  \answerYes{}
  \item Did you describe the limitations of your work?
    \answerYes{The global convergence results in Section~\ref{section: global} are built on restrictive assumptions of the loss function and training data.} 
  \item Did you discuss any potential negative societal impacts of your work?
  \answerNA{}
  \item Have you read the ethics review guidelines and ensured that your paper conforms to them?
    \answerYes{}
\end{enumerate}

\item If you are including theoretical results...
\begin{enumerate}
  \item Did you state the full set of assumptions of all theoretical results?
    \answerYes{In Section~\ref{section: pre},~\ref{section: early stage} and \ref{section: global}.
    }
        \item Did you include complete proofs of all theoretical results?
          \answerYes{In Appendix~\ref{appendix: proof thm1}, \ref{appendix: proof thm2}, \ref{appendix: proof global convergence}, \ref{appendix: proof global convergence linear}, \ref{appendix: proof PRM} and \ref{app: inequalities}.}
\end{enumerate}

\item If you ran experiments...
\begin{enumerate}
  \item Did you include the code, data, and instructions needed to reproduce the main experimental results (either in the supplemental material or as a URL)?
    \answerYes{As an URL in Section \ref{section: exps appendix}.}
  \item Did you specify all the training details (e.g., data splits, hyperparameters, how they were chosen)?
    \answerYes{In Section \ref{section: exps appendix}.}
        \item Did you report error bars (e.g., with respect to the random seed after running experiments multiple times)?
    \answerNA{}
        \item Did you include the total amount of compute and the type of resources used (e.g., type of GPUs, internal cluster, or cloud provider)?
    \answerYes{In Section \ref{section: exps appendix}.}
\end{enumerate}

\item If you are using existing assets (e.g., code, data, models) or curating/releasing new assets...
\begin{enumerate}
  \item If your work uses existing assets, did you cite the creators?
    \answerNA{}
  \item Did you mention the license of the assets?
    \answerNA{}
  \item Did you include any new assets either in the supplemental material or as a URL?
    \answerNA{}
  \item Did you discuss whether and how consent was obtained from people whose data you're using/curating?
    \answerNA{}
  \item Did you discuss whether the data you are using/curating contains personally identifiable information or offensive content?
    \answerNA{}
\end{enumerate}

\item If you used crowdsourcing or conducted research with human subjects...
\begin{enumerate}
  \item Did you include the full text of instructions given to participants and screenshots, if applicable?
    \answerNA{}
  \item Did you describe any potential participant risks, with links to Institutional Review Board (IRB) approvals, if applicable?
    \answerNA{}
  \item Did you include the estimated hourly wage paid to participants and the total amount spent on participant compensation?
    \answerNA{}
\end{enumerate}

\end{enumerate}


\newpage

\appendix

\part{}

\localtableofcontents

\newpage

\section{Proof details of Theorem \ref{thm: binary quadratic}}\label{appendix: proof thm1}

\subsection{Preparation}

\noindent{\bfseries One-step updating.}\footnote{Although ReLU $\sigma(\cdot)$ is not differentiable at $0$, we just define $\sigma'(0)=0$, and this is indeed what is widely used in
practice and theoretical analysis \citep{du2018gradient, du2019gradient}.}
\begin{equation}\label{equ: one step update 1}
\begin{gathered}
    a_k(t+1)
    =
    a_k(t)-\eta\frac{1}{n}\sum_{i=1}^n\Big(f(\boldsymbol{x}_i;\boldsymbol{\theta}(t))-y_i\Big)\sigma\Big(\boldsymbol{b}_k^\top(t)\boldsymbol{x}_i\Big);
    \\
    \boldsymbol{b}_k(t+1)=
    \boldsymbol{b}_k(t)-\eta\frac{1}{n}\sum_{i=1}^n\Big(f(\boldsymbol{x}_i;\boldsymbol{\theta}(t))-y_i\Big)\mathbb{I}\Big\{\boldsymbol{b}_k^\top(t)\boldsymbol{x}_i>0\Big\}a_k(t)\boldsymbol{x}_i.
    \end{gathered}
\end{equation}

\noindent
\textbf{Hitting time.} We define the hitting time $T$ as:
\begin{equation}\label{equ: hitting time def 1}
\begin{aligned}
  T:= \sup\Big\{t\in\mathbb{N}:\ 
    &\max_{i\in[n]}\left|f(\boldsymbol{x}_i;\boldsymbol{\theta}(s))\right|\leq1, \text{and}
    \\& a_k(s)a_k(0)>0,\forall k\in[m],\forall 0\leq s\leq t+1\Big\}.
\end{aligned}
\end{equation}

\noindent
\textbf{Data concentration.}
We define the following constants $\gamma_1,\gamma_2$ about the concentration of data:
\begin{equation}\label{equ: gamma definition}
\begin{gathered}
\gamma_1:=\frac{1}{n^2}\Big(\sum_{i\in[\frac{n}{2}]}\sum_{j\in[\frac{n}{2}]}\boldsymbol{x}_i^\top\boldsymbol{x}_j\big(1-\frac{\arccos(\boldsymbol{x}_i^\top\boldsymbol{x}_j)}{\pi}\big)+\sum_{i\in[n]-[\frac{n}{2}]}\sum_{j\in[n]-[\frac{n}{2}]}\boldsymbol{x}_i^\top\boldsymbol{x}_j\big(1-\frac{\arccos(\boldsymbol{x}_i^\top\boldsymbol{x}_j)}{\pi}\big)\Big),
\\
\gamma_2:=
\frac{1}{n^2}\Big(\sum_{i\in[\frac{n}{2}]}\sum_{j\in[\frac{n}{2}]}\boldsymbol{x}_i^\top\boldsymbol{x}_j+\sum_{i\in[n]-[\frac{n}{2}]}\sum_{j\in[n]-[\frac{n}{2}]}\boldsymbol{x}_i^\top\boldsymbol{x}_j\Big).    
\end{gathered}
\end{equation}

We know that $\gamma_1,\gamma_2$ are only about the data distribution and satisfy $\gamma_2/2\leq\gamma_1\leq\gamma_2$.

To make our proof more clear, without loss of generality, we use the following specific numbers instead of progressive expression about $\eta\leq\mathcal{O}(1)$, $\kappa\leq\mathcal{O}(\eta\mu_0/n)$ and $m\geq{\Omega}(\log(n/\delta))$:
\begin{equation}\label{equ: parameter selection 1}
    \begin{gathered}
      \eta=\frac{1}{100},
    \\
    \kappa\leq\min\Big\{\frac{1}{1000},\frac{\eta}{2000},\frac{\eta}{3n},\frac{\eta\mu_0}{3n}\Big\},
\\
m\geq \max\Big\{144\log(2n^2/\delta),4\Big\}.
    \end{gathered}
\end{equation}

\subsection{Fine-grained dynamical analysis of neuron partition}
An important technique in our proof is the fine-grained analysis of each neuron during training.

For the $i$-th data $(\boldsymbol{x}_i,y_i)$, we divide all neurons into four categories and study them separately.
We denote the true-living neurons $\mathcal{TL}_i(t)$, the true-dead neurons $\mathcal{TD}_i(t)$, the false-living neurons $\mathcal{FL}_i(t)$ and the false-dead neurons $\mathcal{FD}_i(t)$ at time $t$ as:
\begin{gather*}
    \mathcal{TL}_{i}(t):=\Big\{k\in[m]:y_ia_k(t)>0, \boldsymbol{b}_k(t)^\top\boldsymbol{x}_i> 0\Big\},
    \\
    \mathcal{TD}_{i}(t):=\Big\{k\in[m]:y_ia_k(t)>0,\boldsymbol{b}_k(t)^\top\boldsymbol{x}_i\leq 0\Big\},
    \\
    \mathcal{FL}_{i}(t):=\Big\{k\in[m]:y_ia_k(t)<0, \boldsymbol{b}_k(t)^\top\boldsymbol{x}_i> 0\Big\},
    \\
    \mathcal{FD}_{i}(t):=\Big\{k\in[m]:y_ia_k(t)<0,\boldsymbol{b}_k(t)^\top\boldsymbol{x}_i\leq 0\Big\}.
\end{gather*}

It is easy to verify the following lemma about the partition of all neurons.

\begin{lemma}[Partition]\label{lemma: neuron partition 1}
For any $t\leq T$ and $i\in[n]$, we have:
\begin{gather*}
    [m]=\mathcal{TL}_{i}(t)\bigcup\mathcal{TD}_{i}(t)\bigcup\mathcal{FL}_{i}(t)\bigcup\mathcal{FD}_{i}(t).
\end{gather*}
\end{lemma}
\begin{proof}[Proof of Lemma \ref{lemma: neuron partition 1}]\ \\
From the definition of $T$, we know $a_k(t)\ne 0$, which implies the Lemma.
\end{proof}

\begin{lemma}[Estimate of the initial neural partition]\label{lemma: estimate of the number of initial neural partition 1} With probability at least $1-\delta$, we have the following estimates
\[\Bigg|\frac{1}{m}\text{card}\Big(\mathcal{TL}_i(0)\cap\mathcal{TL}_j(0)\Big)-\frac{\pi-\arccos(\boldsymbol{x}_i^\top\boldsymbol{x}_j)}{4\pi}\Bigg|\leq
\sqrt{\frac{\log(n^2/\delta)}{2m}},\text{ for any $i,j$ in the same class},\]
and the same inequalities also hold for $\mathcal{TL}_i(0)\cap\mathcal{TD}_j(0)$, $\mathcal{TD}_i(0)\cap\mathcal{TL}_j(0)$, and $\mathcal{TD}_i(0)\cap\mathcal{TD}_j(0)$.
\end{lemma}

\begin{proof}[Proof of Lemma \ref{lemma: estimate of the number of initial neural partition 1}]\ \\
We define a matrix $\mathbf{P}(0)=({\rm P}_{i,j}(0))\in\mathbb{R}^{n\times n}$ as:
\[
{\rm P}_{i,j}(0)=\frac{1}{m}\sum_{k=1}^m\mathbb{I}\big\{a_k(0)y_i>0\big\}\mathbb{I}\big\{a_k(0)y_j>0\big\}\mathbb{I}\big\{\boldsymbol{b}_k(0)^\top\boldsymbol{x}_i>0\big\}\mathbb{I}\big\{\boldsymbol{b}_k(0)^\top\boldsymbol{x}_j>0\big\}.
\]
It is easy to verify that:
\[
\text{card}\Big(\mathcal{TL}_i(0)\cap\mathcal{TL}_j(0)\Big)=m{\rm P}_{i,j}(0)
\]

\textbf{(I).}
For $i,j$ in the same class, we can let $i,j\in[n/2]$ (the other case is similar).
Now we consider the expectation:
\begin{align*}
{\rm P}_{i,j}^{\infty}:=&\mathbb{E}_{a\sim\mathbb{U}\{\pm1\},\boldsymbol{b}\sim\mathcal{N}(\mathbf{0},\frac{\kappa^2}{md}\mathbf{I}_d)}\Bigg[\mathbb{I}\big\{a>0\big\}\mathbb{I}\big\{\boldsymbol{b}^\top\boldsymbol{x}_i>0\big\}\mathbb{I}\big\{\boldsymbol{b}^\top\boldsymbol{x}_j>0\big\}\Bigg]
\\=&
\mathbb{E}_{a\sim{\rm Unif}\{\pm1\}}\Bigg[
\mathbb{E}_{\boldsymbol{b}\sim\mathcal{N}(\mathbf{0},\frac{\kappa^2}{md}\mathbf{I}_d)}\Big[\mathbb{I}\big\{a>0\big\}\mathbb{I}\big\{\boldsymbol{b}^\top\boldsymbol{x}_i>0\big\}\mathbb{I}\big\{\boldsymbol{b}^\top\boldsymbol{x}_j>0\big\}\Big|a\Big]\Bigg]
\\=&
\frac{1}{2}
\mathbb{E}_{\boldsymbol{b}\sim\mathcal{N}(\mathbf{0},\frac{\kappa^2}{md}\mathbf{I}_d)}\Big[\mathbb{I}\big\{\boldsymbol{b}^\top\boldsymbol{x}_i>0\big\}\mathbb{I}\big\{\boldsymbol{b}^\top\boldsymbol{x}_j>0\big\}\Big]
\\=&\frac{1}{2}
\mathbb{E}_{\boldsymbol{b}\sim\mathcal{N}(\mathbf{0},\mathbf{I}_d)}\Big[\mathbb{I}\big\{\boldsymbol{b}^\top\boldsymbol{x}_i>0\big\}\mathbb{I}\big\{\boldsymbol{b}^\top\boldsymbol{x}_j>0\big\}\Big]
\\=&
\frac{\pi-\arccos(\boldsymbol{x}_i^\top\boldsymbol{x}_j)}{4\pi}.
\end{align*}
By Hoeffding's Inequality (Lemma \ref{lemma: hoeffding}), we have:
\begin{align*}
    \mathbb{P}\Bigg(\left|{\rm P}_{i,j}(0)-{\rm P}_{i,j}^\infty\right|\geq\sqrt{\frac{\log(2/\delta)}{2m}}\Bigg)\leq\delta.
\end{align*}
Applying a union bound over all $i,j$ in the same class, we know that with probability at least $1-\delta$,
\[
\left|{\rm P}_{i,j}(0)-{\rm P}_{i,j}^\infty\right|\leq\sqrt{\frac{\log(n^2/\delta)}{2m}},\ \forall i,j\in[n],
\]
which means
\begin{align*}
\Bigg|\frac{1}{m}\text{card}\Big(\mathcal{TL}_i(0)\cap\mathcal{TL}_j(0)\Big)-\frac{\pi-\arccos(\boldsymbol{x}_i^\top\boldsymbol{x}_j)}{4\pi}\Bigg|\leq
\sqrt{\frac{\log(n^2/\delta)}{2m}},\text{ for any $i,j$ in the same class}.
\end{align*}

And the same inequalities also hold for $\mathcal{TL}_i(0)\cap\mathcal{TD}_j(0)$, $\mathcal{TD}_i(0)\cap\mathcal{TL}_j(0)$, and $\mathcal{TD}_i(0)\cap\mathcal{TD}_j(0)$.

\end{proof}

\begin{lemma}[Initial norm and prediction]\label{lemma: initial norm and prediction 1} With probability at least $1-2me^{-2d}$ we have:
\begin{gather*}
\left\|\boldsymbol{b}_k(0)\right\|\leq\frac{2\kappa}{\sqrt{m}},\ \forall k\in[m];\\
    |f(\boldsymbol{x}_i;\boldsymbol{\theta}(0))|\leq2\kappa,\ \forall i\in[n].
\end{gather*}

\end{lemma}
\begin{proof}[Proof of Lemma \ref{lemma: initial norm and prediction 1}]\ \\
From the fact ${\boldsymbol{b}}_k(0)\sim\mathcal{N}(\mathbf{0},\frac{\kappa^2}{md}\mathbf{I}_d)$, we have the probability inequality:
\[
\mathbb{P}\Big(\left\|\boldsymbol{b}_k(0)\right\|\leq\frac{2\kappa}{\sqrt{m}}\Big)\geq 1 - 2\exp(-2d)
\]

Combining the uniform bound about $k\in[m]$, with probability at least $1-\frac{2m}{\exp(2d)}$ we have
\[
\left\|\boldsymbol{b}_k(0)\right\|\leq\frac{2\kappa}{\sqrt{m}},\ \forall k\in[m].
\]

From the path norm estimate of $f(\cdot;\cdot)$, we have:
\[
|f(\boldsymbol{x}_i;\boldsymbol{\theta}(0))|\leq\sum_{k=1}^m|a_k(0)||\boldsymbol{b}_k(0)^\top\boldsymbol{x}_i|\leq m\frac{1}{\sqrt{m}}\frac{2\kappa}{\sqrt{m}}=2\kappa.
\]

\end{proof}

\begin{lemma}[Dynamics of neuron partition]\label{lemma: correct neurons remain correct 1}
When the events in Lemma \ref{lemma: initial norm and prediction 1} happened,
 we have the following results for any $i\in[n]$ and $t\leq T$:\\
(S1) True-living neurons remain true-living:  $\mathcal{TL}_i(t)\subset\mathcal{TL}_i(t+1)$.\\
(S2) False-dead neurons remain false-dead:  $\mathcal{FD}_i(t)\subset\mathcal{FD}_i(t+1)$.\\
(S3) True-dead neurons turn true-living in the firt step ($t=0$): $\mathcal{TD}_i(0)\subset\mathcal{TL}_i(1)$.\\
(S4) False-living neurons turn false-dead in the firt step ($t=0$): $\mathcal{FL}_i(0)\subset\mathcal{FD}_i(1)$.\\
(S5) For any $i\in[n]$, $1\leq t\leq T$, $k\in[m]$ and $\boldsymbol{b}\in\overline{\boldsymbol{b}_k(t)\boldsymbol{b}_k(t+1)}$, we have $\text{sgn}(\boldsymbol{b}^\top\boldsymbol{x}_i)\equiv\text{sgn}(\boldsymbol{b}_k^\top(1)\boldsymbol{x}_i)\ne0$.
\end{lemma}

\begin{proof}[Proof of Lemma \ref{lemma: correct neurons remain correct 1}]\ \\
WLOG, we only need to consider $i\in[n/2]$, so we have $y_i=1.$ 

From Assumption \ref{ass: data separation} (i), it is easy to verify $
y_i\boldsymbol{x}_i^\top\boldsymbol{x}_j y_j\geq0$ for any $i,j\in[n]$. Besides, from Lemma \ref{lemma: initial norm and prediction 1}, we have
$|\boldsymbol{b}_k(0)^\top\boldsymbol{x}_i|\leq\frac{2\kappa}{\sqrt{m}}$ and $|f(\boldsymbol{x}_i;\boldsymbol{\theta}(0))|\leq 2\kappa$, so $\text{sgn}(y_j)=\text{sgn}(y_j-f(\boldsymbol{x}_j;\boldsymbol{\theta}(0)))$ holds for any $j\in[n]$. Hence, for any $i,j\in[n]$, we have
\[
\Big(y_i-f(\boldsymbol{x}_i;\boldsymbol{\theta}(t)\Big)\boldsymbol{x}_i^\top\boldsymbol{x}_j y_j\geq0.
\]


(Proof of S1.)
Let $k\in\mathcal{TL}_i(t)$, we have $a_k(t)>0$, $ \boldsymbol{b}_k(t)^\top\boldsymbol{x}_i>0$, and $a_{k}(t+1)>0$ (from the definition of $T$). 

Then we prove a stronger result than (S1), ``For $k\in\mathcal{TL}_i(t)$, we have $\boldsymbol{b}^\top\boldsymbol{x}_i>0$ for any $\boldsymbol{b}\in\overline{\boldsymbol{b}_k(t)\boldsymbol{b}_k(t+1)}$'':

There exists $\alpha\in[0,1]$, s.t. $\boldsymbol{b}=\alpha\boldsymbol{b}_k(t)+(1-\alpha)\boldsymbol{b}_k(t+1)$, and
\begin{align*}
    &\boldsymbol{b}^\top\boldsymbol{x}_i
    \\
    =&\alpha\boldsymbol{b}_k(t)^\top\boldsymbol{x}_i+(1-\alpha)\boldsymbol{b}_k(t+1)^\top\boldsymbol{x}_i
    \\=&\boldsymbol{b}_k(t)^\top\boldsymbol{x}_i+(1-\alpha)\eta a_k(t)\frac{1}{n}\sum_{j=1}^n\Big(y_j-f(\boldsymbol{x}_j;\boldsymbol{\theta}(t))\Big)\boldsymbol{x}_j^\top\boldsymbol{x}_i\mathbb{I}\Big\{\boldsymbol{b}_k(t)^\top\boldsymbol{x}_j>0\Big\}
    \\=&\boldsymbol{b}_k(t)^\top\boldsymbol{x}_i+(1-\alpha)\eta a_k(t)y_i\frac{1}{n}\sum_{j=1}^n\mathbb{I}\Big\{\boldsymbol{b}_k(t)^\top\boldsymbol{x}_j>0\Big\}\Big(y_j-f(\boldsymbol{x}_j;\boldsymbol{\theta}(t))\Big)\boldsymbol{x}_j^\top\boldsymbol{x}_i y_i
    \\=&\boldsymbol{b}_k(t)^\top\boldsymbol{x}_i+(1-\alpha)\eta\frac{1}{n}\sum_{j\in[n]}\mathbb{I}\Big\{\boldsymbol{b}_k(t)^\top\boldsymbol{x}_j>0\Big\}\Big((y_j-f(\boldsymbol{x}_j;\boldsymbol{\theta}(0)))\boldsymbol{x}_j^\top\boldsymbol{x}_i y_i\Big)\Big( y_i a_k(t)\Big)
    \\=&\boldsymbol{b}_k(t)^\top\boldsymbol{x}_i+(1-\alpha)\eta\frac{1}{n}(1-f(\boldsymbol{x}_i;\boldsymbol{\theta}(t)))\mathbb{I}\big\{\boldsymbol{b}_k(t)^\top\boldsymbol{x}_i>0\big\} \boldsymbol{x}_i^\top\boldsymbol{x}_i a_k(0)
    \\&\quad\quad\quad\quad+(1-\alpha)\eta\frac{1}{n}\sum_{j\ne i}\mathbb{I}\Big\{\boldsymbol{b}_k(t)^\top\boldsymbol{x}_j>0\Big\}\Big((y_j-f(\boldsymbol{x}_j;\boldsymbol{\theta}(t)))\boldsymbol{x}_j^\top\boldsymbol{x}_i y_i\Big)\Big(y_i a_k(t)\Big)
    \\\geq&\boldsymbol{b}_k(t)^\top\boldsymbol{x}_i+(1-\alpha)\eta\frac{1}{n}(1-f(\boldsymbol{x}_i;\boldsymbol{\theta}(t)))\mathbb{I}\big\{\boldsymbol{b}_k(t)^\top\boldsymbol{x}_i>0\big\} \boldsymbol{x}_i^\top\boldsymbol{x}_i a_k(0)
    \\\geq&\boldsymbol{b}_k(0)^\top\boldsymbol{x}_i.
\end{align*}
Notice that $\boldsymbol{b}_k^\top(0)\boldsymbol{x}_i>0$.
By induction, we know $\boldsymbol{b}^\top\boldsymbol{x}_i>0$.
Moreover, we have $k\in\mathcal{TL}_i(t+1)$.

(Proof of S2.)
Let $k\in\mathcal{FD}_i(t)$, we have $a_k(t)<0$, $ \boldsymbol{b}_k(t)^\top\boldsymbol{x}_i\leq0$, and $a_k(t+1)<0$ (from the definition of $T$).

Then we prove a stronger result than (S2), ``For any $k\in\mathcal{FD}_i(t)$, we have $\boldsymbol{b}^\top\boldsymbol{x}_i\leq0$ for any $\boldsymbol{b}\in\overline{\boldsymbol{b}_k(t)\boldsymbol{b}_k(t+1)}$ '':

Combining Assumption \ref{ass: data separation} (ii), we know there exists $j_i\in[n]$, s.t. $\boldsymbol{b}_k(0)^\top\boldsymbol{x}_{j_i}\geq0$ and $y_i\boldsymbol{x}_i^\top\boldsymbol{x}_{j_i} y_{j_i}\geq\mu_0$.
Hence, with probability at least $1-\frac{2m}{\exp(2d)}$ we have

There exists $\alpha\in[0,1]$, s.t. $\boldsymbol{b}=\alpha\boldsymbol{b}_k(t)+(1-\alpha)\boldsymbol{b}_k(t+1)$, and
\begin{align*}
    &\boldsymbol{b}^\top\boldsymbol{x}_i
    \\
    =&\alpha\boldsymbol{b}_k(t)^\top\boldsymbol{x}_i+(1-\alpha)\boldsymbol{b}_k(t+1)^\top\boldsymbol{x}_i
    \\=&\boldsymbol{b}_k(t)^\top\boldsymbol{x}_i+(1-\alpha)\eta a_k(t)\frac{1}{n}\sum_{j=1}^n\Big(y_j-f(\boldsymbol{x}_j;\boldsymbol{\theta}(t))\Big)\boldsymbol{x}_j^\top\boldsymbol{x}_i\mathbb{I}\Big\{\boldsymbol{b}_k(t)^\top\boldsymbol{x}_j>0\Big\}
    \\=&\boldsymbol{b}_k(t)^\top\boldsymbol{x}_i+(1-\alpha)\eta a_k(t)y_i\frac{1}{n}\sum_{j=1}^n\mathbb{I}\Big\{\boldsymbol{b}_k(t)^\top\boldsymbol{x}_j>0\Big\}\Big(y_j-f(\boldsymbol{x}_j;\boldsymbol{\theta}(t))\Big)\boldsymbol{x}_j^\top\boldsymbol{x}_i y_i    \\=&\boldsymbol{b}_k(t)^\top\boldsymbol{x}_i+(1-\alpha)\eta\frac{1}{n}\mathbb{I}\big\{\boldsymbol{b}_k(t)^\top\boldsymbol{x}_{j_i}\geq0\big\}\Big((y_{j_i}-f(\boldsymbol{x}_{j_i};\boldsymbol{\theta}(t)))\boldsymbol{x}_{j_i}^\top\boldsymbol{x}_i y_i\Big)\Big(y_i a_k(t)\Big)
    \\&\quad\quad\quad\quad+(1-\alpha)\eta\frac{1}{n}\sum_{j\ne j_i}\mathbb{I}\Big\{\boldsymbol{b}_k(t)^\top\boldsymbol{x}_j>0\Big\}\Big((y_j-f(\boldsymbol{x}_j;\boldsymbol{\theta}(t)))\boldsymbol{x}_j^\top\boldsymbol{x}_i y_i\Big)\Big(y_i a_k(t)\Big)
    \\\leq&\boldsymbol{b}_k(t)^\top\boldsymbol{x}_i+(1-\alpha)\eta\frac{1}{n}\mathbb{I}\big\{\boldsymbol{b}_k(t)^\top\boldsymbol{x}_{j_i}\geq0\big\}\Big((y_{j_i}-f(\boldsymbol{x}_{j_i};\boldsymbol{\theta}(t)))\boldsymbol{x}_{j_i}^\top\boldsymbol{x}_i y_i\Big)\Big(y_i a_k(t)\Big)
    \\\leq&\boldsymbol{b}_k(t)^\top\boldsymbol{x}_i.
\end{align*}
By induction, we know $\boldsymbol{b}^\top\boldsymbol{x}_i\leq0$.
Moreover, we have $k\in\mathcal{FD}_i(t+1)$.

(Proof of S3.)
Let $k\in\mathcal{TD}_i(0)$, we have $a_k(0)>0$ and $ \boldsymbol{b}_k(0)^\top\boldsymbol{x}_i<0$. Then we have $a_k(1)>0$ (from the definition of $T$). Combining Assumption \ref{ass: data separation} (ii), we know there exists $j_i\in[n]$, s.t. $\boldsymbol{b}_k(0)^\top\boldsymbol{x}_{j_i}\geq0$ and $y_i\boldsymbol{x}_i^\top\boldsymbol{x}_{j_i} y_{j_i}\geq\mu_0$.
Hence, with probability at least $1-\frac{2m}{\exp(2d)}$ we have
\begin{align*}
    &\boldsymbol{b}_k(1)^\top\boldsymbol{x}_i
    \\\geq&\boldsymbol{b}_k(0)^\top\boldsymbol{x}_i+\eta\frac{1}{n}\sum_{j\in[n]}(y_j-f(\boldsymbol{x}_j;\boldsymbol{\theta}(0)))\mathbb{I}\big\{\boldsymbol{b}_k(0)^\top\boldsymbol{x}_j>0\big\}\boldsymbol{x}_j^\top\boldsymbol{x}_i a_k(0)
    \\=&\boldsymbol{b}_k(0)^\top\boldsymbol{x}_i+\eta\frac{1}{n}\sum_{j\in[n]}\mathbb{I}\big\{\boldsymbol{b}_k(0)^\top\boldsymbol{x}_j>0\big\}\Big((y_j-f(\boldsymbol{x}_j;\boldsymbol{\theta}(0)))\boldsymbol{x}_j^\top\boldsymbol{x}_i y_i\Big)\Big( y_i a_k(0)\Big)
    \\=&\boldsymbol{b}_k(0)^\top\boldsymbol{x}_i+\eta\frac{1}{n}\mathbb{I}\big\{\boldsymbol{b}_k(0)^\top\boldsymbol{x}_{j_i}\geq0\big\}\Big((y_{j_i}-f(\boldsymbol{x}_{j_i};\boldsymbol{\theta}(0)))\boldsymbol{x}_{j_i}^\top\boldsymbol{x}_i y_i\Big)\Big(y_i a_k(0)\Big)
    \\&\quad\quad\quad\quad+\eta\frac{1}{n}\sum_{j\ne j_i}\mathbb{I}\big\{\boldsymbol{b}_k(0)^\top\boldsymbol{x}_j>0\big\}\Big((y_j-f(\boldsymbol{x}_j;\boldsymbol{\theta}(0)))\boldsymbol{x}_j^\top\boldsymbol{x}_i y_i\Big)\Big(y_i a_k(0)\Big)
    \\\geq&\boldsymbol{b}_k(0)^\top\boldsymbol{x}_i+\eta\frac{1}{n}\mathbb{I}\big\{\boldsymbol{b}_k(0)^\top\boldsymbol{x}_{j_i}\geq0\big\}\Big((y_{j_i}-f(\boldsymbol{x}_{j_i};\boldsymbol{\theta}(0)))\boldsymbol{x}_{j_i}^\top\boldsymbol{x}_i y_i\Big)\Big(y_i a_k(0)\Big)
    \\\geq&-\frac{2\kappa}{\sqrt{m}}+\frac{\eta}{n\sqrt{m}}(1-2\kappa)\mu_0
    \overset{(\ref{equ: parameter selection 1})}{>}0,
\end{align*}
so $k\in\mathcal{TL}_i(1)$.

(Proof of S4.) Let $k\in\mathcal{FL}_i(0)$, we have $a_k(0)<0$ and $ \boldsymbol{b}_k(0)^\top\boldsymbol{x}_i\geq0$. Then we have $a_k(1)<0$ (from the definition of $T$). 

So with probability at least $1-\frac{2m}{\exp(2d)}$ we have
\begin{align*}
    &\boldsymbol{b}_k(1)^\top\boldsymbol{x}_i
    \\\leq&\boldsymbol{b}_k(0)^\top\boldsymbol{x}_i+\eta\frac{1}{n}\sum_{j\in[n]}(y_j-f(\boldsymbol{x}_j;\boldsymbol{\theta}(0)))\mathbb{I}\big\{\boldsymbol{b}_k(0)^\top\boldsymbol{x}_j>0\big\}\boldsymbol{x}_j^\top\boldsymbol{x}_i a_k(0)
        \\=&\boldsymbol{b}_k(0)^\top\boldsymbol{x}_i+\eta\frac{1}{n}\sum_{j\in[n]}\mathbb{I}\big\{\boldsymbol{b}_k(0)^\top\boldsymbol{x}_j>0\big\}\Big((y_j-f(\boldsymbol{x}_j;\boldsymbol{\theta}(0)))\boldsymbol{x}_j^\top\boldsymbol{x}_i y_i\Big)\Big( y_i a_k(0)\Big)
        \\=&\boldsymbol{b}_k(0)^\top\boldsymbol{x}_i+\eta\frac{1}{n}(1-f(\boldsymbol{x}_i;\boldsymbol{\theta}(0)))\mathbb{I}\big\{\boldsymbol{b}_k(0)^\top\boldsymbol{x}_i>0\big\} \boldsymbol{x}_i^\top\boldsymbol{x}_i a_k(0)
    \\&\quad\quad\quad\quad+\eta\frac{1}{n}\sum_{j\ne i}\mathbb{I}\big\{\boldsymbol{b}_k(0)^\top\boldsymbol{x}_j>0\big\}\Big((y_j-f(\boldsymbol{x}_j;\boldsymbol{\theta}(0)))\boldsymbol{x}_j^\top\boldsymbol{x}_i y_i\Big)\Big(y_i a_k(0)\Big)
    \\\leq&\boldsymbol{b}_k(0)^\top\boldsymbol{x}_i+\eta\frac{1}{n}(1-f(\boldsymbol{x}_i;\boldsymbol{\theta}(0)))\mathbb{I}\big\{\boldsymbol{b}_k(0)^\top\boldsymbol{x}_i>0\big\} \boldsymbol{x}_i^\top\boldsymbol{x}_i a_k(0)
    \\\leq&\frac{2\kappa}{\sqrt{m}}-\frac{\eta}{n\sqrt{m}}(1-f(\boldsymbol{x}_i;\boldsymbol{\theta}(0)))
    \\\overset{(\ref{equ: parameter selection 1})}{\leq}&\frac{2\kappa}{\sqrt{m}}-\frac{\eta}{n\sqrt{m}}(1-2\kappa)< 0 ,
\end{align*}
so $k\in\mathcal{FD}_i(1)$. 

(Proof of S5.)
From (S3)(S4), we know $m=\mathcal{TL}_i(1)\cup\mathcal{FD}_i(1)$ and $\text{sgn}(\boldsymbol{b}_k(1)^\top\boldsymbol{x}_i)\ne 0$ for any $i\in[n]$ and $k\in[m]$. Recalling the proofs in (S1)(S2), we obtain $\text{sgn}\big(\boldsymbol{b}^\top\boldsymbol{x}_i\big)\equiv\text{sgn}\big(\boldsymbol{b}_k^\top(1)\boldsymbol{x}_i\big)\ne0$ for any $i\in[n]$, $k\in[m]$, $1\leq t\leq T$ and $\boldsymbol{b}
\in\overline{\boldsymbol{b}_k(t)\boldsymbol{b}_k(t+1)}$.

\end{proof}

\begin{lemma}\label{lemma: corollory correct neurons remain correct 1} 
Under the same condition as Lemma \ref{lemma: correct neurons remain correct 1}, we have the following results for any $i\in[n]$ and $1\leq t\leq T+1$:
\begin{gather*}
    \mathcal{TL}_{i}(t)\equiv\mathcal{TL}_{i}(0)\bigcup \mathcal{TD}_{i}(0);\\
   \mathcal{FD}_{i}(t)\equiv\mathcal{FL}_{i}(0)\bigcup \mathcal{FD}_{i}(0).
\end{gather*}

\end{lemma}
\begin{proof}[Proof of Lemma \ref{lemma: corollory correct neurons remain correct 1}]\ \\
This lemma is a direct corollory of Lemma \ref{lemma: correct neurons remain correct 1}.

\end{proof}

\subsection{Estimate of hitting time and parameters}
\begin{remark}\rm
The Lemmas in this section will be discussed and proved when the events in Lemma \ref{lemma: estimate of the number of initial neural partition 1} Lemma \ref{lemma: initial norm and prediction 1} happened. So all the lemmas below hold with probability at least $1-\delta-2me^{-2d}$.
\end{remark}

First, we construct the following exponentially growing hitting time:
\begin{equation}\label{equ: exponential hitting time def 1}
\begin{aligned}
  T_e:= \sup\Big\{t\in\mathbb{N}:\ 
    &\Big(\frac{1}{2}+2\sqrt{\frac{\log(2n^2/\delta)}{m}}\Big)\frac{251001\Big((1+2\eta)^{2(t+1)}-(1-2\eta)^{2(t+1)}\Big)}{1000000}\leq1;
    \\&{\rm and}\ (1+2\eta)^{t+1}\leq 2\sqrt{2}\Big\}.
\end{aligned}
\end{equation}

\begin{lemma}[Parameter estimate]\label{lemma: speed separation 1}
For any $1\leq t\leq T_e+1$, $i\in[n]$, we have:
\begin{align*}
&\text{(P1).\ }\forall k\in\mathcal{TL}_i(t),\ 
\left|a_k(t)\right|\leq\frac{501\Big((1+2\eta)^t+(1-2\eta)^t\Big)}{1000\sqrt{m}};\\
&\text{(P2).\ }\forall k\in\mathcal{TL}_i(t),\ 
    \left\|\boldsymbol{b}_k(t)\right\|\leq\frac{501\Big((1+2\eta)^t-(1-2\eta)^t\Big)}{1000\sqrt{m}};\\
 &   \text{(P3).\ }
    \left|f\big(\boldsymbol{x}_i;\boldsymbol{\theta}(t)\big)\right|\leq\Big(\frac{1}{2}+2\sqrt{\frac{\log(2n^2/\delta)}{m}}\Big)\frac{251001\Big((1+2\eta)^{2t}-(1-2\eta)^{2t}\Big)}{1000000};\\
              &  \text{(P4).\ }\forall k\in[m],\ 
    {\rm sgn}(a_k(0))a_k(t)\geq \frac{20019}{10000\sqrt{m}}-\frac{501}{1000\sqrt{m}}\Big((1+2\eta)^{t}+(1-2\eta)^{t}\Big);\\
                  &  \text{(P5).\ } T_e\leq T; \\
      &  \text{(P6).\ }\forall k\in\mathcal{TL}_i(t),\ 
    \left|a_k(t)\right|\geq|a_k(t-1)|\geq\cdots\geq|a_k(1)|\geq\frac{9999}{10000\sqrt{m}}.
\end{align*}
\end{lemma}

\begin{proof}[Proof of Lemma \ref{lemma: speed separation 1}]\ \\
We will complete the proof in the order: 
(P1)(P2)(P3)$\Rightarrow$(P4)$\Rightarrow$(P5)$\Rightarrow$(P6).

For convenience, we define the following sequences:
\begin{gather*}
    x_k^t = |a_k(t)|,k\in\mathcal{TL}_i(t);
    \\
    y_k^t=\left\|\boldsymbol{b}_k(t)\right\|,k\in\mathcal{TL}_i(t);\\
    f_i^t = \sum_{k=1}^m a_k(t)\sigma\Big(\boldsymbol{b}_k(t)^\top\boldsymbol{x}_i\Big).
\end{gather*}
where $t\leq T_e+1,\ k\in[m],\ i\in[n]$.

\noindent\textbf{Step I.} Proof of (P1)(P2)(P3).

For $0\leq t\leq T_e$, we have:
\begin{gather*}
    x_k^{t+1}\leq x_k^t +\eta \Big(\max\limits_{i\in[n]}\left|f_i^t-y_i\right|\Big) y_k^t
    \leq x_k^t +\eta \Big(\max\limits_{i\in[n]}\left|f_i^t\right|+1\Big) y_k^t, k\in\mathcal{TL}_i(t);
    \\
    y_k^{t+1}\leq y_k^t + \eta\Big(\max\limits_{i\in[n]}\left|f_i^t-y_i\right|\Big) x_k^t\leq y_k^t + \eta\Big(\max\limits_{i\in[n]}\left|f_i^t\right|+1\Big) x_k^t,k\in\mathcal{TL}_i(t).
\end{gather*}
For $1\leq t\leq T_e+1$, from Lemma \ref{lemma: correct neurons remain correct 1} (S1)(S2)(S3)(S4), we have:
\begin{align*}
 f_i^{t}=&\sum_{k=1}^m a_k(t)\sigma(\boldsymbol{b}_k(t)^\top\boldsymbol{x}_i)=\sum_{k\in\mathcal{TL}_i(t)\cup\mathcal{FL}_i(t)} a_k(t)\sigma(\boldsymbol{b}_k(t)^\top\boldsymbol{x}_i)
 \\=&\sum_{k\in\mathcal{TL}_i(t)} a_k(t)\sigma(\boldsymbol{b}_k(t)^\top\boldsymbol{x}_i)
 \leq \sum_{k\in\mathcal{TL}_i(t)}x_k^t y_k^t.
\end{align*}

If we define the following sequences:
\begin{gather*}
\alpha_k^0=x_k^0,\quad \beta_k^0=y_k^0,\quad \gamma^0=\sum_{k=1}^m\alpha_k^0\beta_k^0;
\\
    \alpha_k^{t+1}=\alpha_k^{t}+\eta(\gamma^t+1)\beta_k^{t};
    \\
    \beta_k^{t+1}=\beta_k^t+\eta(\gamma^t+1)\alpha_k^{t};
    \\
    \gamma^{t+1}=\max_{i\in[n]}\sum_{k\in\mathcal{TL}_i(t+1)}\alpha_k^{t+1}\beta_k^{t+1}.
\end{gather*}
It is easy to verify that: for any $1\leq t\leq T_e+1$,
\begin{equation}\label{equ: inequality of sequences 1}
\begin{gathered}
x_k^t\leq \alpha_k^t;\quad y_k^t\leq\beta_k^t;\\
\max_{i\in[n]}|f_i^t|\leq \gamma^t.
\end{gathered}
\end{equation}

Now we aim to prove the following three properties for $1\leq t\leq T_e+1$ by induction:
\begin{equation} \label{equ: induction 1}
    \begin{gathered}
  \alpha_k^t\leq\frac{501\Big((1+2\eta)^t+(1-2\eta)^t\Big)}{1000\sqrt{m}};
    \\
   \beta_k^t\leq\frac{(1+2\eta)^t-(1-2\eta)^t}{2\sqrt{m}};
    \\
    \gamma^t\leq\frac{501\Big((1+2\eta)^{2t}-(1-2\eta)^{2t}\Big)}{2000}\Bigg(\frac{1}{2}+2\sqrt{\frac{\log(2n^2/\delta)}{m}}\Bigg)\leq 1.
    \end{gathered}
\end{equation}

For $t=1$, with Lemma \ref{lemma: estimate of the number of initial neural partition 1} and the one-step updating (\ref{equ: one step update 1}), we have:
\begin{gather*}
    \alpha_k^1=x_k^1\leq\frac{1}{\sqrt{m}}+\eta(1+2\kappa)\frac{2\kappa}{\sqrt{m}}\overset{(\ref{equ: parameter selection 1})}{\leq}\frac{10002}{10000\sqrt{m}},k\in\mathcal{TL}_i(1);
    \\
    \beta_k^1= y_k^1\leq\frac{2\kappa}{\sqrt{m}}+\eta(1+2\kappa)\frac{1}{\sqrt{m}}
    \overset{(\ref{equ: parameter selection 1})}{\leq}\frac{2\eta}{\sqrt{m}},k\in\mathcal{TL}_i(1);
    \\
    \gamma^1\leq m\frac{10002}{10000\sqrt{m}}\frac{2\eta}{\sqrt{m}}\leq1.
\end{gather*}

Assume (\ref{equ: induction 1}) holds for $s\leq t-1\ (t\geq2)$. Then for $s=t$, we have:
\begin{gather*}
    \alpha_{k}^{t}+\beta_k^{t} = (\alpha_k^0+\beta_k^0)\prod\limits_{s=0}^{t-1}\Big(1+\eta(\gamma^s+1)\Big),\\
\alpha_{k}^{t}-\beta_k^{t} = (\alpha_k^0-\beta_k^0)\prod_{s=0}^{t-1}\Big(1-\eta(\gamma^s+1)\Big),
\end{gather*}
which imply:
\begin{gather*}
    \alpha_k^{t}=\frac{\prod\limits_{s=0}^{t-1}\Big(1+\eta(\gamma^{s}+1)\Big)+\prod\limits_{s=0}^{t-1}\Big(1-\eta(\gamma^s+1)\Big)}{2}\alpha_k^0+\frac{\prod\limits_{s=0}^{t-1}\Big(1+\eta(\gamma^s+1)\Big)-\prod\limits_{s=0}^{t-1}\Big(1-\eta(\gamma^s+1)\Big)}{2}\beta_k^0,\\
    \beta_k^{t}=\frac{\prod\limits_{s=0}^{t-1}\Big(1+\eta(\gamma^s+1)\Big)+\prod\limits_{s=0}^{t-1}\Big(1-\eta(\gamma^s+1)\Big)}{2}\beta_k^0+\frac{\prod\limits_{s=0}^{t-1}\Big(1+\eta(\gamma^s+1)\Big)-\prod\limits_{s=0}^{t-1}\Big(1-\eta(\gamma^s+1)\Big)}{2}\alpha_k^0.
\end{gather*}

Combining the definition of $T_e$ (\ref{equ: exponential hitting time def 1}) and Lemma \ref{lemma: alpha increase}, we know:
\begin{align*}
    \alpha_k^{t}\leq&\frac{(1+2\eta)^t+(1-2\eta)^t}{2}\alpha_k^0+\frac{(1+2\eta)^t-(1-2\eta)^t}{2}\beta_k^0
    \\\leq&
    \frac{\Big((1+2\eta)^t+(1-2\eta)^t\Big)}{2}(\alpha_k^0+\beta_k^0)\leq\frac{(1+2\eta)^t+(1-2\eta)^t}{2}\frac{(1+2\kappa)}{\sqrt{m}}
    \\\overset{\eqref{equ: parameter selection 1}}{\leq}&\frac{501\Big((1+2\eta)^t+(1-2\eta)^t\Big)}{1000\sqrt{m}};
\end{align*}
\begin{align*}
    \beta_k^{t}\leq&\frac{(1+2\eta)^t+(1-2\eta)^t}{2}\beta_k^0+\frac{(1+2\eta)^t-(1-2\eta)^t}{2}\alpha_k^0
    \\
    \overset{\text{Lemma \ref{lemma: alpha increase}}}{\leq}&
   \frac{(1+2\eta)^{T_e+1}+(1-2\eta)^{T_e+1}}{2}\frac{2\kappa}{\sqrt{m}}+\frac{(1+2\eta)^t-(1-2\eta)^t}{2\sqrt{m}}
   \\
   \overset{\eqref{equ: exponential hitting time def 1}}{\leq}&
   \frac{4+4}{2}\frac{2\kappa}{\sqrt{m}}+\frac{(1+2\eta)^t-(1-2\eta)^t}{2\sqrt{m}}
   \\\overset{\eqref{equ: parameter selection 1}}{\leq}&\frac{4\eta}{1000\sqrt{m}}+\frac{(1+2\eta)^t-(1-2\eta)^t}{2\sqrt{m}}
    \\\leq&\frac{501\Big((1+2\eta)^t-(1-2\eta)^t\Big)}{1000\sqrt{m}};
\end{align*}
and
\begin{align*}
    \gamma^{t}
    \leq&\max_{i\in[n]}\sum_{k\in\mathcal{TL}_i(t)}\alpha_k^{t}\beta_k^{t}\leq
    \max_{i\in[n]}\sum_{k\in\mathcal{TL}_i(t)\cup\mathcal{TD}_i(t)}\alpha_k^{t}\beta_k^{t}
    \\
    \overset{\text{Lemma \ref{lemma: corollory correct neurons remain correct 1}}}{=}&
    \max_{i\in[n]}\sum_{k\in\mathcal{TL}_i(0)\cup\mathcal{TD}_i(0)}\alpha_k^{t}\beta_k^{t}
    \\
    \overset{\text{Lemma \ref{lemma: estimate of the number of initial neural partition 1}}}{\leq}&
    m\Big(\frac{1}{2}+2\sqrt{\frac{\log(2n^2/\delta)}{m}}\Big)
    \frac{501\Big((1+2\eta)^t+(1-2\eta)^t\Big)}{1000\sqrt{m}}\frac{501\Big((1+2\eta)^t-(1-2\eta)^t\Big)}{1000\sqrt{m}}
     \\
     =&
     \Big(\frac{1}{2}+2\sqrt{\frac{\log(2n^2/\delta)}{m}}\Big)\frac{251001\Big((1+2\eta)^{2t}-(1-2\eta)^{2t}\Big)}{1000000}\overset{\eqref{equ: exponential hitting time def 1}}{\leq} 1.
\end{align*}

By induction, we complete the proof of (\ref{equ: induction 1}). Then with the analysis \eqref{equ: inequality of sequences 1}, we can obtain (P1)(P2)(P3) for all $1\leq t\leq T_e+1$:
\begin{gather*}
    x_k^t\leq\alpha_k^t\leq\frac{501\Big((1+2\eta)^t+(1-2\eta)^t\Big)}{1000\sqrt{m}};
    \\
    y_k^t\leq\beta_k^t\leq\frac{501\Big((1+2\eta)^t-(1-2\eta)^t\Big)}{1000\sqrt{m}};
\end{gather*}
\begin{align*}
    f_i^t\leq\gamma^t\leq&\Big(\frac{1}{2}+2\sqrt{\frac{\log(2n^2/\delta)}{m}}\Big)\frac{251001\Big((1+2\eta)^{2t}-(1-2\eta)^{2t}\Big)}{1000000}\leq1.
\end{align*}

\textbf{Step II.} Proof of (P4).

WLOG, we let $a_k(0)=\frac{1}{\sqrt{m}}$. ($a_k(0)=-\frac{1}{\sqrt{m}}$ is similar.) 

For $t=1$, we have:
\begin{align*}
    a_k(1)=&a_k(0)-\eta\frac{1}{n}\sum_{j=1}^n\Big(f(\boldsymbol{x}_j;\boldsymbol{\theta}(0))-y_j\Big) \sigma(\boldsymbol{b}_k(0)^\top\boldsymbol{x}_j)
    \\\geq&a_k(0)-\eta(1+2\kappa)\frac{2\kappa}{\sqrt{m}}\geq\frac{9999}{10000\sqrt{m}}.
\end{align*}

For any $1\leq t\leq T_e$, we have:
\begin{align*}
a_k(t+1)=&a_k(t)-\eta\frac{1}{n}\sum_{j=1}^n\Big(f(\boldsymbol{x}_j;\boldsymbol{\theta}(t))-y_j\Big) \sigma(\boldsymbol{b}_k(t)^\top\boldsymbol{x}_j)
\\\overset{\text{(P2)}}{\geq}&
a_{k}(t)-2\eta\frac{501\Big((1+2\eta)^t-(1-2\eta)^t\Big)}{1000\sqrt{m}}
\\\geq&
a_k(1)-\frac{1002\eta}{1000\sqrt{m}}\sum_{s=1}^t\Big((1+2\eta)^s-(1-2\eta)^s\Big)
\\\geq&
\frac{9999}{10000\sqrt{m}}-\frac{1002}{1000\sqrt{m}}\Big(-1+\frac{(1+2\eta)^{t+1}+(1-2\eta)^{t+1}}{2}\Big)
\\\geq&
\frac{20019}{10000\sqrt{m}}-\frac{501}{1000\sqrt{m}}\Big((1+2\eta)^{t+1}+(1-2\eta)^{t+1}\Big).
\end{align*}
So for any $1\leq t\leq T_e+1$, we have:
\[
a_{k}(t)\geq\frac{20019}{10000\sqrt{m}}-\frac{501}{1000\sqrt{m}}\Big((1+2\eta)^{t}+(1-2\eta)^{t}\Big).
\]

\textbf{Step III.} Proof of (P5). 

From \eqref{equ: parameter selection 1}, we know $
\sqrt{\frac{\log(2n^2/\delta)}{m}}\leq\frac{1}{12}$. From (P3)(P4) and Lemma \ref{lemma: basic exp}, we can verify that for any $t\leq T_e+1$,
\begin{align*}
    \max_{i\in[n]}\left|f(\boldsymbol{x}_i;\boldsymbol{\theta}(t))\right|\leq&\Big(\frac{1}{2}+2\sqrt{\frac{\log(2n^2/\delta)}{m}}\Big)\frac{251001\Big((1+2\eta)^{2t}-(1-2\eta)^{2t}\Big)}{1000000}
    \\\leq&\Big(\frac{1}{2}+2\sqrt{\frac{\log(2n^2/\delta)}{m}}\Big)\frac{251001\Big((1+2\eta)^{2( T_e+1)}-(1-2\eta)^{2( T_e+1)}\Big)}{1000000}
    \\\leq&1,
\end{align*}
\begin{align*}
  a_k(t)a_k(0)\geq&\frac{20019}{10000{m}}-\frac{501}{1000{m}}\Big((1+2\eta)^{t}+(1-2\eta)^{t}\Big)
  \\\geq&\frac{20019}{10000{m}}-\frac{501}{1000{m}}\Big((1+2\eta)^{T_e+1}+(1-2\eta)^{T_e+1}\Big)
  \\\geq&\frac{20019}{10000{m}}-\frac{501}{1000{m}}(2\sqrt{2}+1)
  >0,\forall k\in[m],
\end{align*}
which means $T_e\leq T$.

\textbf{Step IV.} Proof of (P6). 

WLOG, we let $i\in[\frac{n}{2}]$ ($i\in[n]-[\frac{n}{2}]$ is similar). Let $k\in\mathcal{TL}_i(t)$, then we have $a_k(t)>0$ and $\boldsymbol{b}_k(t)^\top\boldsymbol{x}_i> 0$.
Combining (P5) and Lemma \ref{lemma: correct neurons remain correct 1}, we know that for any $1\leq t\leq T_e+1$,
    \begin{gather*}
     \boldsymbol{b}_k(t)^\top\boldsymbol{x}_j\geq 0
    \Rightarrow k\in\mathcal{TL}_j(t)\Rightarrow a_k(t)y_j>0,
    \end{gather*}
    which implies that for any $j\in[n]-[\frac{n}{2}]$:
    \[
    \Big\{k\in[m]:\boldsymbol{b}_k(t)^\top\boldsymbol{x}_j\geq 0\Big\}=\varnothing.
    \]
    So for any $1\leq t\leq T_e$, we have:
    \begin{align*}
        a_k(t+1)=&a_k(t)-\eta\frac{1}{n}\sum_{j=1}^n \Big(f(\boldsymbol{x}_j;\boldsymbol{\theta}(t))-y_j\Big) \sigma(\boldsymbol{b}_k(t)^\top\boldsymbol{x}_j)
        \\=&
        a_k(t)-\eta\frac{1}{n}\sum_{j\in[\frac{n}{2}]} \Big(f(\boldsymbol{x}_j;\boldsymbol{\theta}(t))-1\Big) \sigma(\boldsymbol{b}_k(t)^\top\boldsymbol{x}_j)
        \\\geq&
        a_k(t)\geq\cdots\geq a_k(1).
    \end{align*}
    From (P2), the one-step update \eqref{equ: one step update 1} and Lemma \ref{lemma: initial norm and prediction 1}, we have:
    \begin{align*}
        a_k(1)\geq& a_k(0)-\eta\frac{1}{n}\sum_{j=1}^n \Big(f(\boldsymbol{x}_j;\boldsymbol{\theta}(0))-y_j\Big) \sigma(\boldsymbol{b}_k(0)^\top\boldsymbol{x}_j)
        \\\geq&\frac{1}{\sqrt{m}}-\eta(1+2\kappa)\frac{2\kappa}{\sqrt{m}}
        \\\overset{\eqref{equ: parameter selection 1}}{\geq}&\frac{9999}{10000\sqrt{m}}.
    \end{align*}

\end{proof}

\begin{lemma}[Hitting time estimate]\label{lemma: hitting time estimate 1}
\[
T+1\geq T_e+1 \geq T^*:=\lfloor\frac{\log 6}{4\eta}\rfloor=44.
\]
\end{lemma}

\begin{proof}[Proof of Lemma \ref{lemma: hitting time estimate 1}]\ \\
Recalling the definition of the exponentially growing hitting time $T_e$ (\ref{equ: exponential hitting time def 1})
\begin{align*}
 T_e= \sup\Big\{t\in\mathbb{N}:\ 
    &\Big(\frac{1}{2}+2\sqrt{\frac{\log(2n^2/\delta)}{m}}\Big)\frac{251001\Big((1+2\eta)^{2(t+1)}-(1-2\eta)^{2(t+1)}\Big)}{1000000}\leq1;
    \\&{\rm and}\ (1+2\eta)^{t+1}\leq 2\sqrt{2}\Big\}.
\end{align*}

From (\ref{equ: parameter selection 1}) we know $
\sqrt{\frac{\log(2n^2/\delta)}{m}}\leq\frac{1}{12}$, so we have the estimate:

\begin{align*}
    &
    \Big(\frac{1}{2}+2\sqrt{\frac{\log(2n^2/\delta)}{m}}\Big)\frac{251001\Big((1+2\eta)^{2T^*}-(1-2\eta)^{2T^*}\Big)}{1000000}
    \\\leq&
    \Big(\frac{1}{2}+\frac{1}{6}\Big)\frac{251001}{1000000}\Big((1+2\eta)^{2T^*}-(1-2\eta)^{2T^*}\Big)
    \\=&
    \frac{251001}{3000000}\Big((1+2\eta)^{\frac{1}{2\eta}4\eta T^*}-(1-2\eta)^{\frac{1}{2\eta}4\eta T^*}\Big)
    \\\overset{\text{Lemma \ref{lemma: basic exp}}}{\leq}&
        \frac{251001}{1500000}\Big(e^{4\eta T^*}-\big(\frac{1}{e}\big)^{\frac{4\eta T^*}{1-2\eta}}\Big)
        \\\leq&
        \frac{251001}{1500000}\Big(6-\big(\frac{1}{6}\big)^{\frac{50}{49}}\Big)
        \\\leq&0.978<1;
\end{align*}
\begin{align*}
    (1+2\eta)^{T^*}\leq e^{2\eta T^*}\leq\sqrt{6}<2\sqrt{2}.
\end{align*}
So we have proved $T_e+1\geq T^*$. Combining $T_e\leq T$ in Lemma \ref{lemma: speed separation 1} (P5), we complete the proof.

\end{proof}

\subsection{Gradient lower bound}

The following Lemma is significantly different from the proof of lazy training in NTK regime like \cite{arora2019fine}. Our proof is based on the dynamical analysis of neuron partition.

\begin{lemma}[Gram matrix estimate]\label{lemma: positive Gram time t 1} Define the Gram matrix $\mathbf{G}(t)\in\mathbb{R}^{n\times n}$ at $t$ as
\[\mathbf{G}(t)=\Big(\nabla f(\boldsymbol{x}_i;\boldsymbol{\theta}(t))^\top\nabla f(\boldsymbol{x}_j;\boldsymbol{\theta}(t))\Big)_{(i,j)\in[n]\times[n]}.\]

Then for any $1\leq t\leq T$ and $i,j\in[\frac{n}{2}]$ we have
\[
\mathbf{G}(t)=\begin{pmatrix}
  \mathbf{G}_{+}(t) & \mathbf{0}_{\frac{n}{2}\times\frac{n}{2}}\\\mathbf{0}_{\frac{n}{2}\times\frac{n}{2}}&\mathbf{G}_{-}(t),
\end{pmatrix}
\]
and
\begin{gather*}
\text{G}_{+}(t)_{i,j}\geq \frac{999}{1000}\boldsymbol{x}_i^\top\boldsymbol{x}_j\Big(\frac{\pi-\arccos(\boldsymbol{x}_i^\top\boldsymbol{x}_j)}{\pi}-\sqrt{\frac{8\log(n^2/\delta)}{m}}\Big);
\\\text{G}_{-}(t)_{i,j}\geq \frac{999}{1000}\boldsymbol{x}_i^\top\boldsymbol{x}_j\Big(\frac{\pi-\arccos(\boldsymbol{x}_i^\top\boldsymbol{x}_j)}{\pi}-\sqrt{\frac{8\log(n^2/\delta)}{m}}\Big).
\end{gather*}
\end{lemma}

\begin{proof}[Proof of Lemma \ref{lemma: positive Gram time t 1}]\ \\
From the form of two-layer neural networks, We can calculate:
\begin{align*}
    &\nabla f(\boldsymbol{x}_i;\boldsymbol{\theta}(t))^\top\nabla f(\boldsymbol{x}_j;\boldsymbol{\theta}(t))
    \\=&
    \sum_{k=1}^m\Bigg(\left<\frac{\partial f(\boldsymbol{x}_i;\boldsymbol{\theta}(t))}{\partial a_k},\frac{\partial f(\boldsymbol{x}_j;\boldsymbol{\theta}(t))}{\partial a_k}\right>+\left<\frac{\partial f(\boldsymbol{x}_i;\boldsymbol{\theta}(t))}{\partial \boldsymbol{b}_k},\frac{\partial f(\boldsymbol{x}_j;\boldsymbol{\theta}(t))}{\partial \boldsymbol{b}_k}\right>\Bigg)
    \\=&
    \sum_{k=1}^m \Bigg(\sigma\Big(\boldsymbol{b}_k(t)^\top\boldsymbol{x}_i\Big)\sigma\Big(\boldsymbol{b}_k(t)^\top\boldsymbol{x}_j\Big)+
    a_k^2(t)\sigma'\Big(\boldsymbol{b}_k(t)^\top\boldsymbol{x}_i\Big)\sigma'\Big(\boldsymbol{b}_k(t)^\top\boldsymbol{x}_j\Big)\boldsymbol{x}_i^\top\boldsymbol{x}_j\Bigg)
    \\=&
    \sum_{k=1}^m\Big(\boldsymbol{b}_k(t)^\top\boldsymbol{x}_i\boldsymbol{b}_k(0)^\top\boldsymbol{x}_j+a_k^2(t)\boldsymbol{x}_i^\top\boldsymbol{x}_j\Big)\mathbb{I}\Big\{\boldsymbol{b}_k(t)^\top\boldsymbol{x}_i> 0\Big\}\mathbb{I}\Big\{\boldsymbol{b}_k(t)^\top\boldsymbol{x}_j>0\Big\}.
    \end{align*}
    
    For any $i,j$, s.t. $|i-j|\geq \frac{n}{2}$, we know $y_iy_j=-1$. From Lemma \ref{lemma: correct neurons remain correct 1}, we know:
    \begin{gather*}
    \boldsymbol{b}_k(t)^\top\boldsymbol{x}_i> 0
    \Rightarrow k\in\mathcal{TL}_i(t)\Rightarrow a_k(t)y_i>0;\\
     \boldsymbol{b}_k(t)^\top\boldsymbol{x}_j\geq 0
    \Rightarrow k\in\mathcal{TL}_j(t)\Rightarrow a_k(t)y_j>0,
    \end{gather*}
    which implies that:
    \[
    \Big\{k\in[m]:\boldsymbol{b}_k(t)^\top\boldsymbol{x}_i> 0,\boldsymbol{b}_k(t)^\top\boldsymbol{x}_j\geq 0\Big\}=\varnothing.
    \]
    So we have:
\[
\mathbf{G}(t)=\begin{pmatrix}
  \mathbf{G}_{+}(t) & \mathbf{0}_{\frac{n}{2}\times\frac{n}{2}}\\\mathbf{0}_{\frac{n}{2}\times\frac{n}{2}}&\mathbf{G}_{-}(t)
\end{pmatrix},\  \forall 1\leq t\leq T
\]

For any $i,j\in[\frac{n}{2}]$, we only need to consider $\text{G}_{+}(t)_{i,j}$. ( $\text{G}_{-}(t)_{i,j}$ is similar.)
    With the help of Lemma \ref{lemma: speed separation 1} and Lemma \ref{lemma: estimate of the number of initial neural partition 1}, we have the following estimate:
    \begin{align*}
    &\nabla f(\boldsymbol{x}_i;\boldsymbol{\theta}(t))^\top\nabla f(\boldsymbol{x}_j;\boldsymbol{\theta}(t))
    \\\geq
    &
    \sum_{k=1}^m a_k^2(t)\boldsymbol{x}_i^\top\boldsymbol{x}_j\mathbb{I}\Big\{\boldsymbol{b}_k(t)^\top\boldsymbol{x}_i> 0\Big\}\mathbb{I}\Big\{\boldsymbol{b}_k(t)^\top\boldsymbol{x}_j>0\Big\}
    \\=&
    \sum_{k\in\mathcal{TL}_i(t)\cap\mathcal{TL}_j(t)}a_k^2(t)\boldsymbol{x}_i^\top\boldsymbol{x}_j
\\\geq&
\Big(\frac{9999}{10000}\Big)^2\boldsymbol{x}_i^\top\boldsymbol{x}_j\text{card}\Big(\mathcal{TL}_i(t)\cap\mathcal{TL}_j(t)\Big)
\\\geq&
\Big(\frac{9999}{10000}\Big)^2\boldsymbol{x}_i^\top\boldsymbol{x}_j\text{card}\Big(\big(\mathcal{TL}_i(0)\cup\mathcal{TD}_i(0)\big)\bigcap\big(\mathcal{TL}_j(0)\cup\mathcal{TD}_j(0)\big)\Big)
\\\geq&
\frac{999}{1000}\boldsymbol{x}_i^\top\boldsymbol{x}_j\Big(\frac{\pi-\arccos(\boldsymbol{x}_i^\top\boldsymbol{x}_j)}{4\pi}-\sqrt{\frac{\log(n^2/\delta)}{2m}}\Big)\cdot 4.
\\=&\frac{999}{1000}\boldsymbol{x}_i^\top\boldsymbol{x}_j\Big(\frac{\pi-\arccos(\boldsymbol{x}_i^\top\boldsymbol{x}_j)}{\pi}-\sqrt{\frac{8\log(n^2/\delta)}{m}}\Big).
\end{align*}

\end{proof}

\begin{lemma}[Gradient lower bound]\label{lemma: gradient lower bound 1}
For any $1\leq t\leq T$ we have
\begin{align*}
&
\left\|\nabla \mathcal{L}(\boldsymbol{\theta}(t))\right\|^2\geq\frac{999}{1000}\Big(1-\varphi(t)\Big)^2\Big(\gamma_1-\gamma_2\sqrt{\frac{8\log(n^2/\delta)}{m}}\Big),
\end{align*}
where $\gamma_1,\gamma_2$ are defined in \eqref{equ: gamma definition}, and $\varphi(t)=\Big(\frac{1}{2}+2\sqrt{\frac{\log(2n^2/\delta)}{m}}\Big)\frac{251001\Big((1+2\eta)^{2t}-(1-2\eta)^{2t}\Big)}{1000000}$.
\end{lemma}

\begin{proof}[Proof of Lemma \ref{lemma: gradient lower bound 1}]\ \\
Combining Lemma \ref{lemma: correct neurons remain correct 1}, \ref{lemma: speed separation 1} and \ref{lemma: positive Gram time t 1}, we have the estimate for any $1\leq t\leq T$:
\begin{align*}
    &\left\|\nabla \mathcal{L}(\boldsymbol{\theta}(t))\right\|^2
    \\=&
    \frac{1}{n^2}\sum_{i\in[n]}\sum_{j\in[n]}\left<\nabla \ell_{i}(\boldsymbol{\theta}(t)),\nabla \ell_{j}(\boldsymbol{\theta}(t))\right>
    \\=&
    \frac{1}{n^2}\sum_{i\in[n]}\sum_{j\in[n]}\Big(f(\boldsymbol{x}_i;\boldsymbol{\theta}(t))-y_i\Big)\Big(f(\boldsymbol{x}_j;\boldsymbol{\theta}(t))-y_j\Big)\left<\nabla f(\boldsymbol{x}_i;\boldsymbol{\theta}(t)),\nabla f(\boldsymbol{x}_j;\boldsymbol{\theta}(t))\right>\\
    =&\frac{1}{n^2}\sum_{i\in[\frac{n}{2}]}\sum_{j\in[\frac{n}{2}]}\Big(f(\boldsymbol{x}_i;\boldsymbol{\theta}(t))-1\Big)\Big(f(\boldsymbol{x}_j;\boldsymbol{\theta}(t))-1\Big)\left<\nabla f(\boldsymbol{x}_i;\boldsymbol{\theta}(t)),\nabla f(\boldsymbol{x}_j;\boldsymbol{\theta}(t))\right>
    \\&\quad\quad+\frac{1}{n^2}\sum_{i\in[n]-[\frac{n}{2}]}\sum_{j\in[n]-[\frac{n}{2}]}\Big(f(\boldsymbol{x}_i;\boldsymbol{\theta}(t))+1\Big)\Big(f(\boldsymbol{x}_j;\boldsymbol{\theta}(t))+1\Big)\left<\nabla f(\boldsymbol{x}_i;\boldsymbol{\theta}(t)),\nabla f(\boldsymbol{x}_j;\boldsymbol{\theta}(t))\right>
    \\
    \overset{\text{Lemma \ref{lemma: speed separation 1}}}{\geq}&\Big(1-\varphi(t)\Big)^2\frac{1}{n^2}\Bigg(\sum_{i\in[\frac{n}{2}]}\sum_{j\in[\frac{n}{2}]}\left<\nabla f(\boldsymbol{x}_i;\boldsymbol{\theta}(t)),\nabla f(\boldsymbol{x}_j;\boldsymbol{\theta}(t))\right>
    \\&\quad\quad\quad\quad\quad\quad\quad
    +\sum_{i\in[n]-[\frac{n}{2}]}\sum_{j\in[n]-[\frac{n}{2}]}\left<\nabla f(\boldsymbol{x}_i;\boldsymbol{\theta}(t)),\nabla f(\boldsymbol{x}_j;\boldsymbol{\theta}(t))\right>\Bigg)
    \\
    \overset{\eqref{equ: gamma definition},\text{ Lemma \ref{lemma: positive Gram time t 1}}}{=}&\frac{999}{1000}\Big(1-\varphi(t)\Big)^2\Big(\gamma_1-\gamma_2\sqrt{\frac{8\log(n^2/\delta)}{m}}\Big),
\end{align*}
where $\varphi(t)=\Big(\frac{1}{2}+2\sqrt{\frac{\log(2n^2/\delta)}{m}}\Big)\frac{251001\Big((1+2\eta)^{2t}-(1-2\eta)^{2t}\Big)}{1000000}$.

\end{proof}

\subsection{Early stage convergence}
\begin{lemma}[Hessian upper bound]\label{lemma: Hessian upper bound ERM 1}  
$T^*$ is defined in Lemma \ref{lemma: hitting time estimate 1}. Consider all parameters along the trajectory of GD \eqref{equ: alg GD} for $t\geq1$:
\[
\mathcal{S}(T^*):=\bigcup\limits_{1\leq t\leq T^*-1}\Big\{\boldsymbol{\theta}:\boldsymbol{\theta}\in\overline{\boldsymbol{\theta}(t)\boldsymbol{\theta}({t+1})}\Big\}
\]
then for any $\boldsymbol{\theta}\in\mathcal{S}(T^*)$, we have:
\[
\left\|\nabla^2 \mathcal{L}(\boldsymbol{\theta})\right\|
\leq \frac{7}{2m}+2\leq 3.
\]

\end{lemma}

\begin{proof}[Proof of Lemma \ref{lemma: Hessian upper bound ERM 1}]\ \\
For any $\boldsymbol{\theta}\in\mathcal{S}(T^*)$, we can define:
\begin{gather*}
    \mathbf{H}_A(\boldsymbol{\theta}):=\frac{1}{n}
    \sum_{i=1}^n \nabla f(\boldsymbol{x}_i;\boldsymbol{\theta})\nabla f(\boldsymbol{x}_i;\boldsymbol{\theta})^\top,
    \\
    \mathbf{H}_B(\boldsymbol{\theta}):=\frac{1}{n}\sum_{i=1}^n\Big(f(\boldsymbol{x}_i;\boldsymbol{\theta})-y_i\Big)\nabla^2 f(\boldsymbol{x}_i;\boldsymbol{\theta}).
\end{gather*}
then we have the relationship:
\begin{align*}
    \nabla^2 \mathcal{L}(\boldsymbol{\theta})
    = \mathbf{H}_A(\boldsymbol{\theta}) + \mathbf{H}_B(\boldsymbol{\theta})
\end{align*}

Now we only need to estimate $\left\|\mathbf{H}_A(\boldsymbol{\theta})\right\|$ and $\left\|\mathbf{H}_B(\boldsymbol{\theta})\right\|$ respectively.

From the definition of $\boldsymbol{\theta}\in\mathcal{S}(T^*)$ and Lemma \ref{lemma: speed separation 1} (P1)(P2), we obtain the consistent estimate for any $\boldsymbol{\theta}\in \mathcal{S}(T)$:
\begin{gather*}
    |a_k|\leq\frac{501\Big((1+2\eta)^{T^*}+(1-2\eta)^{T^*}\Big)}{1000\sqrt{m}}\leq\frac{501}{1000\sqrt{m}}\Big(\sqrt{6}+\frac{1}{\sqrt{6}}\Big)\leq\sqrt{\frac{2}{m}},
    \\
    \left\|\boldsymbol{b}_k\right\|\leq\frac{501\Big((1+2\eta)^{T^*}-(1-2\eta)^{T^*}\Big)}{1000\sqrt{m}}\leq\frac{\sqrt{6}}{2\sqrt{m}}=\sqrt{\frac{3}{2m}}.
\end{gather*}

For the first part $\mathbf{H}_A(\boldsymbol{\theta})$, we have:
\begin{align*}
    \left\|\mathbf{H}_A(\boldsymbol{\theta})\right\|
    =&
    \sup\limits_{\boldsymbol{w}\ne 0}\frac{\frac{1}{n}\sum\limits_{i=1}^n\boldsymbol{w}^\top\nabla f(\boldsymbol{x}_i;\boldsymbol{\theta})\nabla f(\boldsymbol{x}_i;\boldsymbol{\theta})^\top\boldsymbol{w}}{\left\|\boldsymbol{w}\right\|^2}
    \\=&
    \sup_{u_1,\boldsymbol{v}_1,\cdots,u_m,\boldsymbol{v}_m {\rm\ not\  all\ }0}\frac{1}{n}\sum_{i=1}^n\frac{\sum\limits_{k=1}^m\Big(u_k\sigma(\boldsymbol{b}_k^\top\boldsymbol{x}_i)+\boldsymbol{v}_k^\top\boldsymbol{x}_i\sigma'(\boldsymbol{b}_k^\top\boldsymbol{x}_i)a_k\Big)^2}{\sum\limits_{k=1}^m\Big(u_k^2+\left\|\boldsymbol{v}_k\right\|^2\Big)}
    \\\leq&
        \sup_{u_1,\boldsymbol{v}_1,\cdots,u_m,\boldsymbol{v}_m {\rm\ not\  all\ }0}\frac{\sum\limits_{k=1}^m\Big(u_k\left\|\boldsymbol{b}_k\right\|+\left\|\boldsymbol{v}_k\right\||a_k|\Big)^2}{\sum\limits_{k=1}^m\Big(u_k^2+\left\|\boldsymbol{v}_k\right\|^2\Big)}
        \\\leq&
 \sup_{u_1,\boldsymbol{v}_1,\cdots,u_m,\boldsymbol{v}_m {\rm\ not\  all\ }0}\frac{\sum\limits_{k=1}^m\Big(\frac{2}{m}u_k^2+\frac{2\sqrt{3}}{m}|u_k|\left\|\boldsymbol{v}_k\right\|+\frac{3}{2m}\left\|\boldsymbol{v}_k\right\|^2\Big)}{\sum\limits_{k=1}^m\Big(u_k^2+\left\|\boldsymbol{v}_k\right\|^2\Big)}
 \\\leq&
  \sup_{u_1,\boldsymbol{v}_1,\cdots,u_m,\boldsymbol{v}_m {\rm\ not\  all\ }0}\frac{\sum\limits_{k=1}^m\Big(\frac{7}{2m}u_k^2+\frac{7}{2m}\left\|\boldsymbol{v}_k\right\|^2\Big)}{\sum\limits_{k=1}^m\Big(u_k^2+\left\|\boldsymbol{v}_k\right\|^2\Big)}=\frac{7}{2m}.
\end{align*}

For the second part $\mathbf{H}_B(\boldsymbol{\theta})$, we have:
\begin{align*}
\left\|\mathbf{H}_B(\boldsymbol{\theta})\right\|\leq&\Big(\max_{i\in[n]}\left|f(\boldsymbol{x}_i;\boldsymbol{\theta})-y_i\right|\Big)\Big(\max\limits_{i\in[n]}\left\|\nabla^2 f(\boldsymbol{x};\boldsymbol{\theta})\right\|\Big)
\\\leq&2\max\limits_{i\in[n]}\left\|\nabla^2 f(\boldsymbol{x}_i;\boldsymbol{\theta})\right\|.
\end{align*}
So we only need to estimate $\left\|\nabla^2 f(\boldsymbol{x}_i;\boldsymbol{\theta})\right\|$.

$\nabla^2 f(\boldsymbol{x}_i;\boldsymbol{\theta})$ is combined by $m$ diagonal block:
\[
\nabla^2 f(\boldsymbol{x}_i;\boldsymbol{\theta}):=\left(\begin{array}{ccc}  
    \mathbf{H}_f^{(1)} &  & \\
     & \ddots & \\ 
      &  &\mathbf{H}_f^{(m)}\\
  \end{array}\right),
\]
where the $k-$th block $\mathbf{H}_f^{(k)}$ is:
\begin{align*}
\mathbf{H}_f^{(k)}=\begin{pmatrix} 
    \frac{\partial^2 f(\boldsymbol{x}_i;\boldsymbol{\theta})}{\partial^2 a_k} & \frac{\partial^2 f(\boldsymbol{x}_i;\boldsymbol{\theta})}{\partial a_k{\partial \boldsymbol{b}_k}^\top}\\
   \frac{\partial^2 f(\boldsymbol{x}_i;\boldsymbol{\theta})}{\partial \boldsymbol{b}_k\partial a_k}  & \frac{\partial^2 f(\boldsymbol{x}_i;\boldsymbol{\theta})}{\partial \boldsymbol{b}_k{\partial \boldsymbol{b}_k}^\top} 
  \end{pmatrix}
  =
  \begin{pmatrix} 
    0 & \sigma'\Big(\boldsymbol{b}_k^\top\boldsymbol{x}_i\Big)\boldsymbol{x}_i^\top\\
   \sigma'\Big(\boldsymbol{b}_k^\top\boldsymbol{x}_i\Big)\boldsymbol{x}_i  & a_k\sigma''\Big(\boldsymbol{b}_k^\top\boldsymbol{x}_i\Big)\boldsymbol{x}_i\boldsymbol{x}_i^\top
  \end{pmatrix}.
\end{align*}

Recalling Lemma \ref{lemma: correct neurons remain correct 1} (S5), we know that: for any $\boldsymbol{\theta}=(a_1,\cdots,a_m,\boldsymbol{b}_1^\top,\cdots,\boldsymbol{b}_m^\top)^\top\in\mathcal{S}(T^*)$, we have $\boldsymbol{b}_k^\top\boldsymbol{x}_i\ne0$ for any $k\in[m]$ and $i\in[n]$. Hence, $\sigma''(\boldsymbol{b}_k^\top\boldsymbol{x}_i)=0$ holds strictly. So we have the form
\begin{align*}
\mathbf{H}_f^{(k)}=\begin{pmatrix} 
    0 & \sigma'\Big(\boldsymbol{b}_k^\top\boldsymbol{x}_i\Big)\boldsymbol{x}_i^\top\\
   \sigma'\Big(\boldsymbol{b}_k^\top\boldsymbol{x}_i\Big)\boldsymbol{x}_i  & \mathbf{0}_{d\times d}
  \end{pmatrix}.
\end{align*}

So we can estimate:
\begin{align*}
    \left\|\nabla^2 f(\boldsymbol{x}_i;\boldsymbol{\theta})\right\|
    =&\sup\limits_{\boldsymbol{v}_1,\cdots,\boldsymbol{v}_m {\rm are\ not\ all\ 0}}\frac{\sum\limits_{k=1}^m\boldsymbol{v}_k^\top\mathbf{H}_f^{(k)}\boldsymbol{v}_k }{\sum\limits_{k=1}^m\left\|\boldsymbol{v}_k\right\|^2}
    \\\leq&
    \max_{k\in[m]}\left\|\mathbf{H}_f^{(k)}\right\|
    =\max_{i\in[n]}\sup_{(u,\boldsymbol{v})\ne(0,\mathbf{0})}\frac{2\boldsymbol{v}^\top\boldsymbol{x}_i u}{u^2+\left\|\boldsymbol{v}\right\|^2}
    \\\leq&
    \sup_{(u,\boldsymbol{v})\ne(0,\mathbf{0})}\frac{2\left\|\boldsymbol{v}\right\||u|}{u^2+\left\|\boldsymbol{v}\right\|^2}\leq 1.
\end{align*}

So we have the estimate of the second part:
\begin{align*}
\left\|\mathbf{H}_B(\boldsymbol{\theta})\right\|\leq
2\max_{i\in[n]}\left\|\nabla^2 f(\boldsymbol{x}_i;\boldsymbol{\theta})\right\|\leq 2.
\end{align*}

Combining the two estimates, we obtain this lemma:
\[
\left\|\nabla^2 \mathcal{L}(\boldsymbol{\theta})\right\|\leq
\left\|\mathbf{H}_A(\boldsymbol{\theta})\right\|+
\left\|\mathbf{H}_B(\boldsymbol{\theta})\right\|\leq\frac{7}{2m}+2
\overset{\eqref{equ: parameter selection 1}}{\leq}3.
\]

\end{proof}

\begin{lemma}[Quadratic upper bound of loss]\label{lemma: loss quadratic upper bound 1} Let $\boldsymbol{\theta}(t)$ be trained by GD or SGD, then we have the upper bound of loss:
\[
    \mathcal{L}(\boldsymbol{\theta}(t+1))
    \leq
    \mathcal{L}(\boldsymbol{\theta}(t))-\eta_t\Big(1-\frac{\eta_t H_t}{2}\Big)\left\|\nabla \mathcal{L}(\boldsymbol{\theta}(t))\right\|^2,
\]
where $H_t=\sup\limits_{\boldsymbol{\theta}\in\overline{\boldsymbol{\theta}(t)\boldsymbol{\theta}(t+1)}}\left\|\nabla^2 \mathcal{L}(\boldsymbol{\theta})\right
\|$.
\end{lemma}
\begin{proof}[Proof of Lemma \ref{lemma: loss quadratic upper bound 1}]\ \\
\begin{align*}
&\mathcal{L}(\boldsymbol{\theta}(t+1))-\mathcal{L}(\boldsymbol{\theta}(t))-\left<\nabla \mathcal{L}(\boldsymbol{\theta}(t)),\boldsymbol{\theta}(t+1)-\boldsymbol{\theta}(t)\right>
\\
=&\int_{0}^1 \left<\nabla \mathcal{L}\big(\boldsymbol{\theta}(t)+\alpha(\boldsymbol{\theta}(t+1)-\boldsymbol{\theta}(t))\big)-\nabla \mathcal{L}(\boldsymbol{\theta}(t))
,\Theta(t+1)-\Theta(t)\right>{\mathrm d}\alpha
\\
\leq&
\int_{0}^1 \left\|\nabla \mathcal{L}\big(\boldsymbol{\theta}(t)+\alpha(\boldsymbol{\theta}(t+1)-\boldsymbol{\theta}(t))\big)-\nabla \mathcal{L}(\boldsymbol{\theta}(t))
\right\|\left\|\boldsymbol{\theta}(t+1)-\boldsymbol{\theta}(t)\right\|{\mathrm d}\alpha
\\\leq&
\frac{1}{2}\Big(\sup\limits_{\boldsymbol{\theta}\in\overline{\boldsymbol{\theta}(t)\boldsymbol{\theta}(t+1)}}\left\|\nabla^2 \mathcal{L}(\boldsymbol{\theta})\right
\|\Big)\left\|\boldsymbol{\theta}(t+1)-\boldsymbol{\theta}(t)\right\|^2.
\end{align*}
So for Gradient Descent, we have
\begin{align*}
    \mathcal{L}(\boldsymbol{\theta}(t+1))
    \leq&\mathcal{L}(\boldsymbol{\theta}(t))-\eta_t\left\|\nabla \mathcal{L}(\boldsymbol{\theta}(t))\right\|^2+\frac{\eta_t^2}{2}H_t\left\|\nabla \mathcal{L}(\boldsymbol{\theta}(t))\right\|^2
    \\=&
    \mathcal{L}(\boldsymbol{\theta}(t))-\eta_t\Big(1-\frac{\eta_t H_t}{2}\Big)\left\|\nabla \mathcal{L}(\boldsymbol{\theta}(t))\right\|^2,
\end{align*}
where $H_t=\sup\limits_{\boldsymbol{\theta}\in\overline{\boldsymbol{\theta}(t)\boldsymbol{\theta}(t+1)}}\left\|\nabla^2 \mathcal{L}(\boldsymbol{\theta})\right
\|$.

\end{proof}

\begin{lemma}\label{lemma: loss descent estimate 1}
If we define the following sequence \[\phi(t)=\frac{251001\Big((1+2\eta)^{2t}-(1-2\eta)^{2t}\Big)}{1500000}\ (t\leq T^*),\] then we have:
\[
\sum_{t=1}^{T^*-1}\eta\Big(1-\phi(t)\Big)^2\geq0.1965.
\]
\end{lemma}

\begin{proof}[Proof of Lemma \ref{lemma: loss descent estimate 1}]\ \\
For simplify, we denote $a=\frac{251001}{1500000}$, then we have:
\begin{align*}
    &\sum_{t=1}^{T^*-1}\eta\Big(1-\phi(t)\Big)^2
    \\=&
    \sum_{t=1}^{T^*-1}\eta\Bigg(1-a(1+2\eta)^{2t}+a(1-2\eta)^{2t}\Bigg)^2
    \\=&
    \eta\sum_{t=1}^{T^*-1}\Bigg(1-2a(1+2\eta)^{2t}+2a(1-2\eta)^{2t}+a^2(1+2\eta)^{4t}+a^2(1-2\eta)^{4t}-2a^2(1-4\eta^2)^{2t}\Bigg)
    \\=&\eta (T^*-2)+a\frac{(1+2\eta)^2-(1+2\eta)^{2T^*}}{2+2\eta}+a\frac{(1-2\eta)^2-(1-2\eta)^{2T^*}}{2-2\eta}+a^2\frac{(1+2\eta)^4-(1+2\eta)^{4T^*}}{1-(1+2\eta)^4}
    \\&\quad+a^2\frac{(1-2\eta)^4-(1-2\eta)^{4T^*}}{1-(1-2\eta)^4}-2a^2\frac{(1-4\eta^2)^2-(1-4\eta^2)^{2T^*}}{1-(1-4\eta^2)^2}.
\end{align*}
Substitute $
\eta =\frac{1}{100}$, $T^*=44$ and $a=\frac{251001}{1500000}$ into RHS of the equation above, we obtain the result:
\[
\sum_{t=1}^{T^*-1}\eta\Big(1-\phi(t)\Big)^2=0.19659127915806962\geq 0.1965.
\]
\end{proof}

\begin{lemma}\label{lemma: gradient upper bound 1} Under the same conditions of Theorem \ref{thm: binary quadratic}, we have a rough estimate of loss at $t=0$:
\begin{equation*}
    \mathcal{L}(\boldsymbol{\theta}(1))\leq\mathcal{L}(\boldsymbol{\theta}(0))+1.005\eta(1+5\eta)^2(1+4\eta^2).
\end{equation*}
\end{lemma}

\begin{proof}[Proof of Lemma \ref{lemma: gradient upper bound 1}]\ \\
First, for any $\boldsymbol{\theta}\in\overline{\boldsymbol{\theta}(0)\boldsymbol{\theta}(1)}$, we have the estimate of gradient:

\begin{equation*}
    \begin{aligned}
    \left\|\nabla\mathcal{L}(\boldsymbol{\theta})\right\|
    \leq&\frac{1}{n}\sum_{i=1}^n\left|f(\boldsymbol{x}_i;\boldsymbol{\theta})-y_i\right|\left\|\nabla f(\boldsymbol{x}_i;\boldsymbol{\theta})\right\|
    \\
    \leq&\frac{1}{n}\sum_{i=1}^n\left|f(\boldsymbol{x}_i;\boldsymbol{\theta})-y_i\right|\Bigg(\sum_{k=1}^m\Big(\sigma(\boldsymbol{b}_k^\top\boldsymbol{x}_i)\Big)^2+\sum_{k=1}^m a_k^2\left\|\sigma'(\boldsymbol{b}_k^\top\boldsymbol{x}_i)\boldsymbol{x}_i\right\|^2\Bigg)^{1/2}
    \\\leq&
    \frac{1}{n}\sum_{i=1}^n\left|f(\boldsymbol{x}_i;\boldsymbol{\theta})-y_i\right|\left\|\boldsymbol{\theta}\right\|
    \\\overset{\text{Lemma\  \ref{lemma: speed separation 1}}}{\leq}& 1.002(1+5\eta)\sqrt{1+4\eta^2}.
    \end{aligned}
\end{equation*}
Then we have:
\begin{equation*}
    \begin{aligned}
    \mathcal{L}(\boldsymbol{\theta}(1))
    \leq&\mathcal{L}(\boldsymbol{\theta}(0))+\Big(\sup_{\boldsymbol{\theta}\in\overline{\boldsymbol{\theta}(0)\boldsymbol{\theta}(1)}}\left\|\nabla\mathcal{L}(\boldsymbol{\theta})\right\|\Big)\left\|\boldsymbol{\theta}(1)-\boldsymbol{\theta}(0)\right\|
    \\\leq&
    \mathcal{L}(\boldsymbol{\theta}(0))+\eta\Big(\sup_{\boldsymbol{\theta}\in\overline{\boldsymbol{\theta}(0)\boldsymbol{\theta}(1)}}\left\|\nabla\mathcal{L}(\boldsymbol{\theta})\right\|\Big)^2
    \\\leq&
    \mathcal{L}(\boldsymbol{\theta}(0))+1.005\eta(1+5\eta)^2(1+4\eta^2).
    \end{aligned}
\end{equation*}

\end{proof}


\begin{proof}[\bfseries\color{blue} Proof of Theorem \ref{thm: binary quadratic}]\ \\
First, we want to clarify that all the notations that appear here are used in the proofs of Lemmas \ref{lemma: speed separation 1}, \ref{lemma: hitting time estimate 1}, \ref{lemma: positive Gram time t 1}, \ref{lemma: gradient lower bound 1}, \ref{lemma: Hessian upper bound ERM 1}, \ref{lemma: loss quadratic upper bound 1}, \ref{lemma: loss descent estimate 1}, and \ref{lemma: gradient upper bound 1}.

We have the following estimate:
\begin{align*}
    &\mathcal{L}(\boldsymbol{\theta}(T^*))
\\\overset{\text{Lemma \ref{lemma: loss quadratic upper bound 1}}}{\leq}&
\mathcal{L}(\boldsymbol{\theta}(T^*-1))-\eta\left\|\nabla \mathcal{L}(\boldsymbol{\theta}(T^*-1))\right\|^2+\frac{\eta^2}{2}\Big(\sup\limits_{\boldsymbol{\theta}\in\mathcal{S}(T^*)}\left\|\nabla^2 \mathcal{L}(\boldsymbol{\theta})\right
\|\Big)\left\|\nabla \mathcal{L}(\boldsymbol{\theta}(T^*-1))\right\|^2
\\\overset{\text{Lemma \ref{lemma: Hessian upper bound ERM 1}}}{\leq}&
\mathcal{L}(\boldsymbol{\theta}(T^*-1))-\eta(1-\frac{3\eta}{2})\left\|\nabla \mathcal{L}(\boldsymbol{\theta}(T^*-1))\right\|^2
\\\overset{\eqref{equ: parameter selection 1}}{\leq}&
\mathcal{L}(\boldsymbol{\theta}(T^*-1))-\frac{197\eta}{200}\left\|\nabla \mathcal{L}(\boldsymbol{\theta}(T^*-1))\right\|^2
\\\overset{\text{Lemma \ref{lemma: gradient lower bound 1}}}{\leq}&
\mathcal{L}(\boldsymbol{\theta}(T^*-1))-\frac{197\eta}{200}
\frac{999}{1000}\Big(1-\varphi(T^*-1)\Big)^2\Big(\gamma_1-\gamma_2\sqrt{\frac{8\log(n^2/\delta)}{m}}\Big)
\\\leq&\mathcal{L}(\boldsymbol{\theta}(T^*-1))-\frac{197\eta}{200}
\frac{999}{1000}\Big(1-\phi(T^*-1)\Big)^2\Big(\gamma_1-\gamma_2\sqrt{\frac{\log(8n^2/\delta)}{m}}\Big)
\\\leq&\cdots
\\\leq&
\mathcal{L}(\boldsymbol{\theta}(1))-\frac{196803}{200000}\Big(\gamma_1-\gamma_2\sqrt{\frac{8\log(n^2/\delta)}{m}}\Big)\sum_{t=1}^{T^*-1}\Big(1-\phi(t)\Big)^2
\\\overset{\text{Lemma \ref{lemma: gradient upper bound 1}}}{\leq}&
\mathcal{L}(\boldsymbol{\theta}(0))+1.005\eta(1+5\eta)^2(1+4\eta^2)-\frac{196803}{200000}\Big(\gamma_1-\gamma_2\sqrt{\frac{8\log(n^2/\delta)}{m}}\Big)\sum_{t=1}^{T^*-1}\Big(1-\phi(t)\Big)^2
\\\leq&
\mathcal{L}(\boldsymbol{\theta}(0))+0.0111-\frac{196803}{200000}\Big(\gamma_1-\gamma_2\sqrt{\frac{8\log(n^2/\delta)}{m}}\Big)\sum_{t=1}^{T^*-1}\eta\Big(1-\phi(t)\Big)^2
\\\overset{\text{Lemma \ref{lemma: loss descent estimate 1}}}{\leq}&
\mathcal{L}(\boldsymbol{\theta}(0))+0.0111-\frac{196803}{200000}\frac{1965}{10000}\Big(\gamma_1-\gamma_2\sqrt{\frac{8\log(n^2/\delta)}{m}}\Big)
\\\leq&
\mathcal{L}(\boldsymbol{\theta}(0))+0.0111-\frac{193}{1000}\Big(\gamma_1-\gamma_2\sqrt{\frac{8\log(n^2/\delta)}{m}}\Big)
.
\end{align*}

So in the first $44$ iterations, the loss will descend
\begin{align*}
&L\big(\boldsymbol{\theta}(0)\big)-L\big(\boldsymbol{\theta}(T)\big)\geq\frac{193}{1000}\Big(\gamma_1-\gamma_2\sqrt{\frac{8\log(n^2/\delta)}{m}}\Big)-0.0111
\\=&\Omega\Big(\gamma_1-\gamma_2\sqrt{\frac{8\log(n^2/\delta)}{m}}\Big)=\Omega\big(1-\sqrt{\log (n/\delta)/m}\big).
\end{align*}
Now we have proved the theorem for the specific number $\eta,\kappa,m$ \eqref{equ: parameter selection 1}: $\eta=\frac{1}{100}$, $\kappa\leq\min\Big\{\frac{1}{1000},\frac{\eta}{2000},\frac{\eta}{3n},\frac{\eta\mu_0}{3n}\Big\}$ and $m\geq \max\Big\{144\log(2n^2/\delta),4\Big\}$.
In the same way, our result also holds for any other $\eta, \kappa,m$ s.t. $\eta\leq\frac{1}{100}$, $\kappa\leq\min\Big\{\frac{1}{1000},\frac{\eta}{2000},\frac{\eta}{3n},\frac{\eta\mu_0}{3n}\Big\}$ and $m\geq \max\Big\{144\log(2n^2/\delta),4\Big\}$, so we complete our proof.
\end{proof}

\newpage

\section{Proof details of Theorem \ref{thm: one-hot}}\label{appendix: proof thm2}

\subsection{Preparation}

\noindent{\bfseries One-step updating.}
\begin{equation}\label{equ: one step update 2}
\begin{gathered}
     \boldsymbol{a}_k(t+1)
    =
    \boldsymbol{a}_k(t)-\eta\frac{1}{B}\sum_{i\in\mathcal{B}_t}\tilde{\ell}'\Big(\boldsymbol{y}_i^\top\boldsymbol{f}(\boldsymbol{x}_i;\boldsymbol{\theta}(t))
    \Big)\sigma\Big(\boldsymbol{b}_k^\top(t)\boldsymbol{x}_i+c_k(t)\Big)\boldsymbol{y}_i;
    \\
    \boldsymbol{b}_k(t+1)=
    \boldsymbol{b}_k(t)-\eta\frac{1}{B}\sum_{i\in\mathcal{B}_t}\tilde{\ell}'\Big(\boldsymbol{y}_i^\top\boldsymbol{f}(\boldsymbol{x}_i;\boldsymbol{\theta}(t))
    \Big)\sigma'\Big(\boldsymbol{b}_k^\top(t)\boldsymbol{x}_i+c_k(t)\Big)\boldsymbol{y}_i^\top\boldsymbol{a}_k(t)\boldsymbol{x}_i
        \\
    c_k(t+1)=
    c_k(t)-\eta\frac{1}{B}\sum_{i\in\mathcal{B}_t}\tilde{\ell}'\Big(\boldsymbol{y}_i^\top\boldsymbol{f}(\boldsymbol{x}_i;\boldsymbol{\theta}(t))
    \Big)\sigma'\Big(\boldsymbol{b}_k^\top(t)\boldsymbol{x}_i+c_k(t)\Big)\boldsymbol{y}_i^\top\boldsymbol{a}_k(t).
    \end{gathered}
\end{equation}

\noindent
\textbf{Hitting time.} We define the hitting time $T$ as:
\begin{equation}\label{equ: hitting time def 2}
\begin{aligned}
  T:= \sup\Big\{t\in\mathbb{N}:\ 
    &\max_{i\in[n]}f_{\alpha}(\boldsymbol{x}_i;\boldsymbol{\theta}(s))\leq1,\forall \alpha\in[C],
    \forall 0\leq s\leq t+1;
    \\& a_{k,\alpha}(s)a_{k,\alpha}(0)>0,\forall \alpha\in[C],\forall k\in[m],\forall 0\leq s\leq t+1\Big\}.
\end{aligned}
\end{equation}

To make our proof more clear, without loss of generality, we use the following specific numbers instead of progressive expression about $\eta\leq\mathcal{O}(1)$, $\kappa\leq\mathcal{O}(\eta/B)$ and $m\geq{\Omega}(\log(n/\delta))$. Specifically,
\begin{equation}\label{equ: parameter selection 2}
      \eta=\frac{1}{100},\quad
    \kappa\leq\min\Big\{\frac{\eta}{10},\frac{\eta}{3B}\Big\},\quad
m\geq 6,
\end{equation}
$z_0=1$, $g_{\min}=0.99$, $g_{\max}=1$, $h_{\max}=1$ in Assumption \ref{def: general loss} and $s=0$ in Assumption \ref{ass: data concentration}.

\subsection{Fine-grained dynamical analysis of neuron partition}
An important technique in our proof is the fine-grained analysis of each neuron during training.

For the $i$-th data, we divided all neurons into two categories and studied them separately.
We denote the true-living neurons $\mathcal{TL}_i(t)$ and the true-dead neurons $\mathcal{TD}_i(t)$ at time $t$ as:
\begin{gather*}
    \mathcal{TL}_{i}(t):=\Big\{k\in[m]:\boldsymbol{y}_i^\top\boldsymbol{a}_k(t)>0, \boldsymbol{b}_k(t)^\top\boldsymbol{x}_i+c_k(t)> 0\Big\},
    \\
    \mathcal{TD}_{i}(t):=\Big\{k\in[m]:\boldsymbol{y}_i^\top\boldsymbol{a}_k(t)>0, \boldsymbol{b}_k(t)^\top\boldsymbol{x}_i+c_k(t)\leq 0\Big\}.
\end{gather*}

It is easy to verify the following lemma about the partition of all neurons.

\begin{lemma}[Partition]\label{lemma: neuron partition 2}
For any $t\leq T$ and $i\in[n]$, we have:
\begin{gather*}
    [m]=\mathcal{TL}_{i}(t)\bigcup\mathcal{TD}_{i}(t).
\end{gather*}
\end{lemma}
\begin{proof}[Proof of Lemma \ref{lemma: neuron partition 2}]\ \\
From the definition of $T$, we know $\boldsymbol{a}_k(t)\ne \boldsymbol{0}$. Combining $\boldsymbol{a}_k(0)=(\frac{1}{\sqrt{m}},\cdots,\frac{1}{\sqrt{m}})^\top$, we obtain this Lemma.
\end{proof}

\begin{lemma}[Initial norm and prediction]\label{lemma: initial norm and prediction 2} With probability at least $1-4me^{-\frac{d+1}{2}}$ we have:
\begin{gather*}
\left\|\boldsymbol{b}_k(0)\right\|+|c_k(0)|\leq\frac{2\kappa}{\sqrt{m}},\ \forall k\in[m];\\
    |f_\alpha(\boldsymbol{x}_i;\boldsymbol{\theta}(0))|\leq2\kappa,\ \forall\alpha\in[C],\forall i\in[n].
\end{gather*}

\end{lemma}
\begin{proof}[Proof of Lemma \ref{lemma: initial norm and prediction 2}]\ \\
From ${\boldsymbol{b}}_k(0)\sim\mathcal{N}(\mathbf{0},\frac{\kappa^2}{m(d+1)}\mathbf{I}_d)$ and $c_k(0)=\frac{\kappa}{\sqrt{m(d+1)}})$, we have the probability inequalities:
\begin{gather*}
    \mathbb{P}\Big(\left\|\boldsymbol{b}_k(0)\right\|\leq\frac{\kappa}{\sqrt{m}}\Big)\geq 1 - 2\exp(-\frac{d+1}{2}).
\end{gather*}

Combining the uniform bound about $k\in[m]$, with probability at least $1-\frac{4m}{\exp(\frac{d+1}{2})}$ we have
\[
\left\|\boldsymbol{b}_k(0)\right\|+|c_k(0)|\leq\frac{2\kappa}{\sqrt{m}},\ \forall k\in[m].
\]

From $\boldsymbol{f}=(f_1,\cdots,f_C)^\top$ and the path norm estimate of $f_\alpha(\cdot;\cdot)$ ($\alpha\in[C]$), we have
\[
|f_\alpha(\boldsymbol{x}_i;\boldsymbol{\theta}(0))|\leq\sum_{k=1}^m|a_{k,\alpha}(0)||\boldsymbol{b}_k(0)^\top\boldsymbol{x}_i|\leq m\frac{1}{\sqrt{m}}\frac{2\kappa}{\sqrt{m}}=2\kappa,\ \forall i\in[n].
\]

\end{proof}

\begin{lemma}[Dynamics of neuron partition]\label{lemma: correct neurons remain correct 2}
When the events in Lemma \ref{lemma: initial norm and prediction 2} happened, we have the following results for any $t\leq T$ with probability at least $1-m 0.17^B$:\\
(S1) True-living neurons remain true-living:  $\mathcal{TL}_i(t)\subset\mathcal{TL}_i(t+1)$.\\
(S2) True-dead neurons turn true-living in the first step $(t=0)$: $\mathcal{TD}_i(0)\subset\mathcal{TL}_i(1)$.\\
(S3) For any $i\in[n]$, $1\leq t\leq T$, $k\in [m]$ and $(\boldsymbol{b}^\top,c)^\top\in\overline{(\boldsymbol{b}_k^\top(t),c_k(t))^\top\ (\boldsymbol{b}_k^\top(t+1),c_k(t+1))^\top}$,we have $\text{sgn}({b}_k^\top\boldsymbol{x}_i+c_k)\equiv\text{sgn}({b}_k^\top(1)\boldsymbol{x}_i+c_k(1))>0$.

\end{lemma}

\begin{proof}[Proof of Lemma \ref{lemma: correct neurons remain correct 2}]\ \\
For $t+1$, from one-step update rules (\ref{equ: one step update 2}), we have:

(Proof of S1)
Let $k\in\mathcal{TL}_i(t)$.
Then we prove a stronger result than (S1), ``For $k\in\mathcal{FL}_i(t)$, we have $\boldsymbol{b}^\top\boldsymbol{x}_i+c>0$ for any $(\boldsymbol{b}^\top,c)^\top\in\overline{(\boldsymbol{b}_k^\top(t),c_k(t))^\top\ (\boldsymbol{b}_k^\top(t+1),c_k(t+1))^\top}$.'':

There exists $\alpha\in[0,1]$, s.t. $\boldsymbol{b}=\alpha\boldsymbol{b}_k(t)+(1-\alpha)\boldsymbol{b}_k(t+1)$, $c=\alpha c_k(t)+(1-\alpha)c_k(t+1)$, and
\begin{align*}
    &\boldsymbol{b}_k^\top\boldsymbol{x}_i+c_k
    \\=&\alpha\Big(\boldsymbol{b}_k(t)^\top\boldsymbol{x}_i+c_k(t)\Big)+(1-\alpha)\Big(\boldsymbol{b}_k(t+1)^\top\boldsymbol{x}_i+c_k(t+1)\Big)
    \\=&\boldsymbol{b}_k(t)^\top\boldsymbol{x}_i+c_k(t)-(1-\alpha)\eta\frac{1}{B}\sum_{j\in\mathcal{B}_t}\tilde{\ell}'\Big(\boldsymbol{y}_j^\top\boldsymbol{f}(\boldsymbol{x}_j;\boldsymbol{\theta}(t))\Big)\mathbb{I}\big\{\boldsymbol{b}_k(t)^\top\boldsymbol{x}_j+c_k(t)>0\big\}\boldsymbol{y}_j^\top\boldsymbol{a}_k(t)\Big(\boldsymbol{x}_i^\top\boldsymbol{x}_j+1\Big)
    \\\geq&\boldsymbol{b}_k(t)^\top\boldsymbol{x}_i+c_k(0)>0.
\end{align*}

(Proof of S2)
Let $k\in\mathcal{TD}_i(0)$, we have
\begin{align*}
    &\boldsymbol{b}_k(1)^\top\boldsymbol{x}_i+c_k(1)
    \\=&\boldsymbol{b}_k(0)^\top\boldsymbol{x}_i+c_k(0)-\eta\frac{1}{B}\sum_{j\in\mathcal{B}_0}\tilde{\ell}'\Big(\boldsymbol{y}_j^\top\boldsymbol{f}(\boldsymbol{x}_j;\boldsymbol{\theta}(0))\Big)\mathbb{I}\big\{\boldsymbol{b}_k(0)^\top\boldsymbol{x}_j+c_k(0)>0\big\}\boldsymbol{y}_j^\top\boldsymbol{a}_k(0)\Big(\boldsymbol{x}_i^\top\boldsymbol{x}_j+1\Big)
    \\\overset{\text{Lemma \ref{lemma: initial norm and prediction 2}}}{\geq}&
   -\frac{2\kappa}{\sqrt{m}}+\frac{0.99\eta}{B\sqrt{m}}\sum_{j\in\mathcal{B}_0}\mathbb{I}\big\{\boldsymbol{b}_k(0)^\top\boldsymbol{x}_j+c_k(0)>0\big\}.
\end{align*}
Now we need to bound the probability of the event $\Big\{\sum_{j\in\mathcal{B}_0}\mathbb{I}\big\{\boldsymbol{b}_k(0)^\top\boldsymbol{x}_j+c_k(0)>0\big\}\geq1\Big\}$.

Let $c=\frac{\kappa}{\sqrt{m(d+1)}}$.
For any $\boldsymbol{b}\sim\mathcal{N}(0,\mathbf{I}_d)$ and $\boldsymbol{v}\in\mathbb{S}^{d-1}$, we know $\boldsymbol{b}^\top\boldsymbol{v}\sim\mathcal{N}(0,1)$, so $\mathbb{P}(\boldsymbol{b}^\top\boldsymbol{v}+c<0)\leq0.17$. Recalling $\boldsymbol{x}_1,\cdots,\boldsymbol{x}_n\overset{\text{i.i.d.}}{\sim}\mathcal{P}$ and $\boldsymbol{b}_k(0)\overset{\text{i.i.d.}}{\sim}\mathcal{N}(0,\frac{\kappa^2}{m(d+1)}\mathbf{I}_d)$, we have
\begin{align*}
    &\mathbb{P}\Bigg(\sum_{j\in\mathcal{B}_0}\mathbb{I}\big\{\boldsymbol{b}_k(0)^\top\boldsymbol{x}_j+c_k(0)>0\big\}\geq1,\ \forall k\in\mathcal{TD}_i(0)\Bigg)
    \\\geq&\mathbb{P}\Bigg(\sum_{j\in\mathcal{B}_0}\mathbb{I}\big\{\boldsymbol{b}_k(0)^\top\boldsymbol{x}_j+c_k(0)>0\big\}\geq1,\ \forall k\in[m]\Bigg)
    \\=&1-\mathbb{P}\Bigg(\boldsymbol{b}_k(0)^\top\boldsymbol{x}_j+c_k(0)\leq0,\ \exists k\in[m]\Bigg)
    \\\geq&1-\sum_{k\in[m]}\mathbb{P}\Bigg(\boldsymbol{b}_k(0)^\top\boldsymbol{x}_j+c_k(0)\leq0\Bigg)
    \\=&1-\sum_{k\in[m]}\Bigg(\mathbb{P}\Big(\boldsymbol{b}_k(0)^\top\boldsymbol{x}_j+c_k(0)\leq0\Big)\Bigg)^B\geq 1-m 0.17^B.
\end{align*}

Hence, with probability at least $1-m 0.17^B$, we have
\begin{align*}
    \boldsymbol{b}_k(1)^\top\boldsymbol{x}_i+c_k(1)\geq\frac{1}{\sqrt{m}}\Big(-2\kappa+\frac{0.99\eta}{B}\Big)\overset{\eqref{equ: parameter selection 2}}{>}0.
\end{align*}
So $\mathcal{TD}_i(0)\subset\mathcal{TL}_i(1)$.

(Proof of S3) Combining (S2) and the proof of (S1), we obtain (S3).

\end{proof}

\subsection{Estimate of hitting time and parameters}
\begin{remark}\rm
The Lemmas in this section will be discussed and proved when the events in Lemma \ref{lemma: initial norm and prediction 2} and Lemma \ref{lemma: correct neurons remain correct 2} happened. So all the lemmas below hold with probability at least $1-\delta-4me^{-\frac{d+1}{2}}-m 0.17^B$.
\end{remark}

First, we construct the following exponentially growing hitting time:
\begin{equation}\label{equ: exponential hitting time def 2}
\begin{aligned}
  T_e:= \sup\Big\{t\in\mathbb{N}:\ 
    &\frac{251001\Big((1+2\eta)^{2(t+1)}-(1-2\eta)^{2(t+1)}\Big)}{1000000}\leq1;
    \\&{\rm and}\ (1+2\eta)^{t+1}\leq 2\Big\}.
\end{aligned}
\end{equation}

\begin{lemma}[Parameter estimate]\label{lemma: speed separation 2}
For any $1\leq t\leq T_e+1$, $i\in[n]$, we have:
\begin{align*}
&\text{(P1).\ }\forall k\in\mathcal{TL}_i(t),\forall\alpha\in[C],\ a_{k,\alpha}(t)\leq\frac{501\Big((1+2\eta)^t+(1-2\eta)^t\Big)}{1000\sqrt{m}};\\
&\text{(P2).\ }\forall k\in\mathcal{TL}_i(t),\ 
    \left\|\boldsymbol{b}_k(t)\right\|+|c_k(t)|\leq\frac{501\Big((1+2\eta)^t-(1-2\eta)^t\Big)}{1000\sqrt{m}};\\
 &   \text{(P3).\ }\forall \alpha\in[C],\ 
    \left|f_\alpha\big(\boldsymbol{x}_i;\boldsymbol{\theta}(t)\big)\right|\leq\frac{251001\Big((1+2\eta)^{2t}-(1-2\eta)^{2t}\Big)}{1000000};\\
      &  \text{(P4).\ }\forall k\in[m]\ \forall \alpha\in[C],\ 
    a_{k,\alpha}(t)\geq a_{k,\alpha}(t-1)\geq\cdots\geq a_{k,\alpha}(0)\geq\frac{1}{\sqrt{m}}.
\end{align*}
\end{lemma}

\begin{proof}[Proof of Lemma \ref{lemma: speed separation 2}]\ \\
We will complete the proof in the order: (P4)$\Rightarrow$(P1)(P2)(P3).

\noindent\textbf{Step I.} Proof of (P4). 

Let $k\in\mathcal{TL}_i(t)$, we know that for any $0\leq t\leq T$,
\begin{align*}
    \boldsymbol{a}_k({t+1})=&\boldsymbol{a}_k(t)-\eta\frac{1}{B}\sum_{j\in\mathcal{B}_t}\tilde{\ell}'\Big(\boldsymbol{y}_j^\top\boldsymbol{f}(\boldsymbol{x}_j;\boldsymbol{\theta}(t))\Big)\sigma\Big(\boldsymbol{b}_k(t)^\top\boldsymbol{x}_j+c_k(t)\Big)\boldsymbol{y}_j
    \\\geq&\boldsymbol{a}_k(t)
    \geq\cdots\geq\boldsymbol{a}_k(0)=(\frac{1}{\sqrt{m}},\cdots,\frac{1}{\sqrt{m}})^\top,
\end{align*}
where $\geq$ applies to each coordinate.

\textbf{Step II.} Proof of (P1)(P2)(P3).

For convenience, we define the following sequences:
\begin{gather*}
    x_{k}^t = \max_{l\in[C]}a_{k,l}(t),k\in\mathcal{TL}_i(t);
    \\
    y_k^t=\left\|\boldsymbol{b}_k(t)\right\|+|c_k(t)|,k\in\mathcal{TL}_i(t);\\
    f_{i}^t = \max_{l\in[C]} \sum_{k=1}^m a_{k,l}(t)\sigma(\boldsymbol{b}_k(t)^\top\boldsymbol{x}_i+c_k(t)).
\end{gather*}
where $t\leq T_e+1,\ k\in[m],\ i\in[n]$.

From (P4), we know $\boldsymbol{y}_i^\top\boldsymbol{f}(\boldsymbol{x}_i;\boldsymbol{\theta})\geq0$ for any $i\in[n]$.
So for $0\leq t\leq T_e$, we have:
\begin{gather*}
    x_{k}^{t+1}\leq x_{k}^t +\eta g_{\max} y_k^t
    \leq x_{k}^t +\eta y_k^t<x_{k}^t +2\eta y_k^t, k\in\mathcal{TL}_i(t);
    \\
    y_k^{t+1}\leq y_k^t + 2\eta g_{\max} x_k^t\leq y_k^t + 2\eta x_k^t,k\in\mathcal{TL}_i(t).
\end{gather*}
For $1\leq t\leq T_e+1$, from Lemma \ref{lemma: neuron partition 2}, we have:
\begin{align*}
 f_{i}^{t}=&\max_{l\in[C]}\sum_{k=1}^m a_{k,l}(t)\sigma(\boldsymbol{b}_k(t)^\top\boldsymbol{x}_i+c_k(t))=\max_{l\in[C]}\sum_{k\in\mathcal{TL}_i(t)} a_{k,l}(t)\sigma(\boldsymbol{b}_k(t)^\top\boldsymbol{x}_i+c_k(t))
 \leq \sum_{k\in[m]}x_{k}^t y_k^t.
\end{align*}

If we define the following sequences:
\begin{gather*}
\alpha_{k}^0=x_k^0,\quad \beta_k^0=y_k^0,\quad \gamma_\alpha^0=\sum_{k=1}^m\alpha_{k}^0\beta_k^0;
\\
    \alpha_{k}^{t+1}=\alpha_{k,}^{t}+2\eta\beta_k^{t};
    \\
    \beta_k^{t+1}=\beta_k^t+2\eta\alpha_{k}^{t};
    \\
    \gamma^{t+1}=\max_{i\in[n]}\sum_{k\in\mathcal{TL}_i(t+1)}\alpha_{k}^{t+1}\beta_k^{t+1}.
\end{gather*}
It is easy to verify that: for any $1\leq t\leq T_e+1$,
\begin{equation}\label{equ: inequality of sequences 2}
\begin{gathered}
x_k^t\leq \alpha_k^t;\quad y_k^t\leq\beta_k^t;\\
\max_{i\in[n]}f_i^t\leq \gamma^t.
\end{gathered}
\end{equation}

Now we aim to prove the following three properties for $1\leq t\leq T_e+1$ by induction:
\begin{equation} \label{equ: induction 2}
    \begin{gathered}
  \alpha_k^t\leq\frac{501\Big((1+2\eta)^t+(1-2\eta)^t\Big)}{1000\sqrt{m}};
    \\
   \beta_k^t\leq\frac{(1+2\eta)^t-(1-2\eta)^t}{2\sqrt{m}};
    \\
    \gamma^t\leq\frac{501\Big((1+2\eta)^{2t}-(1-2\eta)^{2t}\Big)}{2000}\leq 1.
    \end{gathered}
\end{equation}

For $t=1$, with the one-step updating (\ref{equ: one step update 2}), we have:
\begin{gather*}
    \alpha_k^1=x_k^1\leq\frac{1}{\sqrt{m}}+\eta(1+2\kappa)\frac{2\kappa}{\sqrt{m}}\overset{(\ref{equ: parameter selection 2})}{\leq}\frac{10002}{10000\sqrt{m}},k\in\mathcal{TL}_i(1);
    \\
    \beta_k^1= y_k^1\leq\frac{2\kappa}{\sqrt{m}}+\eta(1+2\kappa)\frac{1}{\sqrt{m}}
    \overset{(\ref{equ: parameter selection 2})}{\leq}\frac{2\eta}{\sqrt{m}},k\in\mathcal{TL}_i(1);
    \\
    \gamma^1\leq m\frac{10002}{10000\sqrt{m}}\frac{2\eta}{\sqrt{m}}\leq1.
\end{gather*}

Assume (\ref{equ: induction 2}) holds for $s\leq t-1\ (t\geq2)$. For $s=t$, we have:
\begin{gather*}
    \alpha_k^{t}=\frac{(1+2\eta)^t+(1-2\eta)^t}{2}\alpha_k^0+\frac{(1+2\eta)^t-(1-2\eta)^t}{2}\beta_k^0,\\
    \beta_k^{t}=\frac{(1+2\eta)^t+(1-2\eta)^t}{2}\beta_k^0+\frac{(1+2\eta)^t-(1-2\eta)^t}{2}\alpha_k^0.
\end{gather*}

Combining the definition of $T_e$ (\ref{equ: exponential hitting time def 2}) and Lemma \ref{lemma: alpha increase}, we know:
\begin{align*}
    \alpha_k^{t}\leq&
    \frac{(1+2\eta)^t+(1-2\eta)^t}{2}(\alpha_k^0+\beta_k^0)\leq\frac{(1+2\eta)^t+(1-2\eta)^t}{2}\frac{(1+2\kappa)}{\sqrt{m}}
    \\\overset{\eqref{equ: parameter selection 2}}{\leq}&\frac{501\Big((1+2\eta)^t+(1-2\eta)^t\Big)}{1000\sqrt{m}};
    \\
    \beta_k^{t}
   \overset{\eqref{equ: exponential hitting time def 2}}{\leq}&
   \frac{4+4}{2}\frac{2\kappa}{\sqrt{m}}+\frac{(1+2\eta)^t-(1-2\eta)^t}{2\sqrt{m}}
   \\\overset{\eqref{equ: parameter selection 2}}{\leq}&\frac{4\eta}{1000\sqrt{m}}+\frac{(1+2\eta)^t-(1-2\eta)^t}{2\sqrt{m}}
    \\\leq&\frac{501\Big((1+2\eta)^t-(1-2\eta)^t\Big)}{1000\sqrt{m}};
\end{align*}
and
\begin{align*}
    &\gamma^{t}\leq
    m
    \frac{501\Big((1+2\eta)^t+(1-2\eta)^t\Big)}{1000\sqrt{m}}\frac{501\Big((1+2\eta)^t-(1-2\eta)^t\Big)}{1000\sqrt{m}}
     \\=&
     \frac{251001\Big((1+2\eta)^{2t}-(1-2\eta)^{2t}\Big)}{1000000}\overset{\eqref{equ: exponential hitting time def 2}}{\leq} 1.
\end{align*}

By induction, we complete the proof of (\ref{equ: induction 2}). Then with the analysis (\ref{equ: inequality of sequences 2}), we can obtain (P1)(P2)(P3) for all $1\leq t\leq T+1$:
\begin{gather*}
    x_k^t\leq\alpha_k^t\leq\frac{501\Big((1+2\eta)^t+(1-2\eta)^t\Big)}{1000\sqrt{m}};
    \\
    y_k^t\leq\beta_k^t\leq\frac{501\Big((1+2\eta)^t-(1-2\eta)^t\Big)}{1000\sqrt{m}};
\end{gather*}
\begin{align*}
    f_i^t\leq\gamma^t\leq&\frac{251001\Big((1+2\eta)^{2t}-(1-2\eta)^{2t}\Big)}{1000000}\leq1.
\end{align*}

\end{proof}

\begin{lemma}[Hitting time estimate]\label{lemma: hitting time estimate 2}
\[
T+1\geq T_e+1 \geq T^*:=\lfloor\frac{\log 4}{4\eta}\rfloor=34.
\]
\end{lemma}

\begin{proof}[Proof of Lemma \ref{lemma: hitting time estimate 2}]\ \\
Recalling the definition of the hitting time $T$ (\ref{equ: hitting time def 2}) and the exponentially growin hitting time $T_e$ (\ref{equ: exponential hitting time def 2}):
\begin{align*}
 T= \sup\Big\{t\in\mathbb{N}:\ 
    &\max_{i\in[n]}\left|f_{\alpha}(\boldsymbol{x}_i;\boldsymbol{\theta}(s))\right|\leq1,\forall \alpha\in[C],
    \forall 0\leq s\leq t+1;
    \\& a_{k,\alpha}(s)a_{k,\alpha}(0)>0,\forall \alpha\in[C],\forall k\in[m],\forall 0\leq s\leq t+1\Big\}.
\end{align*}
\begin{align*}
 T_e=\sup\Big\{t\in\mathbb{N}:\ 
    &\frac{251001\Big((1+2\eta)^{2(t+1)}-(1-2\eta)^{2(t+1)}\Big)}{1000000}\leq1;
    \\&{\rm and}\ (1+2\eta)^{t+1}\leq 2\Big\}.
\end{align*}

From the estimates:
\begin{align*}
&
    \frac{251001}{1000000}\Big((1+2\eta)^{2T^*}-(1-2\eta)^{2T^*}\Big)
    \\=&
    \frac{251001}{1000000}\Big((1+2\eta)^{\frac{1}{2\eta}4\eta T^*}-(1-2\eta)^{\frac{1}{2\eta}4\eta T^*}\Big)
    \\\overset{\text{Lemma \ref{lemma: basic exp}}}{\leq}&
        \frac{251001}{1000000}\Big(e^{4\eta T^*}-\big(\frac{1}{e}\big)^{\frac{4\eta T^*}{1-2\eta}}\Big)
        \\\leq&
        \frac{251001}{1000000}\Big(4-\big(\frac{1}{4}\big)^{\frac{50}{49}}\Big)
        \\\leq&0.944<1,
\end{align*}
\begin{align*}
    (1+2\eta)^{T^*}\leq e^{2\eta T^*}\leq2,
\end{align*}
 we have proved $T_e+1\geq T^*$.

From Lemma \ref{lemma: speed separation 2} (P3)(P4) and Lemma \ref{lemma: basic exp}, we can verify that for any $\alpha\in[C]$, $t\leq T_e+1$,
\begin{align*}
     \left|f_\alpha\big(\boldsymbol{x}_i;\boldsymbol{\theta}(t)\big)\right|\leq&\frac{251001\Big((1+2\eta)^{2t}-(1-2\eta)^{2t}\Big)}{1000000}
    \\\leq&\frac{251001\Big((1+2\eta)^{2( T_e+1)}-(1-2\eta)^{2( T_e+1)}\Big)}{1000000}
    \leq1,
\end{align*}
\begin{align*}
a_{k,\alpha}(t)a_{k,\alpha}(0)>0,\forall k\in[m].
\end{align*}
So we have $T\geq T_e$.

\end{proof}

\subsection{Gradient lower bound}

\begin{lemma}[Gram matrix estimate]\label{lemma: positive Gram time t 2} Define the Gram matrix $\mathbf{G}(t)\in\mathbb{R}^{Cn\times Cn}$ at $t$ as
\[\mathbf{G}(t)=\Big(\nabla \boldsymbol{f}(\boldsymbol{x}_i;\boldsymbol{\theta}(t))^\top\nabla\boldsymbol{f}(\boldsymbol{x}_j;\boldsymbol{\theta}(t))\Big)_{(i,j)\in[n]\times[n]},\]

Then for any $1\leq t\leq T$, $i,j\in[Cn]$, we have:
\begin{gather*}
\text{G}(t)_{i,j}\geq 1.
\end{gather*}
\end{lemma}

\begin{proof}[Proof of Lemma \ref{lemma: positive Gram time t 2}]\ \\
From the form of two-layer neural networks with bias, for any $\alpha,\beta\in[C]$, $i,j\in[n]$ we have
\begin{align*}
    &\nabla f_\alpha(\boldsymbol{x}_i;\boldsymbol{\theta}(t))^\top\nabla f_\beta(\boldsymbol{x}_j;\boldsymbol{\theta}(t))
    \\=&
    \sum_{k=1}^m\Bigg(\left<\frac{\partial f_\alpha(\boldsymbol{x}_i;\boldsymbol{\theta}(t))}{\partial\boldsymbol{ a}_k},\frac{\partial f_\beta(\boldsymbol{x}_j;\boldsymbol{\theta}(t))}{\partial\boldsymbol{ a}_k}\right>+\left<\frac{\partial f_\alpha(\boldsymbol{x}_i;\boldsymbol{\theta}(t))}{\partial \boldsymbol{b}_k},\frac{\partial f_\beta(\boldsymbol{x}_j;\boldsymbol{\theta}(t))}{\partial \boldsymbol{b}_k}\right>
    \\&\quad\quad\quad\quad\quad\quad\quad\quad\quad\quad\quad\quad\quad\quad\quad\quad\quad+\left<\frac{\partial f_\alpha(\boldsymbol{x}_i;\boldsymbol{\theta}(t))}{\partial c_k},\frac{\partial f_\beta(\boldsymbol{x}_j;\boldsymbol{\theta}(t))}{\partial c_k}\right>\Bigg)
    \\=&
    \sum_{k=1}^m\Bigg(\left<\frac{\partial f_\alpha(\boldsymbol{x}_i;\boldsymbol{\theta}(t))}{\partial a_{k,\alpha}},\frac{\partial f_\beta(\boldsymbol{x}_j;\boldsymbol{\theta}(t))}{\partial a_{k,\beta}}\right>+\left<\frac{\partial f_\alpha(\boldsymbol{x}_i;\boldsymbol{\theta}(t))}{\partial \boldsymbol{b}_k},\frac{\partial f_\beta(\boldsymbol{x}_j;\boldsymbol{\theta}(t))}{\partial \boldsymbol{b}_k}\right>
    \\&\quad\quad\quad\quad\quad\quad\quad\quad\quad\quad\quad\quad\quad\quad\quad\quad\quad+\left<\frac{\partial f_\alpha(\boldsymbol{x}_i;\boldsymbol{\theta}(t))}{\partial c_k},\frac{\partial f_\beta(\boldsymbol{x}_j;\boldsymbol{\theta}(t))}{\partial c_k}\right>\Bigg)
    \\=&
    \sum_{k=1}^m \Bigg(\sigma\Big(\boldsymbol{b}_k(t)^\top\boldsymbol{x}_i+c_k(t)\Big)\sigma\Big(\boldsymbol{b}_k(t)^\top\boldsymbol{x}_j+c_k(t)\Big)
    \\&\quad\quad\quad+
    a_{k,\alpha}(t)a_{k,\beta}(t)\mathbb{I}\Big\{\boldsymbol{b}_k(t)^\top\boldsymbol{x}_i+c_k(t)\geq0\Big\}\mathbb{I}\big\{\boldsymbol{b}_k(t)^\top\boldsymbol{x}_j+c_k(t)>0\big\}\Big(\boldsymbol{x}_i^\top\boldsymbol{x}_j+1\Big)\Bigg)
    \\=&
     \sum_{k\in\mathcal{TL}_i(t)\cap\mathcal{TL}_j(t)} \Bigg(\Big(\boldsymbol{b}_k(t)^\top\boldsymbol{x}_i+c_k(t)\Big)\Big(\boldsymbol{b}_k(t)^\top\boldsymbol{x}_j+c_k(t)\Big)+
    a_{k,\alpha}(t)a_{k,\beta}(t)\Big(\boldsymbol{x}_i^\top\boldsymbol{x}_j+1\Big)\Bigg)
    \\\geq&
     \sum_{k\in\mathcal{TL}_i(t)\cap\mathcal{TL}_j(t)}
    a_{k,\alpha}(t)a_{k,\beta}(t)\Big(\boldsymbol{x}_i^\top\boldsymbol{x}_j+1\Big)
    \\\overset{\text{Lemma \ref{lemma: correct neurons remain correct 2}}}{=}&
     \sum_{k\in\mathcal{TL}_i(1)\cap\mathcal{TL}_j(1)}
    a_{k,\alpha}(t)a_{k,\beta}(t)\Big(\boldsymbol{x}_i^\top\boldsymbol{x}_j+1\Big)
    \\\overset{\text{Lemma \ref{lemma: correct neurons remain correct 2}}}{=}&
     \sum_{k\in[m]}
    a_{k,\alpha}(0)a_{k,\beta}(0)\Big(\boldsymbol{x}_i^\top\boldsymbol{x}_j+1\Big)
    \\\overset{\text{Lemma \ref{lemma: speed separation 2}}}{\geq}&
    m(\frac{1}{\sqrt{m}})^2=1.
\end{align*}

\end{proof}

\begin{lemma}[Lower bound of stochastic gradient]\label{lemma: gradient lower bound 2}
For any $1\leq t\leq T$ we have
\begin{align*}
\left<\nabla \mathcal{L}(\boldsymbol{\theta}(t)),\frac{1}{B}\sum_{i\in\mathcal{B}_t}\nabla_{\boldsymbol{\theta}}\ell(\boldsymbol{y}_i;\boldsymbol{f}(\boldsymbol{x}_i;\boldsymbol{\theta}))\right>\geq\frac{9801}{10000}.
\end{align*}

\end{lemma}

\begin{proof}[Proof of Lemma \ref{lemma: gradient lower bound 2}]\ \\
From Lemma \ref{lemma: speed separation 2} (P4) and the definition of hitting time (\ref{equ: hitting time def 2}), we know $0\leq\boldsymbol{y}_i^\top\boldsymbol{f}(\boldsymbol{x}_i;\boldsymbol{\theta}(t))\leq1$ for all $i\in[n]$ and $0\leq t\leq T$, which implies 
\[
-\tilde{\ell}'(\boldsymbol{y}_i^\top\boldsymbol{f}(\boldsymbol{x}_i;\boldsymbol{\theta}(t)))\geq g_{\min}=0.99,\ \forall i\in[n].
\]
Combining Lemma \ref{lemma: positive Gram time t 2}, we have the lower bound of stochastic gradient:
\begin{align*}
&\left<\nabla\mathcal{L}(\boldsymbol{\theta}(t)),\frac{1}{B}\sum_{i\in\mathcal{B}_t}\nabla_{\boldsymbol{\theta}}\ell(\boldsymbol{y}_i;\boldsymbol{f}(\boldsymbol{x}_i;\boldsymbol{\theta}))\right>
\\=&\frac{1}{n B}\sum_{i=1}^n\sum_{j\in\mathcal{B}_t}
\left<\nabla_{\boldsymbol{\theta}}\ell(\boldsymbol{y}_i;\boldsymbol{f}(\boldsymbol{x}_i;\boldsymbol{\theta})),\nabla_{\boldsymbol{\theta}}\ell(\boldsymbol{y}_j;\boldsymbol{f}(\boldsymbol{x}_j;\boldsymbol{\theta}))\right>
\\=&\frac{1}{n B}\sum_{i=1}^n\sum_{j\in\mathcal{B}_t}
\tilde{\ell}'(\boldsymbol{y}_i^\top\boldsymbol{f}(\boldsymbol{x}_i;\boldsymbol{\theta}))\tilde{\ell}'(\boldsymbol{y}_j^\top\boldsymbol{f}(\boldsymbol{x}_j;\boldsymbol{\theta}))\boldsymbol{y}_i^\top\nabla \boldsymbol{f}(\boldsymbol{x}_i;\boldsymbol{\theta}(t))^\top\nabla\boldsymbol{f}(\boldsymbol{x}_j;\boldsymbol{\theta}(t))\boldsymbol{y}_j
\\\geq&
(0.99)^2=\frac{9801}{10000}.
\end{align*}

\end{proof}

\subsection{Early stage convergence}

\begin{lemma}[Upper bound of stochastic gradient]\label{lemma: stochastic gradient upper bound 2} $T^*$ is defined in Lemma \ref{lemma: hitting time estimate 2}. Then we have:
\[
\left\|\frac{1}{B}\sum_{i\in\mathcal{B}_t}\nabla_{\boldsymbol{\theta}}\ell(\boldsymbol{y}_i;\boldsymbol{f}(\boldsymbol{x}_i;\boldsymbol{\theta}(t)))\right\|\leq\sqrt{\frac{11}{2}},\ 0\leq t\leq T^*.
\]
\end{lemma}

\begin{proof}[Proof of Lemma \ref{lemma: stochastic gradient upper bound 2}]\ \\
\begin{align*}
    &\left\|\frac{1}{B}\sum_{i\in\mathcal{B}_t}\nabla_{\boldsymbol{\theta}}\ell(\boldsymbol{y}_i;\boldsymbol{f}(\boldsymbol{x}_i;\boldsymbol{\theta}(t)))\right\|
    \\
    =&\left\|\frac{1}{B}\sum_{i\in\mathcal{B}_t}\tilde{\ell}'(\boldsymbol{y}_i^\top
    \boldsymbol{f}(\boldsymbol{x}_i;\boldsymbol{\theta}(t)))\nabla \boldsymbol{f}(\boldsymbol{x}_i;\boldsymbol{\theta}(t))\boldsymbol{y}_i\right\|
    \\\leq&\max_{i\in[n]}\max_{\alpha\in[C]}\Bigg(\sum_{k=1}^m\Big( \left|\sigma\Big(\boldsymbol{b}_k^\top(t)\boldsymbol{x}_i+c_k(t)\Big)\right|^2+
    \left\|a_{k,\alpha}(t)\sigma'\Big(\boldsymbol{b}_k^\top(t)\boldsymbol{x}_i+c_k(t)\Big)\boldsymbol{x}_i\right\|^2+
    \left|a_{k,\alpha}(t)\sigma'\Big(\boldsymbol{b}_k^\top(t)\boldsymbol{x}_i+c_k(t)\Big)\right|^2\Big)
    \Bigg)^{\frac{1}{2}}
    \\\leq&
    \max_{\alpha\in[C]}\Bigg(\sum_{k=1}^m\Big((\left\|\boldsymbol{b}_k(t)\right\|+\left|c_k(t)\right|)^2+|a_{k,\alpha}(t)|^2+|a_{k,\alpha}(t)|^2\Big)\Bigg)^{\frac{1}{2}}.
\end{align*}

For $0\leq t\leq T^*$, from Lemma \ref{lemma: speed separation 2} (P1)(P2), we can estimate
\begin{gather*}
    |a_{k,\alpha}(t)|\leq\frac{501\Big((1+2\eta)^{T^*}+(1-2\eta)^{T^*}\Big)}{1000\sqrt{m}}\leq\frac{501}{1000\sqrt{m}}\Big(\sqrt{4}+\frac{1}{\sqrt{4}}\Big)\leq\frac{3}{2\sqrt{m}},\ \forall\alpha\in[C],
    \\
    \left\|\boldsymbol{b}_k(t)\right\|+|c_k(t)|\leq\frac{501\Big((1+2\eta)^{T^*}-(1-2\eta)^{T^*}\Big)}{1000\sqrt{m}}\leq\frac{1}{\sqrt{m}}.
\end{gather*}
Hence,
\[
    \left\|\frac{1}{B}\sum_{i\in\mathcal{B}_t}\nabla_{\boldsymbol{\theta}}\ell(\boldsymbol{y}_i;\boldsymbol{f}(\boldsymbol{x}_i;\boldsymbol{\theta}(t)))\right\|\leq
    \sqrt{m\Big(\frac{1}{m}+2\frac{9}{4m}\Big)}=\sqrt{\frac{11}{2}}.
\]

\end{proof}

\begin{lemma}\label{lemma: loss increase upper bound 2} Under the same conditions of Theorem \ref{thm: binary global GD linear}, we have a rough estimate of loss at $t=0$:
\begin{equation*}
    \mathcal{L}(\boldsymbol{\theta}(1))\leq\mathcal{L}(\boldsymbol{\theta}(0))+2.46\eta.
\end{equation*}
\end{lemma}

\begin{proof}[Proof of Lemma \ref{lemma: loss increase upper bound 2}]\ \\
First, for any $\boldsymbol{\theta}\in\overline{\boldsymbol{\theta}(0)\boldsymbol{\theta}(1)}$, we have the estimate of gradient:
\begin{align*}
    \left\|\nabla\mathcal{L}(\boldsymbol{\theta})\right\|
    \leq&
    \sup\limits_{\boldsymbol{\theta}\in\overline{\boldsymbol{\theta}(0)\boldsymbol{\theta}(1)}}\max_{\alpha\in[C]}\Bigg(\sum_{k=1}^m\Big((\left\|\boldsymbol{b}_k(0)\right\|+\left|c_k(0)\right|)^2+|a_{k,\alpha}(0)|^2+|a_{k,\alpha}(0)|^2\Big)\Bigg)^{\frac{1}{2}}
    \\\overset{\text{Proof of Lemma \ref{lemma: speed separation 2}}}{\leq}&
    \sqrt{m(\frac{2004\eta}{1000\sqrt{m}})^2+2m(\frac{1002}{\sqrt{1000m}})^2}\overset{\eqref{equ: parameter selection 2}}{\leq}2.46.
\end{align*}
Then we have:
\begin{equation*}
    \begin{aligned}
    \mathcal{L}(\boldsymbol{\theta}(1))
    \leq&\mathcal{L}(\boldsymbol{\theta}(0))+\Big(\sup_{\boldsymbol{\theta}\in\overline{\boldsymbol{\theta}(0)\boldsymbol{\theta}(1)}}\left\|\nabla\mathcal{L}(\boldsymbol{\theta})\right\|\Big)\left\|\boldsymbol{\theta}(1)-\boldsymbol{\theta}(0)\right\|
    \\\leq&
    \mathcal{L}(\boldsymbol{\theta}(0))+\eta\Big(\sup_{\boldsymbol{\theta}\in\overline{\boldsymbol{\theta}(0)\boldsymbol{\theta}(1)}}\left\|\nabla\mathcal{L}(\boldsymbol{\theta})\right\|\Big)^2
    \\\leq&
    \mathcal{L}(\boldsymbol{\theta}(0))+2.46\eta.
    \end{aligned}
\end{equation*}

\end{proof}

\begin{lemma}[Upper Bound of Hessian]\label{lemma: Hessian upper bound ERM 2}  
$T^*$ is defined in Lemma \ref{lemma: hitting time estimate 2}.
Consider all parameters along the trajectory of SGD (\ref{equ: alg SGD}) when $1\leq t\leq T^*-1$:
\[
\mathcal{S}(T^*):=\bigcup\limits_{1\leq t\leq T^*-1}\Big\{\boldsymbol{\theta}:\boldsymbol{\theta}\in\overline{\boldsymbol{\theta}(t)\boldsymbol{\theta}({t+1})}\Big\}
\]
then for any $\boldsymbol{\theta}\in\mathcal{S}(T^*)$, we have:
\[
\left\|\nabla^2 \mathcal{L}(\boldsymbol{\theta})\right\|
\leq \frac{25}{4m}+2\sqrt{2}\leq4.
\]
\end{lemma}
\begin{proof}[Proof of Lemma \ref{lemma: Hessian upper bound ERM 2}]\ \\
For any $\boldsymbol{\theta}\in\mathcal{S}(T^*)$, we can define:
\begin{gather*}
    \mathbf{H}_A(\boldsymbol{\theta}):=\frac{1}{n}
    \sum_{i=1}^n \tilde{\ell}''(\boldsymbol{y}_i^\top
    \boldsymbol{f}(\boldsymbol{x}_i;\boldsymbol{\theta}))\nabla \boldsymbol{f}(\boldsymbol{x}_i;\boldsymbol{\theta})\boldsymbol{y}_i\boldsymbol{y}_i^\top
    \nabla \boldsymbol{f}(\boldsymbol{x}_i;\boldsymbol{\theta})^\top,
    \\
    \mathbf{H}_B(\boldsymbol{\theta}):=\frac{1}{n}\sum_{i=1}^n\tilde{\ell}'(\boldsymbol{y}_i^\top
    \boldsymbol{f}(\boldsymbol{x}_i;\boldsymbol{\theta}))\Big(y_{i,1}\nabla^2 f_1(\boldsymbol{x}_i;\boldsymbol{\theta})+\cdots+y_{i,C}\nabla^2 f_C(\boldsymbol{x}_i;\boldsymbol{\theta})\Big).
\end{gather*}
then we have the relationship:
\begin{align*}
    \nabla^2 \mathcal{L}(\boldsymbol{\theta})
    = \mathbf{H}_A(\boldsymbol{\theta}) + \mathbf{H}_B(\boldsymbol{\theta})
\end{align*}

Now we only need to estimate $\left\|\mathbf{H}_A(\boldsymbol{\theta})\right\|$ and $\left\|\mathbf{H}_B(\boldsymbol{\theta})\right\|$ respectively.

From the definition of $\boldsymbol{\theta}\in\mathcal{S}(T^*)$ and Lemma \ref{lemma: speed separation 2} (P1)(P2), we obtain the consistent estimate for any $\boldsymbol{\theta}\in \mathcal{S}(T^*)$:
\begin{gather*}
    |a_{k,\alpha}|\leq\frac{501\Big((1+2\eta)^{T^*}+(1-2\eta)^{T^*}\Big)}{1000\sqrt{m}}\leq\frac{501}{1000\sqrt{m}}\Big(\sqrt{4}+\frac{1}{\sqrt{4}}\Big)\leq\frac{3}{2\sqrt{m}},\ \forall\alpha\in[C],
    \\
    \left\|\boldsymbol{b}_k\right\|+|c_k|\leq\frac{501\Big((1+2\eta)^{T^*}-(1-2\eta)^{T^*}\Big)}{1000\sqrt{m}}+\frac{501\Big((1+2\eta)^{T^*}-(1-2\eta)^{T^*}\Big)}{1000\sqrt{m}}\leq\frac{2}{\sqrt{m}}.
\end{gather*}

For the first part $\mathbf{H}_A(\boldsymbol{\theta})$, we have:
\begin{align*}
    \left\|\mathbf{H}_A(\boldsymbol{\theta})\right\|
    =&
    \sup\limits_{\boldsymbol{w}\ne 0}\frac{\frac{1}{n}\sum\limits_{i=1}^n\tilde{\ell}''(\boldsymbol{y}_i^\top
    \boldsymbol{f}(\boldsymbol{x}_i;\boldsymbol{\theta}))\boldsymbol{w}^\top\nabla \boldsymbol{f}(\boldsymbol{x}_i;\boldsymbol{\theta})\boldsymbol{e}_i\boldsymbol{e}_i^\top\nabla \boldsymbol{f}(\boldsymbol{x}_i;\boldsymbol{\theta})^\top\boldsymbol{w}}{\left\|\boldsymbol{w}\right\|^2}
    \\\leq&
        \sup_{u_1,\boldsymbol{v}_1,\cdots,u_m,\boldsymbol{v}_m {\rm\ not\  all\ }0}\frac{\sum\limits_{k=1}^m\Big(u_k(\left\|\boldsymbol{b}_k\right\|+c_k)+\left\|\boldsymbol{v}_k\right\|\max\limits_{\alpha\in[C]}|a_{k,\alpha}|\Big)^2}{\sum\limits_{k=1}^m\Big(u_k^2+\left\|\boldsymbol{v}_k\right\|^2\Big)}
        \\\leq&
 \sup_{u_1,\boldsymbol{v}_1,\cdots,u_m,\boldsymbol{v}_m {\rm\ not\  all\ }0}\frac{\sum\limits_{k=1}^m\Big(\frac{4}{m}u_k^2+\frac{6}{m}|u_k|\left\|\boldsymbol{v}_k\right\|+\frac{9}{4m}\left\|\boldsymbol{v}_k\right\|^2\Big)}{\sum\limits_{k=1}^m\Big(u_k^2+\left\|\boldsymbol{v}_k\right\|^2\Big)}
 \\\leq&
  \sup_{u_1,\boldsymbol{v}_1,\cdots,u_m,\boldsymbol{v}_m {\rm\ not\  all\ }0}\frac{\sum\limits_{k=1}^m\Big(\frac{25}{4m}u_k^2+\frac{25}{4m}\left\|\boldsymbol{v}_k\right\|^2\Big)}{\sum\limits_{k=1}^m\Big(u_k^2+\left\|\boldsymbol{v}_k\right\|^2\Big)}=\frac{25}{4m}.
\end{align*}

For the second part $\mathbf{H}_B(\boldsymbol{\theta})$, we have:
\begin{align*}
\left\|\mathbf{H}_B(\boldsymbol{\theta})\right\|\leq&\sup_{(\boldsymbol{x},\boldsymbol{y})\sim\mathcal{\mu}}\left\|y_{i,1}\nabla^2 f_1(\boldsymbol{x};\boldsymbol{\theta})+\cdots+y_{i,C}\nabla^2 f_C(\boldsymbol{x};\boldsymbol{\theta})\right\|.
\end{align*}
We only need to estimate $\left\|\nabla^2 f_1(\boldsymbol{x};\boldsymbol{\theta})\right\|$.
$\nabla^2 f_1(\boldsymbol{x};\boldsymbol{\theta})$ is combined by $m$ diagonal block:
\[
\nabla^2 f_1(\boldsymbol{x};\boldsymbol{\theta}):=\left(\begin{array}{ccc}  
    \mathbf{H}_{f_1}^{(1)} &  & \\
     & \ddots & \\ 
      &  &\mathbf{H}_{f_1}^{(m)}\\
  \end{array}\right),
\]

where the $k-$th block $\mathbf{H}_{f_1}^{(k)}$ is:
\begin{align*}
\mathbf{H}_{f_1}^{(k)}=&\begin{pmatrix} 
    \frac{\partial^2 f_1(\boldsymbol{x};\boldsymbol{\theta})}{\partial^2 a_{k,1}} & \frac{\partial^2 f_1(\boldsymbol{x};\boldsymbol{\theta})}{\partial a_{k,1}{\partial \boldsymbol{b}_k}^\top} &
    \frac{\partial^2 f_1(\boldsymbol{x};\boldsymbol{\theta})}{\partial a_{k,1}
    \partial c_k}\\
   \frac{\partial^2 f_1(\boldsymbol{x};\boldsymbol{\theta})}{\partial \mathbf{b}_k\partial a_{1,k}}  & \frac{\partial^2 f_1(\boldsymbol{x};\boldsymbol{\theta})}{\partial \boldsymbol{b}_k{\partial \boldsymbol{b}_k}^\top} &
   \frac{\partial^2 f_1(\boldsymbol{x};\boldsymbol{\theta})}{\partial \boldsymbol{b}_k\partial c_k} \\ 
   \frac{\partial^2 f_1(\boldsymbol{x};\boldsymbol{\theta})}{\partial a_{k,1}\partial c_k}
      & 
      \frac{\partial^2 f_1
      (\boldsymbol{x};\boldsymbol{\theta})}{\partial c_k{\partial \boldsymbol
      {b}_k}^\top} &
      \frac{\partial^2 f_1(\boldsymbol{x};\boldsymbol{\theta})}{\partial^2 c_k}
  \end{pmatrix}
  \\=&
  \begin{pmatrix} 
    0 & \sigma'\Big(\boldsymbol{b}_k^\top\boldsymbol{x}+c_k\Big)\boldsymbol{x}^\top &
    \sigma'\Big(\boldsymbol{b}_k^\top\boldsymbol{x}+c_k\Big)\\
   \sigma'\Big(\boldsymbol{b}_k^\top\boldsymbol{x}+c_k\Big)\boldsymbol{x}  & a_{k,1}\sigma''\Big(\boldsymbol{b}_k^\top\boldsymbol{x}+c_k\Big)\boldsymbol{x}\boldsymbol{x}^\top &
   a_{k,1}\sigma''\Big(\boldsymbol{b}_k^\top\boldsymbol{x}+c_k\Big)\boldsymbol{x} \\ 
 \sigma'\Big(\boldsymbol{b}_k^\top\boldsymbol{x}+c_k\Big)
      & 
     a_{k,1}\sigma''\Big(\boldsymbol{b}_k^\top\boldsymbol{x}+c_k\Big)\boldsymbol{x}^\top &
      a_{k,1}\sigma''\Big(\boldsymbol{b}_k^\top\boldsymbol{x}+c_k\Big)
  \end{pmatrix}.
\end{align*}

Recalling Lemma \ref{lemma: correct neurons remain correct 2} (S3), we know that: for any
$\boldsymbol{\theta}=(\boldsymbol{a}_1^\top,\cdots,\boldsymbol{a}_m^\top,\boldsymbol{b}_1^\top,\cdots,\boldsymbol{b}_m^\top,c_1,\cdots,c_m)^\top\in\mathcal{S}(T^*)$, we have $\boldsymbol{b}_k^\top\boldsymbol{x}_i+c_k\leq0$ for any $k\in[m]$ and $i\in[n]$. Hence, $\sigma''(\boldsymbol{b}_k^\top\boldsymbol{x}_i+c_k)=0$ holds strictly. So we have the form
\[
\mathbf{H}_{f_1}^{(k)}=
  \begin{pmatrix} 
    0 & \sigma'\Big(\boldsymbol{b}_k^\top\boldsymbol{x}+c_k\Big)\boldsymbol{x}^\top &
    \sigma'\Big(\boldsymbol{b}_k^\top\boldsymbol{x}+c_k\Big)\\
   \sigma'\Big(\boldsymbol{b}_k^\top\boldsymbol{x}+c_k\Big)\boldsymbol{x}  & 0 & 0 \\ 
   \sigma'\Big(\boldsymbol{b}_k^\top\boldsymbol{x}+c_k\Big) & 0 & 0
  \end{pmatrix}.
\]
So We can estimate:
\begin{align*}
    \left\|\nabla^2 f_1(\boldsymbol{x};\boldsymbol{\theta})\right\|
    =&\sup\limits_{\boldsymbol{v}_1,\cdots,\boldsymbol{v}_m {\rm are\ not\ all\ 0}}\frac{\sum\limits_{k=1}^m\boldsymbol{v}_k^\top\mathbf{H}_{f_1}^{(k)}\boldsymbol{v}_k }{\sum\limits_{k=1}^m\left\|\boldsymbol{v}_k\right\|^2}
    \\\leq&
    \max_{k\in[m]}\left\|\mathbf{H}_{f_1}^{(k)}\right\|\leq\max_{k\in[m]}\left\|\mathbf{H}_{f,1}^{(k)}\right\|_{\text{ F}}
    \\=&
    \Bigg(
    2\left\|\sigma'\Big(\boldsymbol{b}_k^\top\boldsymbol{x}+c_k\Big)\boldsymbol{x}\right\|^2+2\left|\sigma'\Big(\boldsymbol{b}_k^\top\boldsymbol{x}+c_k\Big)\right|^2\Bigg)^{\frac{1}{2}}
    \\\leq&
    \Big(2+2\Big)^{\frac{1}{2}}=2\sqrt{2}.
\end{align*}
So we have the estimate of the second part:
\begin{align*}
\left\|\mathbf{H}_B(\boldsymbol{\theta})\right\|\leq
2\sqrt{2}.
\end{align*}

Combining the two estimates, we obtain this lemma:
\[
\left\|\nabla^2 \mathcal{L}(\boldsymbol{\theta})\right\|\leq
\left\|\mathbf{H}_A(\boldsymbol{\theta})\right\|+
\left\|\mathbf{H}_B(\boldsymbol{\theta})\right\|\leq\frac{25}{4m}+2\sqrt{2}
\overset{\eqref{equ: parameter selection 2}}{\leq}4.
\]

\end{proof}

\begin{proof}[\bfseries\color{blue} Proof of Theorem \ref{thm: one-hot}]\ \\
First, we want to clarify that all the notations that appear here are used in the proofs of Lemmas \ref{lemma: speed separation 2}, \ref{lemma: hitting time estimate 2}, \ref{lemma: positive Gram time t 2}, \ref{lemma: gradient lower bound 2}, \ref{lemma: stochastic gradient upper bound 2}, \ref{lemma: loss increase upper bound 2}, \ref{lemma: Hessian upper bound ERM 2}, and \ref{lemma: loss quadratic upper bound 1}.

We have the estimate:
\begin{align*}
    &\mathcal{L}(\boldsymbol{\theta}(T^*))
    \\
    \overset{\text{Proof of Lemma \ref{lemma: loss quadratic upper bound 1}}}{\leq}&
    \mathcal{L}(\boldsymbol{\theta}(T^*-1))+\left<\nabla \mathcal{L}(\boldsymbol{\theta}(T^*-1)),\boldsymbol{\theta}(T^*)-\boldsymbol{\theta}(T^*-1)\right>
    \\&\quad\quad\quad\quad\quad+\frac{1}{2}\Big(\sup\limits_{\boldsymbol{\theta}\in\mathcal{S}(T^*)}\left\|\nabla^2 \mathcal{L}(\boldsymbol{\theta})\right
\|\Big)\left\|\boldsymbol{\theta}(T^*)-\boldsymbol{\theta}(T^*-1)\right\|^2
\\=&
\mathcal{L}(\boldsymbol{\theta}(T^*-1))-\eta\left<\nabla \mathcal{L}(\boldsymbol{\theta}({T^*-1})),\frac{1}{B}\sum_{i\in\mathcal{B}_{T^*-1}}\nabla_{\boldsymbol{\theta}}\ell(\boldsymbol{y}_i;\boldsymbol{f}(\boldsymbol{x}_i;\boldsymbol{\theta}({T^*-1})))\right>
\\&\quad\quad\quad\quad\quad+\frac{\eta^2}{2}\Big(\sup\limits_{\boldsymbol{\theta}\in\mathcal{S}(T^*)}\left\|\nabla^2 \mathcal{L}(\boldsymbol{\theta})\right
\|\Big)\left\|\frac{1}{B}\sum_{i\in\mathcal{B}_{T^*-1}}\nabla_{\boldsymbol{\theta}}\ell(\boldsymbol{y}_i;\boldsymbol{f}(\boldsymbol{x}_i;\boldsymbol{\theta}({T^*-1})))\right\|^2
\\\overset{\text{Lemma \ref{lemma: Hessian upper bound ERM 2}, \ref{lemma: stochastic gradient upper bound 2}}}{\leq}&
\mathcal{L}(\boldsymbol{\theta}(T^*-1))-\eta\left<\nabla \mathcal{L}(\boldsymbol{\theta}({T^*-1})),\frac{1}{B}\sum_{i\in\mathcal{B}_{T^*-1}}\nabla_{\boldsymbol{\theta}}\ell(\boldsymbol{y}_i;\boldsymbol{f}(\boldsymbol{x}_i;\boldsymbol{\theta}({T^*-1})))\right>
\\&\quad\quad\quad\quad\quad+\frac{\eta^2}{2}\cdot 4\cdot\frac{11}{2}
\\\overset{(\ref{equ: parameter selection 2})}{\leq}&
\mathcal{L}(\boldsymbol{\theta}(T^*-1))-\eta\left<\nabla \mathcal{L}(\boldsymbol{\theta}({T^*-1})),\frac{1}{B}\sum_{i\in\mathcal{B}_{T^*-1}}\nabla_{\boldsymbol{\theta}}\ell(\boldsymbol{y}_i;\boldsymbol{f}(\boldsymbol{x}_i;\boldsymbol{\theta}({T^*-1})))\right>+\frac{11}{100}\eta
\\\overset{\text{Lemma \ref{lemma: gradient lower bound 2}}}{\leq}&
\mathcal{L}(\boldsymbol{\theta}(T^*-1))-
\Big(\frac{9801}{10000}-\frac{11}{100}\Big)\eta
\\\leq&\cdots
\\\leq&\mathcal{L}(\boldsymbol{\theta}(1))-\frac{8701}{10000}\eta (T^*-1)
\\\overset{\text{Lemma \ref{lemma: loss increase upper bound 2}}}{\leq}&\mathcal{L}(\boldsymbol{\theta}(0))+2.46\eta-\frac{8701}{10000}\eta (T^*-1)
\\\leq&\mathcal{L}(\boldsymbol{\theta}(0))-0.262533
.
\end{align*}

So we have proved that in $34$ iterations, loss will descend
\[
\mathcal{L}\big(\boldsymbol{\theta}(0)\big)-\mathcal{L}\big(\boldsymbol{\theta}(T)\big)\geq0.262533=\Omega(1).
\]

Now we have proved the theorem for the specific  number $\eta=\frac{1}{100}$, $\kappa\leq\min\Big\{\frac{\eta}{10},\frac{\eta}{3B}\Big\}$,
$m\geq 6$, $g_{\min}=0.99$, $g_{\max}=1$, $h_{\max}=1$ (in Assumption \ref{def: general loss}) and $s=0$ (in Assumption \ref{ass: data concentration}).
In the same way, our result also holds for any other $\eta, \kappa,m$ s.t. $\eta\leq\frac{1}{100}$, $\kappa\leq\min\Big\{\frac{\eta}{10},\frac{\eta}{3B}\Big\}$,
$m\geq 6$, $g_{\min}=0.99$, $g_{\max}=1$, $z_0$, $g_{\min}$, $g_{\max}$, $h_{\max}$ and $s$.
\end{proof}

\newpage

\section{Proof details of Theorem \ref{thm: binary global GD}}\label{appendix: proof global convergence}

\subsection{Preparation}

\noindent{\bfseries One-step updating.}
\begin{equation}\label{equ: one step update GD binary}
\begin{gathered}
     a_k(t+1)
    =
    a_k(t)-\eta\frac{1}{n}\sum_{i=1}^n\tilde{\ell}'\Big({y}_i{f}(\boldsymbol{x}_i;\boldsymbol{\theta}(t))
    \Big)\sigma\Big(\boldsymbol{b}_k^\top(t)\boldsymbol{x}_i\Big);
    \\
    \boldsymbol{b}_k(t+1)=
    \boldsymbol{b}_k(t)-\eta\frac{1}{n}\sum_{i=1}^n\tilde{\ell}'\Big({y}_i{f}(\boldsymbol{x}_i;\boldsymbol{\theta}(t))
    \Big)\mathbb{I}\Big\{\boldsymbol{b}_k^\top(t)\boldsymbol{x}_i> 0\Big\}a_k(t)\boldsymbol{x}_i.
    \end{gathered}
\end{equation}


To make our proof more clear, without loss of generality, we only need to prove the theorem with specific numbers $g_a=\frac{1}{2}$, $g_b=1$ and $h=1$ in Assumption \ref{def: exponential type loss}, such as the \textbf{logistic loss} $\ell(\boldsymbol{y}_1,\boldsymbol{y}_2)=\log(1+e^{-\boldsymbol{y}_1^\top\boldsymbol{y}_2})$. For other loss functions with different $g_a$, $g_b$ and $h$, the proof is similar.

\subsection{Fine-grained dynamical analysis of neuron partition}\label{subsection: neuron partition analysis}
An important technique in our proof is the fine-grained analysis of each neuron during training. In \textbf{Stage I}, neurons adjust their directions rapidly. Then in \textbf{Stage II}, neurons keep good directions and loss descend fast.

For the $i$-th data, we divide all neurons into four categories and study them separately.
We denote the true-living neurons $\mathcal{TL}_i(t)$, the true-dead neurons $\mathcal{TD}_i(t)$, the false-living neurons $\mathcal{FL}_i(t)$ and the false-dead neurons $\mathcal{FD}_i(t)$ at time $t$ as:
\begin{gather*}
    \mathcal{TL}_{i}(t):=\Big\{k\in[m]:y_ia_k(t)>0, \boldsymbol{b}_k(t)^\top\boldsymbol{x}_i> 0\Big\},
    \\
    \mathcal{TD}_{i}(t):=\Big\{k\in[m]:y_ia_k(t)>0,\boldsymbol{b}_k(t)^\top\boldsymbol{x}_i\leq 0\Big\},
    \\
    \mathcal{FL}_{i}(t):=\Big\{k\in[m]:y_ia_k(t)<0, \boldsymbol{b}_k(t)^\top\boldsymbol{x}_i> 0\Big\},
    \\
    \mathcal{FD}_{i}(t):=\Big\{k\in[m]:y_ia_k(t)<0,\boldsymbol{b}_k(t)^\top\boldsymbol{x}_i\leq 0\Big\}.
\end{gather*}

\subsubsection{Stage I: neurons adjust their directions rapidly}

\begin{lemma}[Initial neuron partition]\label{lemma: neuron partition stage 1}
For any $i\in[n]$, we have:
\begin{gather*}
    [m]=\mathcal{TL}_{i}(0)\bigcup\mathcal{TD}_{i}(0)\bigcup\mathcal{FL}_{i}(0)\bigcup\mathcal{FD}_{i}(0).
\end{gather*}
\end{lemma}
\begin{proof}[Proof of Lemma \ref{lemma: neuron partition stage 1}]\ \\
Notice that $a_k(0)\ne 0$.
\end{proof}

\begin{lemma}[Estimate of the initial neural partition]\label{lemma: estimate of the number of initial neural partition global} With probability at least $1-\delta$, we have the following estimates
\begin{gather*}
\Bigg|\frac{1}{m}\text{card}\Big(\mathcal{TL}_i(0)\cap\mathcal{TL}_j(0)\Big)-\frac{\pi-\arccos(\boldsymbol{x}_i^\top\boldsymbol{x}_j)}{4\pi}\Bigg|\leq
\sqrt{\frac{\log(n^2/\delta)}{2m}},\text{ for any $i,j$ in the same class},
\end{gather*}
and the same inequalities also hold for $\mathcal{TL}_i(0)\cap\mathcal{TD}_j(0)$, $\mathcal{TD}_i(0)\cap\mathcal{TL}_j(0)$, and $\mathcal{TD}_i(0)\cap\mathcal{TD}_j(0)$.
\end{lemma}

\begin{proof}[Proof of Lemma \ref{lemma: estimate of the number of initial neural partition global}]\ \\
We define a matrix $\mathbf{P}(0)=({\rm P}_{i,j}(0))\in\mathbb{R}^{n\times n}$ as:
\[
{\rm P}_{i,j}(0)=\frac{1}{m}\sum_{k=1}^m\mathbb{I}\big\{a_k(0)y_i>0\big\}\mathbb{I}\big\{a_k(0)y_j>0\big\}\mathbb{I}\big\{\boldsymbol{b}_k(0)^\top\boldsymbol{x}_i>0\big\}\mathbb{I}\big\{\boldsymbol{b}_k(0)^\top\boldsymbol{x}_j>0\big\}.
\]
It is easy to verify that:
\[
\text{card}\Big(\mathcal{TN}_i(0)\cap\mathcal{TN}_j(0)\Big)=m{\rm P}_{i,j}(0)
\]

\textbf{(I).}
For $i,j$ in the same class, we can let $i,j\in[n/2]$.
Now we consider the expectation:
\begin{align*}
{\rm P}_{i,j}^{\infty}:=&\mathbb{E}_{a\sim\mathbb{U}\{\pm1\},\boldsymbol{b}\sim\mathcal{N}(\mathbf{0},\frac{\kappa^2}{md}\mathbf{I}_d)}\Bigg[\mathbb{I}\big\{a>0\big\}\mathbb{I}\big\{\boldsymbol{b}^\top\boldsymbol{x}_i>0\big\}\mathbb{I}\big\{\boldsymbol{b}^\top\boldsymbol{x}_j>0\big\}\Bigg]
\\=&
\mathbb{E}_{a\sim{\rm Unif}\{\pm1\}}\Bigg[
\mathbb{E}_{\boldsymbol{b}\sim\mathcal{N}(\mathbf{0},\frac{\kappa^2}{md}\mathbf{I}_d)}\Big[\mathbb{I}\big\{a>0\big\}\mathbb{I}\big\{\boldsymbol{b}^\top\boldsymbol{x}_i>0\big\}\mathbb{I}\big\{\boldsymbol{b}^\top\boldsymbol{x}_j>0\big\}\Big|a\Big]\Bigg]
\\=&
\frac{1}{2}
\mathbb{E}_{\boldsymbol{b}\sim\mathcal{N}(\mathbf{0},\frac{\kappa^2}{md}\mathbf{I}_d)}\Big[\mathbb{I}\big\{\boldsymbol{b}^\top\boldsymbol{x}_i>0\big\}\mathbb{I}\big\{\boldsymbol{b}^\top\boldsymbol{x}_j>0\big\}\Big]
\\=&\frac{1}{2}
\mathbb{E}_{\boldsymbol{b}\sim\mathcal{N}(\mathbf{0},\mathbf{I}_d)}\Big[\mathbb{I}\big\{\boldsymbol{b}^\top\boldsymbol{x}_i>0\big\}\mathbb{I}\big\{\boldsymbol{b}^\top\boldsymbol{x}_j>0\big\}\Big]
\\=&
\frac{\pi-\arccos\Big(\boldsymbol{x}_i^\top\boldsymbol{x}_j\Big)}{4\pi}.
\end{align*}
By Hoeffding's Inequality (Lemma \ref{lemma: hoeffding}), we have:
\begin{align*}
    \mathbb{P}\Bigg(\left|{\rm P}_{i,j}(0)-{\rm P}_{i,j}^\infty\right|\geq\sqrt{\frac{\log(2/\delta)}{2m}}\Bigg)\leq\delta.
\end{align*}
Applying a union bound over all $i,j$ in the same class, we know that with probability at least $1-\delta$,
\[
\left|{\rm P}_{i,j}(0)-{\rm P}_{i,j}^\infty\right|\leq\sqrt{\frac{\log(n^2/\delta)}{2m}},\ \forall i,j\in[n],
\]
which means
\begin{align*}
\Bigg|\frac{1}{m}\text{card}\Big(\mathcal{TL}_i(0)\cap\mathcal{TL}_j(0)\Big)-\frac{\pi-\arccos(\boldsymbol{x}_i^\top\boldsymbol{x}_j)}{4\pi}\Bigg|\leq
\sqrt{\frac{\log(n^2/\delta)}{2m}},\text{ for any $i,j$ in the same class}.
\end{align*}

And the same inequalities also hold for $\mathcal{TL}_i(0)\cap\mathcal{TD}_j(0)$, $\mathcal{TD}_i(0)\cap\mathcal{TL}_j(0)$, and $\mathcal{TD}_i(0)\cap\mathcal{TD}_j(0)$.

\end{proof}

\begin{lemma}[Initial norm and prediction]\label{lemma: initial norm and prediction global} With probability at least $1-2me^{-2d}$ we have:
\begin{gather*}
\left\|\boldsymbol{b}_k(0)\right\|\leq\frac{2\kappa}{\sqrt{m}},\ \forall k\in[m];\\
    |f(\boldsymbol{x}_i;\boldsymbol{\theta}(0))|\leq2\kappa,\ \forall i\in[n].
\end{gather*}

\end{lemma}
\begin{proof}[Proof of Lemma \ref{lemma: initial norm and prediction global}]\ \\
From the fact ${\boldsymbol{b}}_k(0)\sim\mathcal{N}(\mathbf{0},\frac{\kappa^2}{md}\mathbf{I}_d)$, we have the probability inequality:
\[
\mathbb{P}\Big(\left\|\boldsymbol{b}_k(0)\right\|\leq\frac{2\kappa}{\sqrt{m}}\Big)\geq 1 - 2\exp(-2d)
\]

Combining the uniform bound about $k\in[m]$, with probability at least $1-\frac{2m}{\exp(2d)}$ we have
\[
\left\|\boldsymbol{b}_k(0)\right\|\leq\frac{2\kappa}{\sqrt{m}},\ \forall k\in[m].
\]

From the path norm estimate of $f(\cdot;\cdot)$, we have:
\[
|f(\boldsymbol{x}_i;\boldsymbol{\theta}(0))|\leq\sum_{k=1}^m|a_k(0)||\boldsymbol{b}_k(0)^\top\boldsymbol{x}_i|\leq m\frac{1}{\sqrt{m}}\frac{2\kappa}{\sqrt{m}}=2\kappa.
\]

\end{proof}

\begin{lemma}[Parameter estimate at the first step]\label{lemma: parameter estimate first step}
When the events in Lemma \ref{lemma: initial norm and prediction global} happen, if $\kappa\lesssim\eta_0\mu_0/n$ and $\eta_0\lesssim 1$, we have:
\[
a_k(1)\text{sgn}(a_k(0))\geq\frac{1-2\eta_0^2/n}{\sqrt{m}}>0.
\]

\end{lemma}
\begin{proof}[Proof of Lemma \ref{lemma: parameter estimate first step}]\ \\
We only need to consider $a_k(0)=1/\sqrt{m}$.
From Lemma \ref{lemma: initial norm and prediction global}, we have
\begin{align*}
    a_k(1)=&a_k(0)-\eta_0\frac{1}{n}\sum_{j=1}^n\tilde{\ell}'\big(y_j f(\boldsymbol{x}_j;\boldsymbol{\theta}(0))\big) \sigma\big(\boldsymbol{b}_k(0)^\top\boldsymbol{x}_j\big) y_j
    \\\geq&a_k(0)+\eta_0\frac{1}{n}\sum_{j=1}^n\tilde{\ell}'\big(y_j f(\boldsymbol{x}_j;\boldsymbol{\theta}(0))\big)\left\|\boldsymbol{b}_k(0)\right\|
    \\=&\frac{1}{\sqrt{m}}+\eta_0\frac{1}{n}\sum_{j=1}^n\Big(\tilde{\ell}'\big(y_j f(\boldsymbol{x}_j;\boldsymbol{\theta}(0))\big)-\tilde{\ell}'(0)\Big)\left\|\boldsymbol{b}_k(0)\right\|+\eta_0\tilde{\ell}'(0)\left\|\boldsymbol{b}_k(0)\right\|
    \\\geq&\frac{1}{\sqrt{m}}+\eta_0\frac{2\kappa}{\sqrt{m}}\Bigg(-2\kappa\Big(\max\limits_{z\in[-2\kappa,2\kappa]}\tilde{\ell}''(z)\Big)+\tilde{\ell}'(0)\Bigg)
    \\\geq&\frac{1}{\sqrt{m}}-\eta_0\frac{2\kappa}{\sqrt{m}}\Big(2\kappa\max\limits_{z\in[-2\kappa,2\kappa]}\tilde{\ell}''(z)+\tilde{\ell}(0)\Big).
\end{align*}
From $\tilde{\ell}(\cdot)$ is twice continuously differential in $\mathbb{R}$, there exist $\delta_0>0$, s.t. $\max\limits_{z\in[-2\delta_0,2\delta_0]}\tilde{\ell}''(z)\leq 2\tilde{\ell}''(0)$. Let $\kappa\leq\min\Big\{\delta_0,\frac{1}{4},\frac{\eta_0}{2n\tilde{\ell}(0)}\Big\}$, we have
\begin{align*}
    a_k(1)\geq&\frac{1}{\sqrt{m}}-\eta_0\frac{2\kappa}{\sqrt{m}}\Big(4\kappa\tilde{\ell}''(0)+\tilde{\ell}(0)\Big)\geq\frac{1}{\sqrt{m}}-\eta_0\frac{2\kappa(1+4\kappa)}{\sqrt{m}}\tilde{\ell}(0)
    \\\geq&\frac{1}{\sqrt{m}}\Big(1-4\eta_0\kappa\tilde{\ell}(0)\Big)
\end{align*}
\end{proof}

\begin{lemma}[Dynamics of neuron partition at Stage I]\label{lemma: correct neurons remain correct stage I}
When the events in Lemma \ref{lemma: initial norm and prediction global} happen, if $\kappa\lesssim \mu_0\eta_0/n$ and $\eta_0\lesssim 1$, we have the following results for any $i\in[n]$:\\
(S1) True-living neurons remain true-living:  $\mathcal{TL}_i(0)\subset\mathcal{TL}_i(1)$.\\
(S2) False-dead neurons remain false-dead:  $\mathcal{FD}_i(0)\subset\mathcal{FD}_i(1)$.\\
(S3) True-dead neurons turn true-living:  $\mathcal{TD}_i(0)\subset\mathcal{TL}_i(1)$.\\
(S4) False-living neurons turn false-dead: $\mathcal{FL}_i(0)\subset\mathcal{FD}_i(1)$.\\
(S5) The neuron partition holds: $[m]=\mathcal{TL}_i(1)\cup\mathcal{FD}_i(1)$.\\
(S6) The connectivity holds: $\mathcal{FL}_i(0)\cup\mathcal{FD}_i(0)=\mathcal{FD}_i(1)$ and $\mathcal{TL}_i(0)\cup\mathcal{TD}_i(0)=\mathcal{TL}_i(1)$.
\\
(S7) $\text{sgn}(\boldsymbol{b}_k^\top(1)\boldsymbol{x}_i)\neq0$ for any $k\in[m]$.

\end{lemma}

\begin{proof}[Proof of Lemma \ref{lemma: correct neurons remain correct stage I}]\ \\
We only need to consider $i\in[n/2]$ ($y_i=1$).

Proof of (S1).

Let $k\in\mathcal{TL}_i(0)$, we have $a_k(0)>0$ and $\boldsymbol{b}_k(0)^\top\boldsymbol{x}_i>0$. Then we have
\begin{align*}
    &\boldsymbol{b}_k(1)^\top\boldsymbol{x}_i
    \\=&\boldsymbol{b}_k(0)^\top\boldsymbol{x}_i-\eta_0\frac{1}{n}\sum_{j\in[n]}\tilde{\ell}'\big({y}_j{f}(\boldsymbol{x}_j;\boldsymbol{\theta}(0))\mathbb{I}\big\{\boldsymbol{b}_k(0)^\top\boldsymbol{x}_j>0\big\}y_j\boldsymbol{x}_j^\top\boldsymbol{x}_i a_k(0)
    \\=&\boldsymbol{b}_k(0)^\top\boldsymbol{x}_i-\eta_0\frac{1}{n}\sum_{j\in[n]}\tilde{\ell}'\big({y}_j{f}(\boldsymbol{x}_j;\boldsymbol{\theta}(0))\mathbb{I}\big\{\boldsymbol{b}_k(0)^\top\boldsymbol{x}_j>0\big\}\Big(y_j\boldsymbol{x}_j^\top\boldsymbol{x}_i y_i\Big)\Big( y_i a_k(0)\Big)
    \\=&\boldsymbol{b}_k(0)^\top\boldsymbol{x}_i-\eta_0\frac{1}{n}\tilde{\ell}'\big({y}_i{f}(\boldsymbol{x}_i;\boldsymbol{\theta}(0))\mathbb{I}\big\{\boldsymbol{b}_k(0)^\top\boldsymbol{x}_i>0\big\} \boldsymbol{x}_i^\top\boldsymbol{x}_i a_k(0)
    \\&\quad\quad\quad\quad-\eta_0\frac{1}{n}\sum_{j\ne i}\tilde{\ell}'\big({y}_j{f}(\boldsymbol{x}_j;\boldsymbol{\theta}(0))\mathbb{I}\big\{\boldsymbol{b}_k(0)^\top\boldsymbol{x}_j>0\big\}\Big(y_j\boldsymbol{x}_j^\top\boldsymbol{x}_i y_i\Big)\Big(y_i a_k(0)\Big)
    \\\geq&\boldsymbol{b}_k(0)^\top\boldsymbol{x}_i-\eta_0\frac{1}{n}\tilde{\ell}'\big({y}_i{f}(\boldsymbol{x}_i;\boldsymbol{\theta}(0))\mathbb{I}\big\{\boldsymbol{b}_k(0)^\top\boldsymbol{x}_i>0\big\} \boldsymbol{x}_i^\top\boldsymbol{x}_i a_k(0)
    \\\geq&-\frac{2\kappa}{\sqrt{m}}+\frac{\eta_0}{n\sqrt{m}}\Big(-\max\limits_{z\in[-2\kappa,2\kappa]}\tilde{\ell}(z)\Big)>0
\end{align*}

Combining Lemma \ref{lemma: parameter estimate first step}, we know $a_k(1)>0$, so we have $k\in\mathcal{TL}_i(1)$.

Proof of (S2).

Let $k\in\mathcal{FD}_i(0)$, we have $a_k(0)<0$ and $ \boldsymbol{b}_k(0)^\top\boldsymbol{x}_i\leq0$. Combining Assumption \ref{ass: data separation} (ii), we know there exists $j_i\in[n]$, s.t. $\boldsymbol{b}_k(0)^\top\boldsymbol{x}_{j_i}>0$ and $y_i\boldsymbol{x}_i^\top\boldsymbol{x}_{j_i} y_{j_i}\geq\mu_0$.
Hence, we have
\begin{align*}
    &\boldsymbol{b}_k(1)^\top\boldsymbol{x}_i
    \\=&\boldsymbol{b}_k(0)^\top\boldsymbol{x}_i-\eta_0\frac{1}{n}\sum_{j\in[n]}\tilde{\ell}'\big({y}_j{f}(\boldsymbol{x}_j;\boldsymbol{\theta}(0))\mathbb{I}\big\{\boldsymbol{b}_k(0)^\top\boldsymbol{x}_j>0\big\}y_j\boldsymbol{x}_j^\top\boldsymbol{x}_i a_k(0)
    \\=&\boldsymbol{b}_k(0)^\top\boldsymbol{x}_i-\eta_0\frac{1}{n}\sum_{j\in[n]}\tilde{\ell}'\big({y}_j{f}(\boldsymbol{x}_j;\boldsymbol{\theta}(0))\mathbb{I}\big\{\boldsymbol{b}_k(0)^\top\boldsymbol{x}_j>0\big\}\Big(y_j\boldsymbol{x}_j^\top\boldsymbol{x}_i y_i\Big)\Big( y_i a_k(0)\Big)
    \\=&\boldsymbol{b}_k(0)^\top\boldsymbol{x}_i-\eta_0\frac{1}{n}\tilde{\ell}'\big({y}_{j_i}{f}(\boldsymbol{x}_{j_i};\boldsymbol{\theta}(0))\mathbb{I}\big\{\boldsymbol{b}_k(0)^\top\boldsymbol{x}_{j_i}>0\big\} \boldsymbol{x}_{j_i}^\top\boldsymbol{x}_{j_i} a_k(0)
    \\&\quad\quad\quad\quad-\eta_0\frac{1}{n}\sum_{j\ne {j_i}}\tilde{\ell}'\big({y}_j{f}(\boldsymbol{x}_j;\boldsymbol{\theta}(0))\mathbb{I}\big\{\boldsymbol{b}_k(0)^\top\boldsymbol{x}_j>0\big\}\Big(y_j\boldsymbol{x}_j^\top\boldsymbol{x}_i y_i\Big)\Big(y_i a_k(0)\Big)
    \\\leq&\boldsymbol{b}_k(0)^\top\boldsymbol{x}_i-\eta_0\frac{1}{n}\tilde{\ell}'\big({y}_{j_i}{f}(\boldsymbol{x}_{j_i};\boldsymbol{\theta}(0))\mathbb{I}\big\{\boldsymbol{b}_k(0)^\top\boldsymbol{x}_{j_i}>0\big\} \boldsymbol{x}_{j_i}^\top\boldsymbol{x}_{j_i} a_k(0)
    \\\leq&\frac{2\kappa}{\sqrt{m}}+\frac{\eta_0\mu_0}{n\sqrt{m}}\Big(\min\limits_{z\in[-2\kappa,2\kappa]}\tilde{\ell}(z)\Big)<0
\end{align*}
Combining Lemma \ref{lemma: parameter estimate first step}, we know $a_k(1)<0$, so we have $k\in\mathcal{FD}_i(1)$.

Proof of (S3).

Let $k\in\mathcal{TD}_i(0)$, we have $a_k(0)>0$ and $\boldsymbol{b}_k(0)^\top\boldsymbol{x}_i\leq0$.
Combining Assumption \ref{ass: data separation} (ii), we know there exists $j_i\in[n]$, s.t. $\boldsymbol{b}_k(0)^\top\boldsymbol{x}_{j_i}>0$ and $y_i\boldsymbol{x}_i^\top\boldsymbol{x}_{j_i} y_{j_i}\geq\mu_0$.
Hence, we have
\begin{align*}
    &\boldsymbol{b}_k(1)^\top\boldsymbol{x}_i
    \\=&\boldsymbol{b}_k(0)^\top\boldsymbol{x}_i-\eta_0\frac{1}{n}\sum_{j\in[n]}\tilde{\ell}'\big({y}_j{f}(\boldsymbol{x}_j;\boldsymbol{\theta}(0))\mathbb{I}\big\{\boldsymbol{b}_k(0)^\top\boldsymbol{x}_j>0\big\}y_j\boldsymbol{x}_j^\top\boldsymbol{x}_i a_k(0)
    \\=&\boldsymbol{b}_k(0)^\top\boldsymbol{x}_i-\eta_0\frac{1}{n}\sum_{j\in[n]}\tilde{\ell}'\big({y}_j{f}(\boldsymbol{x}_j;\boldsymbol{\theta}(0))\mathbb{I}\big\{\boldsymbol{b}_k(0)^\top\boldsymbol{x}_j>0\big\}\Big(y_j\boldsymbol{x}_j^\top\boldsymbol{x}_i y_i\Big)\Big( y_i a_k(0)\Big)
    \\=&\boldsymbol{b}_k(0)^\top\boldsymbol{x}_i-\eta_0\frac{1}{n}\tilde{\ell}'\big({y}_{j_i}{f}(\boldsymbol{x}_{j_i};\boldsymbol{\theta}(0))\mathbb{I}\big\{\boldsymbol{b}_k(0)^\top\boldsymbol{x}_{j_i}>0\big\} \boldsymbol{x}_{j_i}^\top\boldsymbol{x}_{j_i} a_k(0)
    \\&\quad\quad\quad\quad-\eta_0\frac{1}{n}\sum_{j\ne {j_i}}\tilde{\ell}'\big({y}_j{f}(\boldsymbol{x}_j;\boldsymbol{\theta}(0))\mathbb{I}\big\{\boldsymbol{b}_k(0)^\top\boldsymbol{x}_j>0\big\}\Big(y_j\boldsymbol{x}_j^\top\boldsymbol{x}_i y_i\Big)\Big(y_i a_k(0)\Big)
    \\\geq&\boldsymbol{b}_k(0)^\top\boldsymbol{x}_i-\eta_0\frac{1}{n}\tilde{\ell}'\big({y}_{j_i}{f}(\boldsymbol{x}_{j_i};\boldsymbol{\theta}(0))\mathbb{I}\big\{\boldsymbol{b}_k(0)^\top\boldsymbol{x}_{j_i}>0\big\} \boldsymbol{x}_{j_i}^\top\boldsymbol{x}_{j_i} a_k(0)
    \\\geq&-\frac{2\kappa}{\sqrt{m}}+\frac{\eta_0\mu_0}{n\sqrt{m}}\Big(-\max\limits_{z\in[-2\kappa,2\kappa]}\tilde{\ell}(z)\Big)>0.
\end{align*}
Combining Lemma \ref{lemma: parameter estimate first step}, we know $a_k(1)>0$, so we have $k\in\mathcal{TL}_i(1)$.

Proof of (S4).

Let $k\in\mathcal{FL}_i(0)$, we have $a_k(0)<0$ and $ \boldsymbol{b}_k(0)^\top\boldsymbol{x}_i>0$. With the help of Lemma \ref{lemma: initial norm and prediction global}, we have 
\begin{align*}
    &\boldsymbol{b}_k(1)^\top\boldsymbol{x}_i
    \\=&\boldsymbol{b}_k(0)^\top\boldsymbol{x}_i-\eta_0\frac{1}{n}\sum_{j\in[n]}\tilde{\ell}'\big({y}_j{f}(\boldsymbol{x}_j;\boldsymbol{\theta}(0))\mathbb{I}\big\{\boldsymbol{b}_k(0)^\top\boldsymbol{x}_j>0\big\}y_j\boldsymbol{x}_j^\top\boldsymbol{x}_i a_k(0)
    \\=&\boldsymbol{b}_k(0)^\top\boldsymbol{x}_i-\eta_0\frac{1}{n}\sum_{j\in[n]}\tilde{\ell}'\big({y}_j{f}(\boldsymbol{x}_j;\boldsymbol{\theta}(0))\mathbb{I}\big\{\boldsymbol{b}_k(0)^\top\boldsymbol{x}_j>0\big\}\Big(y_j\boldsymbol{x}_j^\top\boldsymbol{x}_i y_i\Big)\Big( y_i a_k(0)\Big)
    \\=&\boldsymbol{b}_k(0)^\top\boldsymbol{x}_i-\eta_0\frac{1}{n}\tilde{\ell}'\big({y}_i{f}(\boldsymbol{x}_i;\boldsymbol{\theta}(0))\mathbb{I}\big\{\boldsymbol{b}_k(0)^\top\boldsymbol{x}_i>0\big\} \boldsymbol{x}_i^\top\boldsymbol{x}_i a_k(0)
    \\&\quad\quad\quad\quad-\eta_0\frac{1}{n}\sum_{j\ne i}\tilde{\ell}'\big({y}_j{f}(\boldsymbol{x}_j;\boldsymbol{\theta}(0))\mathbb{I}\big\{\boldsymbol{b}_k(0)^\top\boldsymbol{x}_j>0\big\}\Big(y_j\boldsymbol{x}_j^\top\boldsymbol{x}_i y_i\Big)\Big(y_i a_k(0)\Big)
    \\\leq&\boldsymbol{b}_k(0)^\top\boldsymbol{x}_i-\eta_0\frac{1}{n}\tilde{\ell}'\big({y}_i{f}(\boldsymbol{x}_i;\boldsymbol{\theta}(0))\mathbb{I}\big\{\boldsymbol{b}_k(0)^\top\boldsymbol{x}_i>0\big\} \boldsymbol{x}_i^\top\boldsymbol{x}_i a_k(0)
    \\\leq&\frac{2\kappa}{\sqrt{m}}+\frac{\eta_0}{n\sqrt{m}}\Big(\max\limits_{z\in[-2\kappa,2\kappa]}\tilde{\ell}(z)\Big)<0
\end{align*}

Combining Lemma \ref{lemma: parameter estimate first step}, we know $a_k(1)<0$, so we have $k\in\mathcal{FD}_i(1)$.

(S5) and (S6) are the corollaries of (S1)(S2)(S3)(S4). And from the proof of (S1)(S2)(S3)(S4), we can obtain (S7).

\end{proof}

\subsubsection{Stage II: neurons keep good directions}

\begin{lemma}[Dynamics of neuron partition at Stage II]\label{lemma: correct neurons remain correct stage II} Under the same condition of Lemma \ref{lemma: correct neurons remain correct stage I}, we have the following results for any $t\geq 1$ and $i\in[n]$:\\
(S1) $a_k(t+1)\text{sgn}(a_k(t))\geq a_k(t)\text{sgn}(a_k(t))$, for all $k\in[m]$.\\
(S2) True-living neurons remain true-living:  $\mathcal{TL}_i(t)\subset\mathcal{TL}_i(t+1)$.\\
(S3) False-dead neurons remain false-dead:  $\mathcal{FD}_i(t)\subset\mathcal{FD}_i(t+1)$.\\
(S4) The neuron partition holds: $[m]\equiv\mathcal{TL}_i(t+1)\cup\mathcal{FD}_i(t+1)$.\\
(S5) For any $i\in[n]$, $k\in[m]$ and $\boldsymbol{\theta}\in\overline{\boldsymbol{b}_k(t)\boldsymbol{b}_k(t+1)}$, we have $\text{sgn}(\boldsymbol{b}^\top\boldsymbol{x}_i)\equiv\text{sgn}(\boldsymbol{b}_k(1)^\top\boldsymbol{x}_i)\neq0$.

\end{lemma}

\begin{proof}[Proof of Lemma \ref{lemma: correct neurons remain correct stage II}]\ \\
WLOG, we only need to prove $i\in[n/2]$ ($y_i=1$).

Proof of (S1).

From Lemma \ref{lemma: correct neurons remain correct stage I} (S5), we know $[m]=\mathcal{TL}_i(t)\cup\mathcal{FD}_i(t)$, $t\geq1$.

If $k\in\mathcal{TL}_i(t)$, we have $a_k(t)>0$. Then for any $j\in[n]-[n/2]$, we have $a_k(t)y_j < 0$, so $k\in\mathcal{FL}_j(t)\cup\mathcal{FD}_j(t)$. Combining Lemma \ref{lemma: correct neurons remain correct stage I} (S5), we know $k\in\mathcal{FD}_j(t)$, which means $\boldsymbol{b}_k(t)^\top\boldsymbol{x}_j\leq0$. So we have:
\begin{align*}
    a_k(t+1)
    =&a_k(t)-\eta_t\frac{1}{n}\sum_{j=1}^n\tilde{\ell}'\big(y_j f(\boldsymbol{x}_j;\boldsymbol{\theta}(t))\big) \sigma\big(\boldsymbol{b}_k(t)^\top\boldsymbol{x}_j\big) y_j
    \\=&a_k(t)-\eta_t\frac{1}{n}\sum_{j\in[n/2]}\tilde{\ell}'\big(y_j f(\boldsymbol{x}_j;\boldsymbol{\theta}(t))\big) \sigma\big(\boldsymbol{b}_k(t)^\top\boldsymbol{x}_j\big) y_j-0
    \\=&a_k(t)-\eta_t\frac{1}{n}\sum_{j\in[n/2]}\tilde{\ell}'\big(y_j f(\boldsymbol{x}_j;\boldsymbol{\theta}(t))\big) \sigma\big(\boldsymbol{b}_k(1)^\top\boldsymbol{x}_j\big)
    \\\geq&a_k(t).
\end{align*}

If $k\in\mathcal{FD}_i(t)$, we have $a_k(t)<0$. Then for any $j\in[n/2]$, we have $a_k(t)y_j < 0$, so $k\in\mathcal{FL}_j(t)\cup\mathcal{FD}_j(t)$. Combining Lemma \ref{lemma: correct neurons remain correct stage I} (S5), we know $k\in\mathcal{FD}_j(t)$, which means $\boldsymbol{b}_k(t)^\top\boldsymbol{x}_j\leq0$. So we have:
\begin{align*}
    -a_k(t+1)
    =&-a_k(t)+\eta_t\frac{1}{n}\sum_{j=1}^n\tilde{\ell}'\big(y_j f(\boldsymbol{x}_j;\boldsymbol{\theta}(t))\big) \sigma\big(\boldsymbol{b}_k(t)^\top\boldsymbol{x}_j\big) y_j
    \\=&-a_k(t)+\eta_t\frac{1}{n}\sum_{j\in[n]-[n/2]}\tilde{\ell}'\big(y_j f(\boldsymbol{x}_j;\boldsymbol{\theta}(t))\big) \sigma\big(\boldsymbol{b}_k(t)^\top\boldsymbol{x}_j\big) y_j-0
    \\=&-a_k(t)-\eta_t\frac{1}{n}\sum_{j\in[n]-[n/2]}\tilde{\ell}'\big(y_j f(\boldsymbol{x}_j;\boldsymbol{\theta}(t))\big) \sigma\big(\boldsymbol{b}_k(1)^\top\boldsymbol{x}_j\big)
    \\\geq&-a_k(t).
\end{align*}

Proof of (S2). 

Let $k\in\mathcal{TL}_i(1)$, we have $a_k(1)>0$ and $\boldsymbol{b}_k(1)^\top\boldsymbol{x}_i>0$. 

We will prove a stronger result than (S2), ``Under the same conditions, we have $\boldsymbol{b}^\top\boldsymbol{x}_i>0$ for any $\boldsymbol{b}\in\overline{\boldsymbol{b}_k(t)\boldsymbol{b}_k(t+1)}$'':

There exists $\alpha\in[0,1]$, s.t. $\boldsymbol{b}=\alpha\boldsymbol{b}_k(t)+(1-\alpha)\boldsymbol{b}_k(t+1)$, and
\begin{align*}
    &\boldsymbol{b}^\top\boldsymbol{x}_i
    \\
    =&\alpha\boldsymbol{b}_k(t)^\top\boldsymbol{x}_i+(1-\alpha)\boldsymbol{b}_k(t+1)^\top\boldsymbol{x}_i
    \\=&\boldsymbol{b}_k(t)^\top\boldsymbol{x}_i-(1-\alpha)\eta_t\frac{1}{n}\sum_{j=1}^n\tilde{\ell}'\big({y}_j f(\boldsymbol{x}_j;\boldsymbol{\theta}(t))
    \big)\mathbb{I}\Big\{\boldsymbol{b}_k(t)^\top\boldsymbol{x}_j>0\Big\}y_j\boldsymbol{x}_j^\top\boldsymbol{x}_i a_k(1)
    \\=&\boldsymbol{b}_k(t)^\top\boldsymbol{x}_i-(1-\alpha)\eta_t\frac{1}{n}\sum_{j=1}^n\tilde{\ell}'\big({y}_j{f}(\boldsymbol{x}_j;\boldsymbol{\theta}(t))
    \big)\mathbb{I}\Big\{\boldsymbol{b}_k(t)^\top\boldsymbol{x}_j>0\Big\}\big(y_j\boldsymbol{x}_j^\top\boldsymbol{x}_i y_i\big) \big(y_i a_k(t)\big)
    \\\geq&\boldsymbol{b}_k(t)^\top\boldsymbol{x}_i.
\end{align*}
By induction, we know $\boldsymbol{b}^\top\boldsymbol{x}_i\geq\boldsymbol{b}_k(1)^\top\boldsymbol{x}_i$. Recalling Lemma \ref{lemma: correct neurons remain correct stage I} (S7), we know $\boldsymbol{b}^\top\boldsymbol{x}_i\geq\boldsymbol{b}_k(1)^\top\boldsymbol{x}_i>0$.
Combining (S1), we have $k\in\mathcal{TL}_i(t+1)$.

Proof of (S3).

Let $k\in\mathcal{FD}_i(1)$, we have $a_k(1)<0$ and $ \boldsymbol{b}_k(1)^\top\boldsymbol{x}_i<0$.

We will prove a stronger result than (S3), ``Under the same conditions, we have $\boldsymbol{b}^\top\boldsymbol{x}_i<0$ for any $\boldsymbol{b}\in\overline{\boldsymbol{b}_k(t)\boldsymbol{b}_k(t+1)}$'':

There exists $\alpha\in[0,1]$, s.t. $\boldsymbol{b}=\alpha\boldsymbol{b}_k(t)+(1-\alpha)\boldsymbol{b}_k(t+1)$, and
\begin{align*}
    &\boldsymbol{b}^\top\boldsymbol{x}_i
    \\
    =&\alpha\boldsymbol{b}_k(t)^\top\boldsymbol{x}_i+(1-\alpha)\boldsymbol{b}_k(t+1)^\top\boldsymbol{x}_i
    \\=&\boldsymbol{b}_k(t)^\top\boldsymbol{x}_i-(1-\alpha)\eta_t\frac{1}{n}\sum_{j=1}^n\tilde{\ell}'\big({y}_j f(\boldsymbol{x}_j;\boldsymbol{\theta}(t))
    \big)\mathbb{I}\Big\{\boldsymbol{b}_k(t)^\top\boldsymbol{x}_j>0\Big\}y_j\boldsymbol{x}_j^\top\boldsymbol{x}_i a_k(1)
    \\=&\boldsymbol{b}_k(t)^\top\boldsymbol{x}_i-(1-\alpha)\eta_t\frac{1}{n}\sum_{j=1}^n\tilde{\ell}'\big({y}_j{f}(\boldsymbol{x}_j;\boldsymbol{\theta}(t))
    \big)\mathbb{I}\Big\{\boldsymbol{b}_k(t)^\top\boldsymbol{x}_j>0\Big\}\big(y_j\boldsymbol{x}_j^\top\boldsymbol{x}_i y_i\big) \big(y_i a_k(t)\big)
    \\\leq&\boldsymbol{b}_k(t)^\top\boldsymbol{x}_i.
\end{align*}
By induction, we know $\boldsymbol{b}^\top\boldsymbol{x}_i\leq\boldsymbol{b}_k(1)^\top\boldsymbol{x}_i$. Recalling Lemma \ref{lemma: correct neurons remain correct stage I} (S7), we know $\boldsymbol{b}^\top\boldsymbol{x}_i\leq\boldsymbol{b}_k(1)^\top\boldsymbol{x}_i<0$.
Combining (S1), we have $k\in\mathcal{FD}_i(t+1)$.

Proof of (S4). (S4) is the corollary of (S1)(S2)(S3).

Proof of (S5). Recalling the stronger results in our proof of (S2)(S3), we have proved (S5).

\end{proof}

The following lemma indicates that GD will always classify samples correctly after the first step.
\begin{lemma}[Correct Classification]\label{lemma: correct classification}
For any $i\in[n]$ and $t\geq 1$, we have:
\begin{align*}
    y_i f(\boldsymbol{x}_i;\boldsymbol{\theta}(t))>0.
\end{align*}

\end{lemma}
\begin{proof}[Proof of Lemma \ref{lemma: correct classification}]\ \\
From Lemma \ref{lemma: correct neurons remain correct stage I} (S5) and \ref{lemma: correct neurons remain correct stage II} (S5), for any $t\geq 1$ we have:
\begin{align*}
    y_i f(\boldsymbol{x}_i;\boldsymbol{\theta}(t))=&y_i\sum_{k=1}^m a_k(t)\sigma\big(\boldsymbol{b}_k(t)^\top\boldsymbol{x}_i\big)=\sum_{k=1}^m \mathbb{I}\big\{\boldsymbol{b}_k(t)^\top\boldsymbol{x}_i>0\big\}\boldsymbol{b}_k(t)^\top\boldsymbol{x}_i a_k(t)y_i
    \\=&
    \sum_{k\in\mathcal{TL}_i(t)\cup\mathcal{FL}_i(t)} \mathbb{I}\big\{\boldsymbol{b}_k(t)^\top\boldsymbol{x}_i>0\big\}\boldsymbol{b}_k(t)^\top\boldsymbol{x}_i a_k(t)y_i
    \\=&
    \sum_{k\in\mathcal{TL}_i(t)} \sigma\big(\boldsymbol{b}_k(t)^\top\boldsymbol{x}_i\big) a_k(t)y_i
    \\\geq&
        \sum_{k\in\mathcal{TL}_i(t)} \sigma\big(\boldsymbol{b}_k(1)^\top\boldsymbol{x}_i\big) a_k(1)y_i
    >0.
\end{align*}

\end{proof}

\subsection{Gradient lower bound}\label{subsection: gradient lower bound binary}
\begin{remark}\rm
The results in this section will be discussed and proved when the events in Lemma \ref{lemma: estimate of the number of initial neural partition global} and \ref{lemma: initial norm and prediction global} happened. So all the results below hold with probability at least $1-\delta-2me^{-2d}$.
\end{remark}

The following lemma is a crucial lemma about the lower bound of Gram matrix, which can be proved by our analysis about neuron partition below.

\begin{lemma}[Gram matrix estimate]\label{lemma: lower bound of Gram binary} Define the Gram matrix $\mathbf{G}(t)\in\mathbb{R}^{n\times n}$ at $t$ as
\[\mathbf{G}(t)=\Big(\nabla f(\boldsymbol{x}_i;\boldsymbol{\theta}(t))^\top\nabla f(\boldsymbol{x}_j;\boldsymbol{\theta}(t))\Big)_{(i,j)\in[n]\times[n]}.\]

Then for any $i,j\in[\frac{n}{2}]$ and $t\geq 1$, we have
\[
\mathbf{G}(t)=\begin{pmatrix}
  \mathbf{G}_{+}(t) & \mathbf{0}_{\frac{n}{2}\times\frac{n}{2}}\\\mathbf{0}_{\frac{n}{2}\times\frac{n}{2}}&\mathbf{G}_{-}(t)
\end{pmatrix},
\]
And if $\eta_0\leq1/3$, we have
\begin{gather*}
\text{G}_{+}(t)_{i,j},\ \text{G}_{-}(t)_{i,j}\geq \Big(\frac{1}{2}-
\sqrt{\frac{8\log(n^2/\delta)}{m}}\Big)\frac{\boldsymbol{x}_i^\top\boldsymbol{x}_j}{2}.
\end{gather*}
\end{lemma}

\begin{proof}[Proof of Lemma \ref{lemma: lower bound of Gram binary}]\ \\
From the form of two-layer neural networks, We can calculate:
\begin{align*}
    &\nabla f(\boldsymbol{x}_i;\boldsymbol{\theta}(t))^\top\nabla f(\boldsymbol{x}_j;\boldsymbol{\theta}(t))
    \\=&
    \sum_{k=1}^m\Bigg(\left<\frac{\partial f(\boldsymbol{x}_i;\boldsymbol{\theta}(t))}{\partial a_k},\frac{\partial f(\boldsymbol{x}_j;\boldsymbol{\theta}(t))}{\partial a_k}\right>+\left<\frac{\partial f(\boldsymbol{x}_i;\boldsymbol{\theta}(t))}{\partial \boldsymbol{b}_k},\frac{\partial f(\boldsymbol{x}_j;\boldsymbol{\theta}(t))}{\partial \boldsymbol{b}_k}\right>\Bigg)
    \\=&
    \sum_{k=1}^m \Bigg(\sigma\Big(\boldsymbol{b}_k(t)^\top\boldsymbol{x}_i\Big)\sigma\Big(\boldsymbol{b}_k(t)^\top\boldsymbol{x}_j\Big)+
    a_k^2(t)\sigma'\Big(\boldsymbol{b}_k(t)^\top\boldsymbol{x}_i\Big)\sigma'\Big(\boldsymbol{b}_k(t)^\top\boldsymbol{x}_j\Big)\boldsymbol{x}_i^\top\boldsymbol{x}_j\Bigg)
    \\=&
    \sum_{k=1}^m\Big(\boldsymbol{b}_k(t)^\top\boldsymbol{x}_i\boldsymbol{b}_k(0)^\top\boldsymbol{x}_j+a_k^2(t)\boldsymbol{x}_i^\top\boldsymbol{x}_j\Big)\mathbb{I}\Big\{\boldsymbol{b}_k(t)^\top\boldsymbol{x}_i> 0\Big\}\mathbb{I}\Big\{\boldsymbol{b}_k(t)^\top\boldsymbol{x}_j>0\Big\}.
    \end{align*}
    
    For any $i,j$, s.t. $|i-j|\geq \frac{n}{2}$, we know $y_iy_j=-1$. From Lemma \ref{lemma: correct neurons remain correct stage I} (S5) and \ref{lemma: correct neurons remain correct stage II} (S5), we know
    \begin{gather*}
    \boldsymbol{b}_k(t)^\top\boldsymbol{x}_i> 0
    \Rightarrow k\in\mathcal{TL}_i(t)\Rightarrow a_k(t)y_i>0,\\
     \boldsymbol{b}_k(t)^\top\boldsymbol{x}_j\geq 0
    \Rightarrow k\in\mathcal{TL}_j(t)\Rightarrow a_k(t)y_j>0,
    \end{gather*}
    which implies that:
    \[
    \Big\{k\in[m]:\boldsymbol{b}_k(t)^\top\boldsymbol{x}_i> 0,\boldsymbol{b}_k(t)^\top\boldsymbol{x}_j> 0\Big\}=\varnothing.
    \]
    So we have:
\[
\mathbf{G}(t)=\begin{pmatrix}
  \mathbf{G}_{+}(t) & \mathbf{0}_{\frac{n}{2}\times\frac{n}{2}}\\\mathbf{0}_{\frac{n}{2}\times\frac{n}{2}}&\mathbf{G}_{-}(t)
\end{pmatrix},\  \forall t\geq 1.
\]

For any $i,j\in[\frac{n}{2}]$, we only need to consider the lower bound of $\text{G}_{+}(t)_{i,j}$. ( $\text{G}_{-}(t)_{i,j}$ is similar.)
    With the help of Lemma \ref{lemma: correct neurons remain correct stage II} (S1)(S2)(S5), Lemma \ref{lemma: correct neurons remain correct stage I} (S1) and Lemma \ref{lemma: estimate of the number of initial neural partition global}, we have the following estimate for any $t
    \geq1$:
    \begin{align*}
    &\nabla f(\boldsymbol{x}_i;\boldsymbol{\theta}(t))^\top\nabla f(\boldsymbol{x}_j;\boldsymbol{\theta}(t))
    \\\geq
    &
    \sum_{k=1}^m a_k^2(t)\boldsymbol{x}_i^\top\boldsymbol{x}_j\mathbb{I}\Big\{\boldsymbol{b}_k(t)^\top\boldsymbol{x}_i> 0\Big\}\mathbb{I}\Big\{\boldsymbol{b}_k(t)^\top\boldsymbol{x}_j>0\Big\}
    \\=&
    \sum_{k\in\mathcal{TL}_i(t)\cap\mathcal{TL}_j(t)}a_k^2(t)\boldsymbol{x}_i^\top\boldsymbol{x}_j
    \\=&
    \sum_{k\in\big(\mathcal{TL}_i(0)\cup\mathcal{TD}_i(0)\big)\bigcap\big(\mathcal{TL}_j(0)\cup\mathcal{TD}_j(0)\big)}a_k^2(t)\boldsymbol{x}_i^\top\boldsymbol{x}_j
     \\=&
    \sum_{k\in\big(\mathcal{TL}_i(0)\cup\mathcal{TD}_i(0)\big)\bigcap\big(\mathcal{TL}_j(0)\cup\mathcal{TD}_j(0)\big)}a_k^2(0)\boldsymbol{x}_i^\top\boldsymbol{x}_j
    \\\geq & m\Big(\frac{\pi-\arccos(\boldsymbol{x}_i^\top\boldsymbol{x}_j)}{\pi}-\sqrt{\frac{8\log(n^2/\delta)}{m}}\Big)\frac{(1-2\eta_0^2/n)^2}{m}\boldsymbol{x}_i^\top\boldsymbol{x}_j
    \\\geq&\Big(\frac{\pi-\arccos(\boldsymbol{x}_i^\top\boldsymbol{x}_j)}{\pi}-\sqrt{\frac{8\log(n^2/\delta)}{m}}\Big)\frac{\boldsymbol{x}_i^\top\boldsymbol{x}_j}{2}
    \\\geq&
    \Big(\frac{1}{2}-\sqrt{\frac{8\log(n^2/\delta)}{m}}\Big)\frac{\boldsymbol{x}_i^\top\boldsymbol{x}_j}{2}>0.
\end{align*}

\end{proof}

\begin{lemma}[Gradient lower bound]\label{lemma: gradient lower bound global GD}
Under the same condition of Lemma \ref{lemma: lower bound of Gram binary}, we have the following gradient lower bound for any $ t\geq1$:
\begin{align*}
&
\left\|\nabla \mathcal{L}(\boldsymbol{\theta}(t))\right\|^2
\\\geq&\frac{1}{16}\Big(\frac{1}{2}-
\sqrt{\frac{8\log(n^2/\delta)}{m}}\Big)\max\Big\{\frac{2}{n}+\frac{n-2}{n}\gamma,\lambda_{\min}(\mathbf{X}_{+}^\top\mathbf{X}_+)\wedge\lambda_{\min}(\mathbf{X}_{-}^\top\mathbf{X}_-)\Big\}\mathcal{L}(\boldsymbol{\theta}(t))^2.
\end{align*}

where $\mathbf{X}_+:=(\boldsymbol{x}_1,\cdots,\boldsymbol{x}_\frac{n}{2})$, $\mathbf{X}_-:=(\boldsymbol{x}_{\frac{n}{2}+1},\cdots,\boldsymbol{x}_n)$ and $\gamma=\min\limits_{i,j\text{ in the same class}}\boldsymbol{x}_i^\top\boldsymbol{x}_j$.
\end{lemma}

\begin{proof}[Proof of Lemma \ref{lemma: gradient lower bound global GD}]\ \\
From Lemma \ref{lemma: correct classification}, we have $y_i f(\boldsymbol{x}_i;\boldsymbol{\theta}(t))\geq0$ for any $t\geq 1$ and $i\in[n]$, so we have:
\begin{align*}
    \tilde{\ell}'(y_i f(\boldsymbol{x}_i;\boldsymbol{\theta}(t)))\leq
    -\frac{1}{2}\tilde{\ell}(y_i f(\boldsymbol{x}_i;\boldsymbol{\theta}(t))).
\end{align*}

Combining Lemma \ref{lemma: lower bound of Gram binary}, we have:
\begin{align*}
    &\left\|\nabla \mathcal{L}(\boldsymbol{\theta}(t))\right\|^2
    \\=&
    \frac{1}{n^2}\sum_{i\in[n]}\sum_{j\in[n]}\left<\nabla_{\boldsymbol{\theta}} \tilde{\ell}(y_i f(\boldsymbol{x}_i;\boldsymbol{\theta}(t))),\nabla_{\boldsymbol{\theta}} \tilde{\ell}(y_j f(\boldsymbol{x}_j;\boldsymbol{\theta}(t)))\right>
    \\=&
    \frac{1}{n^2}\sum_{i\in[n]}\sum_{j\in[n]}
    \tilde{\ell}'(y_i f(\boldsymbol{x}_i;\boldsymbol{\theta}(t)))\tilde{\ell}'(y_j f(\boldsymbol{x}_j;\boldsymbol{\theta}(t)))\left<\nabla f(\boldsymbol{x}_i;\boldsymbol{\theta}(t)),\nabla f(\boldsymbol{x}_j;\boldsymbol{\theta}(t))\right>y_i y_j\\
    =&\frac{1}{n^2}\sum_{i\in[\frac{n}{2}]}\sum_{j\in[\frac{n}{2}]}\tilde{\ell}'(y_i f(\boldsymbol{x}_i;\boldsymbol{\theta}(t)))\tilde{\ell}'(y_j f(\boldsymbol{x}_j;\boldsymbol{\theta}(t)))\left<\nabla f(\boldsymbol{x}_i;\boldsymbol{\theta}(t)),\nabla f(\boldsymbol{x}_j;\boldsymbol{\theta}(t))\right>y_i y_j
    \\&\quad\quad+\frac{1}{n^2}\sum_{i\in[n]-[\frac{n}{2}]}\sum_{j\in[n]-[\frac{n}{2}]}\tilde{\ell}'(y_i f(\boldsymbol{x}_i;\boldsymbol{\theta}(t)))\tilde{\ell}'(y_j f(\boldsymbol{x}_j;\boldsymbol{\theta}(t)))\left<\nabla f(\boldsymbol{x}_i;\boldsymbol{\theta}(t)),\nabla f(\boldsymbol{x}_j;\boldsymbol{\theta}(t))\right>y_i y_j
    \\{\geq}&\Big(\frac{1}{2}-\sqrt{\frac{8\log(n^2/\delta)}{m}}\Big)\frac{1}{8n^2}\Big(\sum_{i\in[\frac{n}{2}]}\sum_{j\in[\frac{n}{2}]}\tilde{\ell}(y_i f(\boldsymbol{x}_i;\boldsymbol{\theta}(t)))\tilde{\ell}(y_j f(\boldsymbol{x}_j;\boldsymbol{\theta}(t))){\boldsymbol{x}_i^\top\boldsymbol{x}_j}
\\&\quad\quad\quad\quad\quad\quad\quad\quad\quad\quad\quad\quad+\sum_{i\in[n]-[\frac{n}{2}]}\sum_{j\in[n]-[\frac{n}{2}]}\tilde{\ell}(y_i f(\boldsymbol{x}_i;\boldsymbol{\theta}(t)))\tilde{\ell}(y_j f(\boldsymbol{x}_j;\boldsymbol{\theta}(t))){\boldsymbol{x}_i^\top\boldsymbol{x}_j}\Big).
\end{align*}
There are two lower bounds for the RHS of the inequality above. For simplify, we denote $l_i:=\tilde{\ell}(y_i f(\boldsymbol{x}_i;\boldsymbol{\theta}(t)))$

The first lower bound is:
\begin{align*}
    \text{RHS}=&\Big(\frac{1}{2}-\sqrt{\frac{2\log(n^2/\delta)}{m}}\Big)\frac{1}{8n^2}\Big(\sum_{i\in[\frac{n}{2}]}\sum_{j\in[\frac{n}{2}]}l_i l_j{\boldsymbol{x}_i^\top\boldsymbol{x}_j}+\sum_{i\in[n]-[\frac{n}{2}]}\sum_{j\in[n]-[\frac{n}{2}]}l_i l_j{\boldsymbol{x}_i^\top\boldsymbol{x}_j}\Big)
    \\=&
    \Big(\frac{1}{2}-\sqrt{\frac{8\log(n^2/\delta)}{m}}\Big)\frac{1}{8n^2}\Big(\sum_{i\in[n]} l_i^2+\sum_{i,j\in[\frac{n}{2}],\ i\ne j}l_i l_j{\boldsymbol{x}_i^\top\boldsymbol{x}_j}+\sum_{i,j\in[n]-[\frac{n}{2}],\ i\ne j}l_i l_j{\boldsymbol{x}_i^\top\boldsymbol{x}_j}\Big)
    \\\geq&
    \Big(\frac{1}{2}-\sqrt{\frac{8\log(n^2/\delta)}{m}}\Big)\frac{1}{8n^2}\Big(\sum_{i\in[n]} l_i^2+s\sum_{i,j\in[\frac{n}{2}],\ i\ne j}l_i l_j+s\sum_{i,j\in[n]-[\frac{n}{2}],\ i\ne j}l_i l_j\Big)
    \\=&
    \Big(\frac{1}{2}-\sqrt{\frac{8\log(n^2/\delta)}{m}}\Big)\frac{1}{8}\Big(\frac{1-s}{n^2}\sum_{i\in[n]} l_i^2+\frac{s}{n^2}(\sum_{i\in[\frac{n}{2}]}l_i)^2+\frac{s}{n^2}(\sum_{i\in[n]-[\frac{n}{2}]}l_i)^2\Big)
    \\\geq&
    \Big(\frac{1}{2}-\sqrt{\frac{8\log(n^2/\delta)}{m}}\Big)\frac{1}{8}\Big(\frac{1-s}{n^3}(\sum_{i\in[n]} l_i)^2+\frac{s}{2n^2}(\sum_{i\in[n]}l_i)^2\Big)
    \\=&
    \Big(\frac{1}{2}-\sqrt{\frac{8\log(n^2/\delta)}{m}}\Big)\frac{1}{8}\Big(\frac{1-s}{n}+\frac{s}{2}\Big)\mathcal{L}(\boldsymbol{\theta}(t))^2.
\end{align*}

From the definition of $\mathbf{X}_+$ and $\mathbf{X}_-$, we have the second lower bound:
\begin{align*}
    \text{RHS}
    =&\Big(\frac{1}{2}-\sqrt{\frac{8\log(n^2/\delta)}{m}}\Big)\frac{1}{8n^2}\Big(\sum_{i\in[\frac{n}{2}]}\sum_{j\in[\frac{n}{2}]}l_i l_j{\boldsymbol{x}_i^\top\boldsymbol{x}_j}+\sum_{i\in[n]-[\frac{n}{2}]}\sum_{j\in[n]-[\frac{n}{2}]}l_i l_j{\boldsymbol{x}_i^\top\boldsymbol{x}_j}\Big)
    \\=&
    \Big(\frac{1}{2}-\sqrt{\frac{8\log(n^2/\delta)}{m}}\Big)\frac{1}{8n^2}\Big((l_1,\cdots,l_\frac{n}{2})\mathbf{X}_{+}^\top\mathbf{X}_+(l_{1},\cdots,l_{\frac{n}{2}})^\top
    \\&\quad\quad\quad\quad\quad\quad\quad\quad\quad\quad\quad
    +(l_{\frac{n}{2}+1},\cdots,l_{n})\mathbf{X}_{-}^\top\mathbf{X}_-(l_{\frac{n}{2}+1},\cdots,l_{n})^\top\Big)
    \\\geq&
    \Big(\frac{1}{2}-\sqrt{\frac{8\log(n^2/\delta)}{m}}\Big)\frac{1}{8n^2}\Big( \lambda_{\min}(\mathbf{X}_{+}^\top\mathbf{X}_+)(\sum_{i\in[\frac{n}{2}]}l_i)^2+\lambda_{\min}(\mathbf{X}_{-}^\top\mathbf{X}_-)(\sum_{i\in[n]-[\frac{n}{2}]}l_i)^2\Big)
    \\\geq&
    \Big(\frac{1}{2}-\sqrt{\frac{8\log(n^2/\delta)}{m}}\Big)\frac{\lambda_{\min}(\mathbf{X}_{+}^\top\mathbf{X}_+)\wedge\lambda_{\min}(\mathbf{X}_{-}^\top\mathbf{X}_-)}{8n^2}\Big((\sum_{i\in[\frac{n}{2}]}l_i)^2+(\sum_{i\in[n]-[\frac{n}{2}]}l_i)^2\Big)
    \\\geq&
    \Big(\frac{1}{2}-\sqrt{\frac{8\log(n^2/\delta)}{m}}\Big)\frac{\lambda_{\min}(\mathbf{X}_{+}^\top\mathbf{X}_+)\wedge\lambda_{\min}(\mathbf{X}_{-}^\top\mathbf{X}_-)}{16n^2}(\sum_{i\in[n]}l_i)^2
    \\=&
    \Big(\frac{1}{2}-\sqrt{\frac{8\log(n^2/\delta)}{m}}\Big)\frac{\lambda_{\min}(\mathbf{X}_{+}^\top\mathbf{X}_+)\wedge\lambda_{\min}(\mathbf{X}_{-}^\top\mathbf{X}_-)}{16}\mathcal{L}(\boldsymbol{\theta}(t))^2.
\end{align*}

\end{proof}

\subsection{Global convergence at polynomial rate}
\begin{remark}\rm
The results in this section will be discussed and proved when the events in Lemma \ref{lemma: estimate of the number of initial neural partition global} and \ref{lemma: initial norm and prediction global} happened. So all the results below hold with probability at least $1-\delta-2me^{-2d}$.
\end{remark}

\begin{lemma}[Gradient upper bound]\label{lemma: gradient upper bound glocal GD} 
\[
\left\|\nabla \mathcal{L}(\boldsymbol{\theta})\right\|
\leq\left\|\boldsymbol{\theta}\right\|\mathcal{L}(\boldsymbol{\theta}).
\]
\end{lemma}

\begin{proof}[Proof of Lemma \ref{lemma: gradient upper bound glocal GD}]\ \\
From
\begin{align*}
    &\left|\tilde{\ell}'(y_if(\boldsymbol{x}_i;\boldsymbol{\theta}))\right|
    \leq\tilde{\ell}(y_if(\boldsymbol{x}_i;\boldsymbol{\theta})),
\end{align*}
we have
\begin{align*}
    \left\|\nabla \mathcal{L}(\boldsymbol{\theta})\right\|
    =&\left\|\frac{1}{n}\sum_{i=1}^n\tilde{\ell}'(y_if(\boldsymbol{x}_i;\boldsymbol{\theta}))y_i\nabla f(\boldsymbol{x}_i;\boldsymbol{\theta})\right\|
    \\\leq&
    \Big(\frac{1}{n}\sum_{i=1}^n \tilde{\ell}(y_if(\boldsymbol{x}_i;\boldsymbol{\theta}))\Big)\sup_{\left\|\boldsymbol{x}\right\|\leq1}\left\|\nabla f(\boldsymbol{x};\boldsymbol{\theta})\right\|
    \\\leq&
    \mathcal{L}(\boldsymbol{\theta})\sup_{\left\|\boldsymbol{x}\right\|=1}\sqrt{\sum_{k=1}^m\Big(\sigma^2(\boldsymbol{b}_k^\top\boldsymbol{x})+a_k^2\sigma'(\boldsymbol{b}_k^\top\boldsymbol{x})^2\left\|\boldsymbol{x}\right\|^2\Big)}
    \\\leq&
     \mathcal{L}(\boldsymbol{\theta})\sqrt{\sum_{k=1}^m (a_k^2+\left\|\boldsymbol{b}_k\right\|^2)}=\left\|\boldsymbol{\theta}\right\|\mathcal{L}(\boldsymbol{\theta}).
\end{align*}
\end{proof}

\begin{lemma}[Parameter norm estimate]\label{lemma: norm estimate}
If we use the learning rate
\[\eta_t=\frac{c}{t\mathcal{L}(\boldsymbol{\theta}(t))},\ t\geq1,
\]
where $c<1/2$, then
we have the dynamical estimate of parameter norms for any $t\geq 1$:
\[
1-2\eta_0^2/n\leq\left\|\boldsymbol{\theta}(t)\right\|\leq
 \left\|\boldsymbol{\theta}(1)\right\|\sqrt{1+\frac{2ct}{1-2c}}.
\]
\end{lemma}
\begin{proof}[Proof of Lemma \ref{lemma: norm estimate}]\ \\
\textbf{Step I.} Proof of the second $\leq$.

For simplify, we denote $P=\frac{2c}{1-2c}$.

For $t=1$, it holds naturally. 

Assume it holds for $s\leq t$, then for $s=t+1$ we have:
\begin{align*}
    \left\|\boldsymbol{\theta}(t+1)\right\|\leq&\left\|\boldsymbol{\theta}(t)\right\|+\eta_t\left\|\nabla\mathcal{L}(\boldsymbol{\theta}(t))\right\|
    \\\overset{\text{Lemma \ref{lemma: gradient upper bound glocal GD}}}{\leq}&
    \left\|\boldsymbol{\theta}(t)\right\|+\eta_t\left\|\boldsymbol{\theta}(t)\right\|\mathcal{L}(\boldsymbol{\theta}(t))
    \\\leq&
    \left\|\boldsymbol{\theta}(1)\right\|\sqrt{1+P(t-1)}\Big(1+\frac{c}{t}\Big)
    \\=&
    \left\|\boldsymbol{\theta}(1)\right\|
    \Big(\sqrt{1+Pt}-\sqrt{1+Pt}+\sqrt{1+P(t-1)}+\frac{c\sqrt{1+P(t-1)}}{t}\Big)
        \\=&
    \left\|\boldsymbol{\theta}(1)\right\|\sqrt{1+Pt}+\left\|\boldsymbol{\theta}(1)\right\|
    \Big(\sqrt{1+P(t-1)}-\sqrt{1+Pt}+\frac{c\sqrt{1+P(t-1)}}{t}\Big)
    \\=&
        \left\|\boldsymbol{\theta}(1)\right\|\sqrt{1+Pt}+\left\|\boldsymbol{\theta}(1)\right\|
    \Big(\frac{c\sqrt{1+P(t-1)}}{t}-\frac{P}{\sqrt{1+P(t-1)}+\sqrt{1+Pt}}\Big)
        \\\leq&
        \left\|\boldsymbol{\theta}(1)\right\|\sqrt{1+Pt}+\left\|\boldsymbol{\theta}(1)\right\|
    \Big(\frac{c\sqrt{1+Pt}}{t}-\frac{P}{2\sqrt{1+Pt}}\Big)
            \\=&
        \left\|\boldsymbol{\theta}(1)\right\|\sqrt{1+Pt}+\left\|\boldsymbol{\theta}(1)\right\|
    \Big(\frac{2c+(2c-1)Pt}{2t\sqrt{1+Pt}}\Big)
    \\\leq&
        \left\|\boldsymbol{\theta}(1)\right\|\sqrt{1+Pt}+\left\|\boldsymbol{\theta}(1)\right\|
    \Big(\frac{2c+(2c-1)P}{2t\sqrt{1+Pt}}\Big)\leq \left\|\boldsymbol{\theta}(1)\right\|\sqrt{1+Pt}.
\end{align*}

By induction, we have proved the second $\leq$.

\textbf{Step II.} Proof of the first $\leq$.

From Lemma \ref{lemma: correct neurons remain correct stage II} (S1) and Lemma \ref{lemma: parameter estimate first step}, for any $t\geq 1$ we have:
\begin{align*}
    \left\|\boldsymbol{\theta}(t)\right\|^2\geq\sum_{k=1}^m a_k(t)^2\geq\sum_{k=1}^m a_k(1)^2\geq m \frac{(1-2\eta_0^2/n)^2}{m}=(1-2\eta_0^2/n)^2.
\end{align*}

\end{proof}

\begin{lemma}[Hessian upper bound]\label{lemma: Hessian upper bound global GD} For any $\boldsymbol{\theta}=(a_1,\cdots,a_m,\boldsymbol{b}_1^\top,\cdots,\boldsymbol{b}_m^\top)^\top$, if $\boldsymbol{b}_k^\top\boldsymbol{x}_i\ne0$ holds for any $k\in[m]$ and $i\in[n]$, we have
\[
\left\|\nabla^2 \mathcal{L}(\boldsymbol{\theta})\right\|
\leq\big(\left\|\boldsymbol{\theta}\right\|^2+1\big)\mathcal{L}(\boldsymbol{\theta}).
\]
\end{lemma}
\begin{proof}[Proof of Lemma \ref{lemma: Hessian upper bound global GD}]\ \\
\begin{align*}
    \nabla^2 \mathcal{L}(\boldsymbol{\theta})
    = \frac{1}{n}\sum_{i=1}^n\mathbf{H}_{A,i}(\boldsymbol{\theta}) +\frac{1}{n}\sum_{i=1}^n \mathbf{H}_{B,i}(\boldsymbol{\theta}),
\end{align*}
where
\begin{gather*}
    \mathbf{H}_{A,i}(\boldsymbol{\theta}):=\tilde{\ell}''(y_if(\boldsymbol{x}_i;\boldsymbol{\theta})) y_i^2\nabla f(\boldsymbol{x}_i;\boldsymbol{\theta})\nabla f(\boldsymbol{x}_i;\boldsymbol{\theta})^\top,
    \\
    \mathbf{H}_{B,i}(\boldsymbol{\theta}):=\tilde{\ell}'(y_if(\boldsymbol{x}_i;\boldsymbol{\theta}))y_i\nabla^2 f(\boldsymbol{x}_i;\boldsymbol{\theta}).
\end{gather*}

Now we need to estimate $\left\|\mathbf{H}_{A,i}(\boldsymbol{\theta})\right\|$ and $\left\|\mathbf{H}_{B,i}(\boldsymbol{\theta})\right\|$ respectively.

For the first part $\mathbf{H}_{A,i}(\boldsymbol{\theta})$, from
\begin{align*}
    &\tilde{\ell}''(y_if(\boldsymbol{x}_i;\boldsymbol{\theta}))\leq\tilde{\ell}(y_if(\boldsymbol{x}_i;\boldsymbol{\theta})),
\end{align*}
we have:
\begin{align*}
    \left\|\mathbf{H}_{A,i}(\boldsymbol{\theta})\right\|
    =&
    \sup\limits_{\boldsymbol{w}\ne 0}\frac{\boldsymbol{w}^\top\nabla f(\boldsymbol{x}_i;\boldsymbol{\theta})\nabla f(\boldsymbol{x}_i;\boldsymbol{\theta})^\top\boldsymbol{w}}{\left\|\boldsymbol{w}\right\|^2}\tilde{\ell}''(y_if(\boldsymbol{x}_i;\boldsymbol{\theta}))
    \\\leq&\sup\limits_{\boldsymbol{w}\ne 0}\frac{\boldsymbol{w}^\top\nabla f(\boldsymbol{x}_i;\boldsymbol{\theta})\nabla f(\boldsymbol{x}_i;\boldsymbol{\theta})^\top\boldsymbol{w}}{\left\|\boldsymbol{w}\right\|^2}\tilde{\ell}(y_if(\boldsymbol{x}_i;\boldsymbol{\theta}))
    \\=&
    \sup_{u_1,\boldsymbol{v}_1,\cdots,u_m,\boldsymbol{v}_m {\rm\ not\  all\ }0}\frac{\sum\limits_{k=1}^m\Big(u_k\sigma(\boldsymbol{b}_k^\top\boldsymbol{x}_i)+\boldsymbol{v}_k^\top\boldsymbol{x}_i\sigma'(\boldsymbol{b}_k^\top\boldsymbol{x}_i)a_k\Big)^2}{\sum\limits_{k=1}^m\Big(u_k^2+\left\|\boldsymbol{v}_k\right\|^2\Big)}\tilde{\ell}(y_if(\boldsymbol{x}_i;\boldsymbol{\theta}))
    \\\leq&
        \sup_{u_1,\boldsymbol{v}_1,\cdots,u_m,\boldsymbol{v}_m {\rm\ not\  all\ }0}\frac{\sum\limits_{k=1}^m\Big(u_k\left\|\boldsymbol{b}_k\right\|+\left\|\boldsymbol{v}_k\right\||a_k|\Big)^2}{\sum\limits_{k=1}^m\Big(u_k^2+\left\|\boldsymbol{v}_k\right\|^2\Big)}\tilde{\ell}(y_if(\boldsymbol{x}_i;\boldsymbol{\theta}))
        \\\leq&\left\|\boldsymbol{\theta}\right\|^2\tilde{\ell}(y_if(\boldsymbol{x}_i;\boldsymbol{\theta})).
\end{align*}

For the second part $\mathbf{H}_{B,i}(\boldsymbol{\theta})$,
from
\begin{align*}
    &\left|\tilde{\ell}'(y_if(\boldsymbol{x}_i;\boldsymbol{\theta}))\right|
    \leq\tilde{\ell}(y_if(\boldsymbol{x}_i;\boldsymbol{\theta})),
\end{align*}

we have:
\begin{align*}
\left\|\mathbf{H}_{B,i}(\boldsymbol{\theta})\right\|\leq&\left|\ell'(y_if(\boldsymbol{x}_i;\boldsymbol{\theta}))\right|\sup\limits_{\left\|\boldsymbol{x}\right\|\leq1}\left\|\nabla^2 f(\boldsymbol{x};\boldsymbol{\theta})\right\|
\\\leq&\ell(y_if(\boldsymbol{x}_i;\boldsymbol{\theta}))\sup\limits_{\left\|\boldsymbol{x}\right\|\leq1}\left\|\nabla^2 f(\boldsymbol{x};\boldsymbol{\theta})\right\|.
\end{align*}
So we only need to estimate $\left\|\nabla^2 f(\boldsymbol{x}_i;\boldsymbol{\theta})\right\|$.

$\nabla^2 f(\boldsymbol{x};\boldsymbol{\theta})$ is combined by $m$ diagonal block:
\[
\nabla^2 f(\boldsymbol{x}_i;\boldsymbol{\theta}):=\left(\begin{array}{ccc}  
    \mathbf{H}_f^{(1)} &  & \\
     & \ddots & \\ 
      &  &\mathbf{H}_f^{(m)}\\
  \end{array}\right),
\]
where the $k-$th block $\mathbf{H}_f^{(k)}$ is:
\begin{align*}
\mathbf{H}_f^{(k)}=\begin{pmatrix} 
    \frac{\partial^2 f(\boldsymbol{x}_i;\boldsymbol{\theta})}{\partial^2 a_k} & \frac{\partial^2 f(\boldsymbol{x}_i;\boldsymbol{\theta})}{\partial a_k{\partial \boldsymbol{b}_k}^\top}\\
   \frac{\partial^2 f(\boldsymbol{x}_i;\boldsymbol{\theta})}{\partial \boldsymbol{b}_k\partial a_k}  & \frac{\partial^2 f(\boldsymbol{x}_i;\boldsymbol{\theta})}{\partial \boldsymbol{b}_k{\partial \boldsymbol{b}_k}^\top} 
  \end{pmatrix}
  =
  \begin{pmatrix} 
    0 & \sigma'\Big(\boldsymbol{b}_k^\top\boldsymbol{x}_i\Big)\boldsymbol{x}_i^\top\\
   \sigma'\Big(\boldsymbol{b}_k^\top\boldsymbol{x}_i\Big)\boldsymbol{x}_i  & a_k\sigma''\Big(\boldsymbol{b}_k^\top\boldsymbol{x}_i\Big)\boldsymbol{x}_i\boldsymbol{x}_i^\top
  \end{pmatrix}.
\end{align*}

Becuase $\boldsymbol{b}_k^\top\boldsymbol{x}_i\ne0$ holds for any $k\in[m]$ and $i\in[n]$, $\sigma''(\boldsymbol{b}_k^\top\boldsymbol{x}_i)=0$ holds strictly. So we have the form
\begin{align*}
\mathbf{H}_f^{(k)}=\begin{pmatrix} 
    0 & \sigma'\Big(\boldsymbol{b}_k^\top\boldsymbol{x}_i\Big)\boldsymbol{x}_i^\top\\
   \sigma'\Big(\boldsymbol{b}_k^\top\boldsymbol{x}_i\Big)\boldsymbol{x}_i  & \mathbf{0}_{d\times d}
  \end{pmatrix}.
\end{align*}

And we can estimate:
\begin{align*}
    \left\|\nabla^2 f(\boldsymbol{x};\boldsymbol{\theta})\right\|
    =&\sup\limits_{\boldsymbol{v}_1,\cdots,\boldsymbol{v}_m {\rm are\ not\ all\ 0}}\frac{\sum\limits_{k=1}^m\boldsymbol{v}_k^\top\mathbf{H}_f^{(k)}\boldsymbol{v}_k }{\sum\limits_{k=1}^m\left\|\boldsymbol{v}_k\right\|^2}
    \\\leq&
    \max_{k\in[m]}\left\|\mathbf{H}_f^{(k)}\right\|
    =\max_{i\in[n]}\sup_{(u,\boldsymbol{v})\ne(0,\mathbf{0})}\frac{2\boldsymbol{v}^\top\boldsymbol{x}_i u}{u^2+\left\|\boldsymbol{v}\right\|^2}
    \\\leq&
    \sup_{(u,\boldsymbol{v})\ne(0,\mathbf{0})}\frac{2\left\|\boldsymbol{v}\right\||u|}{u^2+\left\|\boldsymbol{v}\right\|^2}\leq 1.
\end{align*}
So we have the estimate of the second part:
\begin{align*}
\left\|\mathbf{H}_{B,i}(\boldsymbol{\theta})\right\|\leq\tilde{\ell}(y_if(\boldsymbol{x}_i;\boldsymbol{\theta})).
\end{align*}

Combining the two estimates, we obtain this lemma:
\begin{align*}
    &\left\|\nabla^2 \mathcal{L}(\boldsymbol{\theta})\right\|\leq
\frac{1}{n}\sum_{i=1}^n\left\|\mathbf{H}_{A,i}(\boldsymbol{\theta})\right\|+\frac{1}{n}\sum_{i=1}^n
\left\|\mathbf{H}_{B,i}(\boldsymbol{\theta})\right\|
\\\leq&\big(\left\|\boldsymbol{\theta}\right\|^2+1\big)\frac{1}{n}\sum_{i=1}^n\tilde{\ell}(y_if(\boldsymbol{x}_i;\boldsymbol{\theta}))\leq\big(\left\|\boldsymbol{\theta}\right\|^2+1\big)\mathcal{L}(\boldsymbol{\theta}).
\end{align*}

\end{proof}

\begin{lemma}[Hessian upper bound along GD trajectory]\label{lemma: hessian upper bound GD trajectory}
Let $\boldsymbol{\theta}(t)$ be trained by Gradient Descent, and we define $H_t=\sup\limits_{\boldsymbol{\theta}\in\overline{\boldsymbol{\theta}(t)\boldsymbol{\theta}(t+1)}}\left\|\nabla^2 \mathcal{L}(\boldsymbol{\theta})\right
\|$. If $\eta_t$ are chosen in the same way as Lemma \ref{lemma: norm estimate} $\eta_0\leq\frac{1}{2\sqrt{2}}$ and $c<\frac{1}{2(1+\left\|\boldsymbol{\theta}(1)\right\|^2)}$, we have the following estimate for any $t\geq1$:
\[
H_t\leq6\left\|\boldsymbol{\theta}(1)\right\|^2\Big(1+\frac{2ct}{1-2c}\Big)\mathcal{L}(\boldsymbol{\theta}(t)).
\]
\end{lemma}

\begin{proof}[Proof of Lemma \ref{lemma: hessian upper bound GD trajectory}]\ \\
First, we estimate  $V_t:=\sup\limits_{\boldsymbol{\theta}\in\overline{\boldsymbol{\theta}(t)\boldsymbol{\theta}(t+1)}}\mathcal{L}(\boldsymbol{\theta}(t))$. For any $\boldsymbol{\theta}_t^\alpha:=\boldsymbol{\theta}(t)+\alpha(\boldsymbol{\theta}(t+1)-\boldsymbol{\theta}(t))$ ($\alpha\in[0,1]$), there exist $\beta\in[0,1]$, s.t.
\begin{align*}
    \mathcal{L}(\boldsymbol{\theta}_t^\alpha)=&\mathcal{L}(\boldsymbol{\theta}(t))+\left<\nabla \mathcal{L}(\boldsymbol{\theta}_t^\beta),\boldsymbol{\theta}_t^\alpha-\boldsymbol{\theta}(t)\right>
    \\\leq&
    \mathcal{L}(\boldsymbol{\theta}(t))+\alpha\eta_t\left\|\nabla \mathcal{L}(\boldsymbol{\theta}_t^\beta)\right\|\left\|\nabla\mathcal{L}(\boldsymbol{\theta}(t))\right\|
    \\\overset{\text{Lemma \ref{lemma: gradient upper bound glocal GD}}}{\leq}&
     \mathcal{L}(\boldsymbol{\theta}(t))+\alpha\eta_t\left\|\boldsymbol{\theta}_t^\beta\right\|\left\|\boldsymbol{\theta}(t)\right\|\mathcal{L}(\boldsymbol{\theta}_t^\beta)\mathcal{L}(\boldsymbol{\theta}(t))
         \\\overset{\text{Lemma \ref{lemma: norm estimate}}}{\leq}&
     \mathcal{L}(\boldsymbol{\theta}(t))+\alpha\eta_t\mathcal{L}(\boldsymbol{\theta}(t))\left\|\boldsymbol{\theta}(1)\right\|^2\Big(1+\frac{2ct}{1-2c}\Big)\mathcal{L}(\boldsymbol{\theta}_t^\beta)
              \\{\leq}&
     \mathcal{L}(\boldsymbol{\theta}(t))+\frac{c\left\|\boldsymbol{\theta}(1)\right\|^2}{t}\Big(1+\frac{2ct}{1-2c}\Big)V_t
                   \\{\leq}&
     \mathcal{L}(\boldsymbol{\theta}(t))+c\left\|\boldsymbol{\theta}(1)\right\|^2\Big(1+\frac{2c}{1-2c}\Big)V_t
\end{align*}
So we have $
    V_t\leq\mathcal{L}(\boldsymbol{\theta}(t))+\frac{c}{1-2c}\left\|\boldsymbol{\theta}(1)\right\|^2 V_t$. From the selection of $c<\frac{1}{2(1+\left\|\boldsymbol{\theta}(1)\right\|^2)}$, we have:
\[
V_t\leq2\mathcal{L}(\boldsymbol{\theta}(t)).
\]

Recalling Lemma \ref{lemma: correct neurons remain correct stage II} (S5), we know that: for any $\boldsymbol{\theta}=(a_1,\cdots,a_m,\boldsymbol{b}_1^\top,\cdots,\boldsymbol{b}_m^\top)^\top\in\overline{\boldsymbol{\theta}(t)\boldsymbol{\theta}(t+1)}$ $(t\geq 1)$, we have $\boldsymbol{b}_k^\top\boldsymbol{x}_i\ne0$ for any $k\in[m]$ and $i\in[n]$. Hence, $\sigma''(\boldsymbol{b}_k^\top\boldsymbol{x}_i)=0$ holds strictly. Combing Lemma \ref{lemma: Hessian upper bound global GD}, we have the following estimate for any $t\geq 1$:

\begin{align*}
H_t
\leq&\sup\limits_{\boldsymbol{\theta}\in\overline{\boldsymbol{\theta}(t)\boldsymbol{\theta}(t+1)}}\big(\left\|\boldsymbol{\theta}\right\|^2+1\big)\mathcal{L}(\boldsymbol{\theta})\leq
 V_t\big(\sup\limits_{\boldsymbol{\theta}\in\overline{\boldsymbol{\theta}(t)\boldsymbol{\theta}(t+1)}}\left\|\boldsymbol{\theta}\right\|^2+1\big)
 \\\overset{\text{Lemma \ref{lemma: norm estimate}}, \eta_0<1/2\sqrt{2}}{\leq}&
 3V_t \left\|\boldsymbol{\theta}(1)\right\|^2\Big(1+\frac{2ct}{1-2c}\Big)
 \leq6\left\|\boldsymbol{\theta}(1)\right\|^2\Big(1+\frac{2ct}{1-2c}\Big)\mathcal{L}(\boldsymbol{\theta}(t)).
\end{align*}

\end{proof}

\begin{lemma}[Theorem 4.3 in \cite{lyu2019gradient}]\label{lemma: lyu2019} Under Assumption \ref{def: exponential type loss} on the loss function, let $\boldsymbol{\theta}(t)$ be the parameters of model \eqref{equ: model 2NN nobias} trained by GD \eqref{equ: alg GD}.
If $\mathcal{L}(\boldsymbol{\theta}(T_0))<\frac{\tilde{\ell}(0)}{n}$ and the learning rate $\eta_{t}=\mathcal{O}\Big(\frac{1}{\mathcal{L}(\boldsymbol{\theta}(t)){\rm poly}\log\frac{1}{\mathcal{L}(\boldsymbol{\theta}(t))}}\Big)$ $(t\geq T_0)$, then we have:
\[
\mathcal{L}(\boldsymbol{\theta}(t))=\Theta\Big(\frac{1}{T(t)\log T(t)}\Big),\ t\geq T_0,
\]
where $T(t)=\sum_{s=T_0}^{t-1}\eta_t$.
\end{lemma}

\begin{theorem}[Restatement of Theorem \ref{thm: binary global GD}]
Under Assumption \ref{def: exponential type loss} and Assumption \ref{ass: data separation}, let $\boldsymbol{\theta}(t)$ be the parameters of model \eqref{equ: model 2NN nobias} trained by Gradient Descent \eqref{equ: alg GD} starting from random initialization~\eqref{equ: random initialization}. Let the width $m={\Omega}(\log(n/\delta))$, the initialization scale $\kappa=\mathcal{O}(\eta_0\mu_0/n)$ in \eqref{equ: random initialization}, the constant $V$ be defined in \eqref{equ: V}, the constant $c\in(0,\frac{1}{6(1+2\eta_0)^2+2}]$, the constant $c'>0$ be sufficiently small, the hitting time $T_0=\lceil(n\tilde{\ell}(0))^{\frac{2}{Vc}}\rceil$, the parameter $r\geq 1$ and the learning rate satisfy
\[
\left \{ \begin{array}{ll}
\eta_0\leq\frac{1}{2\sqrt{2}},& t=0 \\
\eta_t=\frac{c}{t\mathcal{L}(\boldsymbol{\theta}(t))},& 1\leq t< T_0 \\
\eta_t=\frac{c'}{\mathcal{L}(\boldsymbol{\theta}(t))^{1-\frac{1}{2r}}}, & t\geq T_0\\
\end{array}\right..
\]
Then with probability at least $1-\delta-2me^{-2d}$, GD will converge at \textbf{polynomial} rate:
\[
\left \{ \begin{array}{ll}
\mathcal{L}(\boldsymbol{\theta}(t))\leq\frac{ \mathcal{L}(\boldsymbol{\theta}(1))}{t^\frac{V c}{2}},& 1\leq t< T_0 \\
\mathcal{L}(\boldsymbol{\theta}(t))=\mathcal{O}\big(\frac{1}{t^r}\big), & t\geq T_0\\
\end{array}\right..
\]
\end{theorem}

\begin{proof}[\bfseries\color{blue} Proof of Theorem \ref{thm: binary global GD}]\ \\
\textbf{Step I.}
From Lemma \ref{lemma: hessian upper bound GD trajectory} and the selection of $c\leq\frac{1}{6(1+2\eta_0)^2+2}\leq\frac{1}{6(1+\eta_0)(1+2\kappa)+2}\leq\frac{1}{6\left\|\boldsymbol{\theta}(1)\right\|^2+2}$, for any $t\geq1$ we have:
\begin{align*}
&\eta_t\sup\limits_{\boldsymbol{\theta}\in\overline{\boldsymbol{\theta}(t)\boldsymbol{\theta}(t+1)}}\left\|\nabla^2 \mathcal{L}(\boldsymbol{\theta})\right
\|\leq6\eta_t\mathcal{L}(\boldsymbol{\theta}(t))\left\|\boldsymbol{\theta}(1)\right\|^2\Big(1+\frac{2ct}{1-2c}\Big)
\\\leq&\frac{6c}{t}\left\|\boldsymbol{\theta}(1)\right\|^2\Big(1+\frac{2ct}{1-2c}\Big)\leq{6c}\left\|\boldsymbol{\theta}(1)\right\|^2\Big(1+\frac{2c}{1-2c}\Big)
\\=&\frac{\leq{6c}\left\|\boldsymbol{\theta}(1)\right\|^2}{1-2c}\leq1.
\end{align*}

Combining Lemma \ref{lemma: loss quadratic upper bound 1} and Lemma \ref{lemma: gradient lower bound global GD}, we have:
\begin{align*}
     \mathcal{L}(\boldsymbol{\theta}(t+1))
    \leq&
    \mathcal{L}(\boldsymbol{\theta}(t))-\eta_t\Big(1-\frac{1}{2}\Big)\left\|\nabla \mathcal{L}(\boldsymbol{\theta}(t))\right\|^2
    \\\leq&
    \mathcal{L}(\boldsymbol{\theta}(t))-\frac{\eta_t}{2}V\mathcal{L}(\boldsymbol{\theta}(t))^2
    \leq\mathcal{L}(\boldsymbol{\theta}(t))\Big(1-\frac{\eta_t}{2}V\mathcal{L}(\boldsymbol{\theta}(t))\Big)
    \\\leq&
\mathcal{L}(\boldsymbol{\theta}(t))(1-\frac{V c}{2t})\leq\cdots
\\\leq&
\mathcal{L}(\boldsymbol{\theta}(1))(1-\frac{V c}{2})\cdots(1-\frac{V c}{2t})
\\\leq&
\mathcal{L}(\boldsymbol{\theta}(1))\exp\Big(\sum_{s=1}^t\log(1-\frac{Vc}{2s})\Big)
\\\leq&
\mathcal{L}(\boldsymbol{\theta}(1))\exp\Big(-\frac{Vc}{2}\sum_{s=1}^t \frac{1}{s}\Big)
\\\leq&
\mathcal{L}(\boldsymbol{\theta}(1))\exp\Big(-\frac{Vc}{2}\log(t+1)\Big)
\\=&
\frac{\mathcal{L}(\boldsymbol{\theta}(1))}{(t+1)^{Vc/2}}
\end{align*}

Now we have proved that for any $t\geq1$,
\[
\mathcal{L}(\boldsymbol{\theta}(t))\leq\frac{\mathcal{L}(\boldsymbol{\theta}(1))}{t^{Vc/2}}.
\]

\textbf{Step II.} 
Let $T_0=\lceil(n\mathcal{L}(\boldsymbol{\theta}(1)))^{\frac{2}{Vc}}\rceil$, we have $\mathcal{L}(\boldsymbol{\theta}(T_0))\leq\mathcal{L}(\boldsymbol{\theta}(1))<\frac{1}{n}$.

From \textbf{Step I}, we know $\lim\limits_{t\to+\infty}\mathcal{L}(\boldsymbol{\theta}(t))=0$, so there exists a sufficiently small constant $c'>0$, s.t.
\[
\eta_t=\frac{c'}{\mathcal{L}(\boldsymbol{\theta}(t))^{1-\frac{1}{2r}}}=\mathcal{O}\Big(\frac{1}{\mathcal{L}(\boldsymbol{\theta}(t)){\rm poly}\log\frac{1}{\mathcal{L}(\boldsymbol{\theta}(t))}}\Big),\ t\geq T_0.
\]

Then all the conditions in Lemma \ref{lemma: lyu2019} hold for $t\geq T_0$. So there exist $T_1>T_0$ and $M>0$, s.t.
\[
\mathcal{L}(\boldsymbol{\theta}(t))\leq\frac{M}{\Big(\sum_{s=T_0}^{t-1}\eta_s\Big)\log\Big(\sum_{s=T_0}^{t-1}\eta_s\Big)},\ t\geq T_1,
\]

From $\lim\limits_{t\to+\infty}\mathcal{L}(\boldsymbol{\theta}(t))=0$ and the choice of $\eta_t$, we know that there exist $T_2>T_0$, s.t. 
\[\sum_{s=T_0}^{T_2-1}\eta_t>e.
\]

From $r\geq1$, we know that there exist $T_3>T_0$, s.t.
\begin{align*}
&\sum_{s=T_0}^{t} s^{r-\frac{1}{2}}\geq\frac{1}{r+\frac{1}{2}}\sum_{s=T_0}^{t}\Big( s^{r+\frac{1}{2}}-(s-1)^{r+\frac{1}{2}}\Big)
\\=&\frac{1}{r+\frac{1}{2}} \Big((t+1)^{r+\frac{1}{2}}-T_0^{r+\frac{1}{2}}\Big)\geq\frac{1}{2r+1}(t+1)^{r+\frac{1}{2}},\ \forall t \geq T_3.
\end{align*}

There also exists $T_4>T_0$, s.t.
\[\sqrt{t+1}\geq\frac{(2r+1)M}{c'},\  \forall t\geq T_4.\]

Let $T_5=T_1\vee T_2 \vee T_3\vee T_4>T_0$, we know that there exists $Q>1$, s.t. 
\[
\mathcal{L}(\boldsymbol{\theta}(t))\leq\frac{Q}{t^r},\ \forall  t\in[T_0, T_5].
\]

Assume that $\mathcal{L}(\boldsymbol{\theta}(t))\leq\frac{Q}{t^r}$ holds for any $T_5\leq s\leq t$. Then for $s=t+1$, we have:
\begin{align*}
    \mathcal{L}(\boldsymbol{\theta}(t+1))
    \leq&\frac{M}{\Big(\sum_{s=T_0}^{t}\eta_s\Big)\log\Big(\sum_{s=T_0}^{t}\eta_s\Big)}\leq\frac{M}{\sum_{s=T_0}^{t}\eta_s}
    \\=&\frac{M}{c'\sum_{s=T_0}^{t}\frac{1}{\mathcal{L}(\boldsymbol{\theta}(s))^{1-\frac{1}{2r}}}}\leq\frac{M Q^{1-\frac{1}{2r}}}{c'\sum_{s=T_0}^{t}s^{r-\frac{1}{2}} }
    \\\leq&\frac{M Q^{1}}{c'\sum_{s=T_0}^{t}s^{r-\frac{1}{2}} }\leq\frac{(1+2r)MQ}{c'(t+1)^{r+\frac{1}{2}}}\leq\frac{Q}{(t+1)^r}.
\end{align*}

So we obtain:
\[\mathcal{L}(\boldsymbol{\theta}(t))=\mathcal{O}\big(\frac{1}{t^r}\big).\]

Finally, combining \textbf{Step I} ($1\leq t<T_0$) and \textbf{Step II} ($t\geq T_0$), we have proved the theorem for specific numbers $g_a=\frac{1}{2}$, $g_b=1$ and $h=1$ in Assumption \ref{def: exponential type loss}, such as the logistic loss $\ell(\boldsymbol{y}_1,\boldsymbol{y}_2)=\log(1+e^{-\boldsymbol{y}_1^\top\boldsymbol{y}_2})$ (Notice that $\tilde{\ell}(0)=\log 2$.)

In the same way, our result also holds for any other $g_a$, $g_b$ and $h$ in Assumption \ref{def: exponential type loss}. So we complete our proof.
\end{proof}

\newpage

\section{Proof of Theorem \ref{thm: binary global GD linear}}\label{appendix: proof global convergence linear}

\subsection{Preparation and similar lemmas}

The only difference between Theorem \ref{thm: binary global GD linear} and Theorem \ref{thm: binary global GD} is the training of $a_k$.

To make our proof more clear, without loss of generality, we only need to prove the theorem with specific numbers $g_a=\frac{1}{2}$, $g_b=1$ and $h=1$ in Assumption \ref{def: exponential type loss}, such as the \textbf{logistic loss} $\ell(\boldsymbol{y}_1,\boldsymbol{y}_2)=\log(1+e^{-\boldsymbol{y}_1^\top\boldsymbol{y}_2})$. For other loss functions with different $g_a$, $g_b$ and $h$, the proof is similar.

We can prove the following lemmas in the same way as Section \ref{subsection: neuron partition analysis} and \ref{subsection: gradient lower bound binary}.
Then we only need to focus on the convergence rate.

\begin{lemma}[Initial neuron partition]\label{lemma: neuron partition stage 1 linear}
For any $i\in[n]$, we have:
\begin{gather*}
    [m]=\mathcal{TL}_{i}(0)\bigcup\mathcal{TD}_{i}(0)\bigcup\mathcal{FL}_{i}(0)\bigcup\mathcal{FD}_{i}(0).
\end{gather*}
\end{lemma}

\begin{lemma}[Estimate of the initial neural partition]\label{lemma: estimate of the number of initial neural partition global linear} With probability at least $1-\delta$, we have the following estimates
\begin{gather*}
\Bigg|\frac{1}{m}\text{card}\Big(\mathcal{TL}_i(0)\cap\mathcal{TL}_j(0)\Big)-\frac{\pi-\arccos(\boldsymbol{x}_i^\top\boldsymbol{x}_j)}{4\pi}\Bigg|\leq
\sqrt{\frac{\log(n^2/\delta)}{2m}},\text{ for any $i,j$ in the same class};
\\
\Bigg|\frac{1}{m}\text{card}\Big(\mathcal{TL}_i(0)\Big)-\frac{1}{4}\Bigg|\leq
\sqrt{\frac{\log(n^2/\delta)}{2m}},\text{ for any $i\in[n]$}.
\end{gather*}
\end{lemma}

\begin{lemma}[Initial norm and prediction]\label{lemma: initial norm and prediction global linear} With probability at least $1-2me^{-2d}$ we have:
\begin{gather*}
\left\|\boldsymbol{b}_k(0)\right\|\leq\frac{2\kappa}{\sqrt{m}},\ \forall k\in[m];\\
    |f(\boldsymbol{x}_i;\boldsymbol{\theta}(0))|\leq2\kappa,\ \forall i\in[n].
\end{gather*}

\end{lemma}

\begin{lemma}[Dynamics of neuron partition at Stage I]\label{lemma: correct neurons remain correct stage I linear}
When the events in Lemma \ref{lemma: initial norm and prediction global linear} happen, if $\kappa\lesssim\eta_0\mu_0/n$ and $\eta_0\lesssim 1$, we have the following results for any $i\in[n]$:\\
(S1) True-living neurons remain true-living:  $\mathcal{TL}_i(0)\subset\mathcal{TL}_i(1)$.\\
(S2) False-dead neurons remain false-dead:  $\mathcal{FD}_i(0)\subset\mathcal{FD}_i(1)$.\\
(S3) True-dead neurons turn true-living:  $\mathcal{TD}_i(0)\subset\mathcal{TL}_i(1)$ .\\
(S4) False-living neurons turn false-dead: $\mathcal{FL}_i(0)\subset\mathcal{FD}_i(1)$.\\
(S5) The neuron partition holds for $t=1$: $[m]=\mathcal{TL}_i(1)\cup\cup\mathcal{FD}_i(1)$.\\
(S6) The connectivity holds: $\mathcal{FL}_i(0)\cup\mathcal{FD}_i(0)=\mathcal{FD}_i(1)$ and $\mathcal{TL}_i(0)\cup\mathcal{TD}_i(0)=\mathcal{TL}_i(1)$.\\
(S7) $\text{sgn}(\boldsymbol{b}_k^\top(1)\boldsymbol{x}_i)\ne0$ for any $k\in[m]$.

\end{lemma}

\begin{lemma}[Dynamics of neuron partition at Stage II]\label{lemma: correct neurons remain correct stage II linear} Under the same condition of Lemma \ref{lemma: correct neurons remain correct stage I linear}, we have the following results for any $t\geq 1$ and $i\in[n]$:\\
(S1) $a_k(t+1)\text{sgn}(a_k(t))\geq a_k(t)\text{sgn}(a_k(t))$, for all $k\in[m]$.\\
(S2) True-living neurons remain true-living:  $\mathcal{TL}_i(t)\subset\mathcal{TL}_i(t+1)$.\\
(S3) False-dead neurons remain false-dead:  $\mathcal{FD}_i(t)\subset\mathcal{FD}_i(t+1)$.\\
(S4) The neuron partition holds: $[m]\equiv\mathcal{TL}_i(t+1)\cup\mathcal{FD}_i(t+1)$.\\
(S5) For any $i\in[n]$, $k\in[m]$ and $\boldsymbol{\theta}\in\overline{\boldsymbol{b}_k(t)\boldsymbol{b}_k(t+1)}$, we have $\text{sgn}(\boldsymbol{b}^\top\boldsymbol{x}_i)\equiv\text{sgn}(\boldsymbol{b}_k(1)^\top\boldsymbol{x}_i)\neq0$.

\end{lemma}

\begin{lemma}[Correct Classification]\label{lemma: correct classificationc linear}
For any $i\in[n]$ and $t\geq 1$, we have:
\begin{align*}
    y_i f(\boldsymbol{x}_i;\boldsymbol{\theta}(t))>0.
\end{align*}

\end{lemma}

The lemmas below will be discussed when the events in Lemma \ref{lemma: estimate of the number of initial neural partition global linear} and \ref{lemma: initial norm and prediction global linear} happened. So all the results below hold with probability at least $1-\delta-2me^{-2d}$.

\begin{lemma}[Gram matrix estimate]\label{lemma: lower bound of Gram binary linear} Define the Gram matrix $\mathbf{G}(t)\in\mathbb{R}^{n\times n}$ at $t$ as
\[\mathbf{G}(t)=\Big(\nabla f(\boldsymbol{x}_i;\boldsymbol{\theta}(t))^\top\nabla f(\boldsymbol{x}_j;\boldsymbol{\theta}(t))\Big)_{(i,j)\in[n]\times[n]}.\]

Then for any $i,j\in[\frac{n}{2}]$ and $t\geq 1$, we have
\[
\mathbf{G}(t)=\begin{pmatrix}
  \mathbf{G}_{+}(t) & \mathbf{0}_{\frac{n}{2}\times\frac{n}{2}}\\\mathbf{0}_{\frac{n}{2}\times\frac{n}{2}}&\mathbf{G}_{-}(t)
\end{pmatrix},
\]
And if $\eta_0\leq1/2\sqrt{2}$, we have
\begin{gather*}
\text{G}_{+}(t)_{i,j},\ \text{G}_{-}(t)_{i,j}\geq \Big(\frac{1}{2}-
\sqrt{\frac{8\log(n^2/\delta)}{m}}\Big)\frac{\boldsymbol{x}_i^\top\boldsymbol{x}_j}{2}.
\end{gather*}
\end{lemma}

\begin{lemma}[Gradient lower bound]\label{lemma: gradient lower bound global GD linear}
Under the same condition of Lemma \ref{lemma: lower bound of Gram binary linear}, we have the following gradient lower bound for any $ t\geq1$:
\begin{align*}
&
\left\|\nabla \mathcal{L}(\boldsymbol{\theta}(t))\right\|^2
\\\geq&\frac{1}{16}\Big(\frac{1}{2}-
\sqrt{\frac{8\log(n^2/\delta)}{m}}\Big)\max\Big\{\frac{2}{n}+\frac{n-2}{n}\gamma,\lambda_{\min}(\mathbf{X}_{+}^\top\mathbf{X}_+)\wedge\lambda_{\min}(\mathbf{X}_{-}^\top\mathbf{X}_-)\Big\}\mathcal{L}(\boldsymbol{\theta}(t))^2.
\end{align*}

where $\mathbf{X}_+:=(\boldsymbol{x}_1,\cdots,\boldsymbol{x}_\frac{n}{2})$, $\mathbf{X}_-:=(\boldsymbol{x}_{\frac{n}{2}+1},\cdots,\boldsymbol{x}_n)$ and $\gamma=\min\limits_{i,j\text{ in the same class}}\boldsymbol{x}_i^\top\boldsymbol{x}_j$.
\end{lemma}

\subsection{Global convergence at exponential rate}
\begin{remark}\rm
The results in this section will be discussed and proved when the events in Lemma \ref{lemma: estimate of the number of initial neural partition global linear} and \ref{lemma: initial norm and prediction global linear} happened. So all the results below hold with probability at least $1-\delta-2me^{-2d}$.
\end{remark}

\begin{lemma}[Gradient upper bound]\label{lemma: gradient upper bound glocal GD linear}
\[
\left\|\nabla \mathcal{L}(\boldsymbol{\theta})\right\|
\leq\mathcal{L}(\boldsymbol{\theta}).
\]
\end{lemma}

\begin{proof}[Proof of Lemma \ref{lemma: gradient upper bound glocal GD linear}]\ \\
From  
\begin{align*}
    &\left|\tilde{\ell}'(y_if(\boldsymbol{x}_i;\boldsymbol{\theta}))\right|
    \leq\tilde{\ell}(y_if(\boldsymbol{x}_i;\boldsymbol{\theta})),
\end{align*}
we have
\begin{align*}
    \left\|\nabla \mathcal{L}(\boldsymbol{\theta})\right\|
    =&\left\|\frac{1}{n}\sum_{i=1}^n\tilde{\ell}'(y_if(\boldsymbol{x}_i;\boldsymbol{\theta}))y_i\nabla f(\boldsymbol{x}_i;\boldsymbol{\theta})\right\|
    \\\leq&
    \Big(\frac{1}{n}\sum_{i=1}^n \tilde{\ell}(y_if(\boldsymbol{x}_i;\boldsymbol{\theta}))\Big)\sup_{\left\|\boldsymbol{x}\right\|\leq1}\left\|\nabla f(\boldsymbol{x};\boldsymbol{\theta})\right\|
    \\\leq&
    \mathcal{L}(\boldsymbol{\theta})\sup_{\left\|\boldsymbol{x}\right\|=1}\sqrt{\sum_{k=1}^m\frac{1}{m}\sigma'(\boldsymbol{b}_k^\top\boldsymbol{x})^2\left\|\boldsymbol{x}\right\|^2}
    \\\leq&
     \mathcal{L}(\boldsymbol{\theta})\sqrt{\sum_{k=1}^m \frac{1}{m}}=\mathcal{L}(\boldsymbol{\theta}).
\end{align*}
\end{proof}




\begin{lemma}[Hessian upper bound]\label{lemma: Hessian upper bound global GD linear}For any $\boldsymbol{\theta}=(a_1,\cdots,a_m,\boldsymbol{b}_1^\top,\cdots,\boldsymbol{b}_m^\top)^\top$, if $\boldsymbol{b}_k^\top\boldsymbol{x}_i\ne0$ holds for any $k\in[m]$ and $i\in[n]$, we have
\[
\left\|\nabla^2 \mathcal{L}(\boldsymbol{\theta})\right\|
\leq\mathcal{L}(\boldsymbol{\theta}(t)).
\]
\end{lemma}
\begin{proof}[Proof of Lemma \ref{lemma: Hessian upper bound global GD linear}]\ \\
\begin{align*}
    \nabla^2 \mathcal{L}(\boldsymbol{\theta})
    = \frac{1}{n}\sum_{i=1}^n\mathbf{H}_{A,i}(\boldsymbol{\theta}) +\frac{1}{n}\sum_{i=1}^n \mathbf{H}_{B,i}(\boldsymbol{\theta}),
\end{align*}
where
\begin{gather*}
    \mathbf{H}_{A,i}(\boldsymbol{\theta}):=\tilde{\ell}''(y_if(\boldsymbol{x}_i;\boldsymbol{\theta})) y_i^2\nabla f(\boldsymbol{x}_i;\boldsymbol{\theta})\nabla f(\boldsymbol{x}_i;\boldsymbol{\theta})^\top,
    \\
    \mathbf{H}_{B,i}(\boldsymbol{\theta}):=\tilde{\ell}'(y_if(\boldsymbol{x}_i;\boldsymbol{\theta}))y_i\nabla^2 f(\boldsymbol{x}_i;\boldsymbol{\theta}).
\end{gather*}

Now we need to estimate $\left\|\mathbf{H}_{A,i}(\boldsymbol{\theta})\right\|$ and $\left\|\mathbf{H}_{B,i}(\boldsymbol{\theta})\right\|$ respectively.

For the first part $\mathbf{H}_{A,i}(\boldsymbol{\theta})$, from
\begin{align*}
    &\tilde{\ell}''(y_if(\boldsymbol{x}_i;\boldsymbol{\theta}))\leq\tilde{\ell}(y_if(\boldsymbol{x}_i;\boldsymbol{\theta})),
\end{align*}
we have:
\begin{align*}
    \left\|\mathbf{H}_{A,i}(\boldsymbol{\theta})\right\|
    =&
    \sup\limits_{\boldsymbol{w}\ne 0}\frac{\boldsymbol{w}^\top\nabla f(\boldsymbol{x}_i;\boldsymbol{\theta})\nabla f(\boldsymbol{x}_i;\boldsymbol{\theta})^\top\boldsymbol{w}}{\left\|\boldsymbol{w}\right\|^2}\tilde{\ell}''(y_if(\boldsymbol{x}_i;\boldsymbol{\theta}))
    \\\leq&\sup\limits_{\boldsymbol{w}\ne 0}\frac{\boldsymbol{w}^\top\nabla f(\boldsymbol{x}_i;\boldsymbol{\theta})\nabla f(\boldsymbol{x}_i;\boldsymbol{\theta})^\top\boldsymbol{w}}{\left\|\boldsymbol{w}\right\|^2}\tilde{\ell}(y_if(\boldsymbol{x}_i;\boldsymbol{\theta}))
    \\=&
    \sup_{u_1,\boldsymbol{v}_1,\cdots,u_m,\boldsymbol{v}_m {\rm\ not\  all\ }0}\frac{\sum\limits_{k=1}^m\Big(\boldsymbol{v}_k^\top\boldsymbol{x}_i\sigma'(\boldsymbol{b}_k^\top\boldsymbol{x}_i)\frac{1}{\sqrt{m}}\Big)^2}{\sum\limits_{k=1}^m\left\|\boldsymbol{v}_k\right\|^2}\tilde{\ell}(y_if(\boldsymbol{x}_i;\boldsymbol{\theta}))
        \\\leq&\tilde{\ell}(y_if(\boldsymbol{x}_i;\boldsymbol{\theta})).
\end{align*}

For the second part $\mathbf{H}_{B,i}(\boldsymbol{\theta})$,
from
\begin{align*}
    &\left|\tilde{\ell}'(y_if(\boldsymbol{x}_i;\boldsymbol{\theta}))\right|
    \leq\tilde{\ell}(y_if(\boldsymbol{x}_i;\boldsymbol{\theta})),
\end{align*}

we have:
\begin{align*}
\left\|\mathbf{H}_{B,i}(\boldsymbol{\theta})\right\|\leq&\left|\ell'(y_if(\boldsymbol{x}_i;\boldsymbol{\theta}))\right|\sup\limits_{\left\|\boldsymbol{x}\right\|\leq1}\left\|\nabla^2 f(\boldsymbol{x};\boldsymbol{\theta})\right\|
\\\leq&\ell(y_if(\boldsymbol{x}_i;\boldsymbol{\theta}))\sup\limits_{\left\|\boldsymbol{x}\right\|\leq1}\left\|\nabla^2 f(\boldsymbol{x};\boldsymbol{\theta})\right\|.
\end{align*}
So we only need to estimate $\left\|\nabla^2 f(\boldsymbol{x};\boldsymbol{\theta})\right\|$.

$\nabla^2 f(\boldsymbol{x};\boldsymbol{\theta})$ is combined by $m$ diagonal block:
\[
\nabla^2 f(\boldsymbol{x};\boldsymbol{\theta}):=\left(\begin{array}{ccc}  
    \mathbf{H}_f^{(1)} &  & \\
     & \ddots & \\ 
      &  &\mathbf{H}_f^{(m)}\\
  \end{array}\right),
\]
where the $k-$th block $\mathbf{H}_f^{(k)}$ is
\begin{align*}
\mathbf{H}_f^{(k)}= \frac{\partial^2 f(\boldsymbol{x};\boldsymbol{\theta})}{\partial \boldsymbol{b}_k{\partial \boldsymbol{b}_k}^\top} 
  =
   \frac{1}{\sqrt{m}}\sigma''\Big(\boldsymbol{b}_k^\top\boldsymbol{x}\Big)\boldsymbol{x}\boldsymbol{x}^\top.
\end{align*}
Because $\boldsymbol{b}_k^\top\boldsymbol{x}_i\ne0$ for any $k\in[m]$ and $i\in[n]$, $\sigma''(\boldsymbol{b}_k^\top\boldsymbol{x}_i)=0$ holds strictly, so $\mathbf{H}_f^{(k)}=\mathbf{0}_{d\times d}$.
Hence, we have the estimate of the second part:
\begin{align*}
\left\|\mathbf{H}_{B,i}(\boldsymbol{\theta})\right\|=0.
\end{align*}

Combining the two estimates, we obtain this lemma:
\begin{align*}
    &\left\|\nabla^2 \mathcal{L}(\boldsymbol{\theta})\right\|\leq
\frac{1}{n}\sum_{i=1}^n\left\|\mathbf{H}_{A,i}(\boldsymbol{\theta})\right\|+\frac{1}{n}\sum_{i=1}^n
\left\|\mathbf{H}_{B,i}(\boldsymbol{\theta})\right\|
\leq\mathcal{L}(\boldsymbol{\theta}(t)).
\end{align*}

\end{proof}

\begin{lemma}[Hessian upper bound along GD trajectory]\label{lemma: hessian upper bound GD trajectory linear}
Let $\boldsymbol{\theta}(t)$ be trained by Gradient Descent, and we define $H_t=\sup\limits_{\boldsymbol{\theta}\in\overline{\boldsymbol{\theta}(t)\boldsymbol{\theta}(t+1)}}\left\|\nabla^2 \mathcal{L}(\boldsymbol{\theta})\right
\|$. If $\eta_t=\frac{c}{\mathcal{L}(\boldsymbol{\theta}(t))}$ and $c\leq\frac{1}{2}$, we have the following estimate:
\[
H_t\leq2\mathcal{L}(\boldsymbol{\theta}(t)).
\]
\end{lemma}

\begin{proof}[Proof of Lemma \ref{lemma: hessian upper bound GD trajectory linear}]\ \\
First, we estimate $V_t:=\sup\limits_{\boldsymbol{\theta}\in\overline{\boldsymbol{\theta}(t)\boldsymbol{\theta}(t+1)}}\mathcal{L}(\boldsymbol{\theta}(t))$. For any $\boldsymbol{\theta}_t^\alpha:=\boldsymbol{\theta}(t)+\alpha(\boldsymbol{\theta}(t+1)-\boldsymbol{\theta}(t))$ ($\alpha\in[0,1]$), there exist $\beta\in[0,1]$, s.t.
\begin{align*}
    \mathcal{L}(\boldsymbol{\theta}_t^\alpha)=&\mathcal{L}(\boldsymbol{\theta}(t))+\left<\nabla \mathcal{L}(\boldsymbol{\theta}_t^\beta),\boldsymbol{\theta}_t^\alpha-\boldsymbol{\theta}(t)\right>
    \\\leq&
    \mathcal{L}(\boldsymbol{\theta}(t))+\alpha\eta_t\left\|\nabla \mathcal{L}(\boldsymbol{\theta}_t^\beta)\right\|\left\|\nabla\mathcal{L}(\boldsymbol{\theta}(t))\right\|
    \\\overset{\text{Lemma \ref{lemma: gradient upper bound glocal GD linear}}}{\leq}&
     \mathcal{L}(\boldsymbol{\theta}(t))+\alpha\eta_t\mathcal{L}(\boldsymbol{\theta}_t^\beta)\mathcal{L}(\boldsymbol{\theta}(t))
                   \\{\leq}&
     \mathcal{L}(\boldsymbol{\theta}(t))+c V_t
\end{align*}
So we have $
    V_t\leq\mathcal{L}(\boldsymbol{\theta}(t))+c V_t$. From the selection of $c\leq\frac{1}{2}$, we have:
\[
V_t\leq2\mathcal{L}(\boldsymbol{\theta}(t)).
\]

Combining Lemma \ref{lemma: Hessian upper bound global GD linear}, we have
\begin{align*}
H_t
\leq&\sup\limits_{\boldsymbol{\theta}\in\overline{\boldsymbol{\theta}(t)\boldsymbol{\theta}(t+1)}}\mathcal{L}(\boldsymbol{\theta})=
 V_t\leq 2\mathcal{L}(\boldsymbol{\theta}(t)).
\end{align*}

\end{proof}

\begin{theorem}[Restatement of Theorem \ref{thm: binary global GD linear}]
Under Assumption \ref{def: exponential type loss} and \ref{ass: data separation}, let $\{\boldsymbol{b}_k(t)\}_{k\in[m]}$ be trained by Gradient Descent \eqref{equ: alg GD}. Let the width $m={\Omega}(\log(n/\delta))$, $\kappa=\mathcal{O}(\eta_0/n)$, and the learning rate
\begin{align*}
&\eta_0\leq{1}/{2\sqrt{2}},
\\
&\eta_t=\frac{c}{\mathcal{L}(\boldsymbol{\theta}(t))},\ t\geq1,
\end{align*}
where $c\leq\frac{1}{2}$.
Then with probability at least $1-\delta-2me^{-2d}$, for any $t\geq 1$ we have:
\[
\mathcal{L}(\boldsymbol{\theta}(t))\leq\Big(1-\frac{V c}{2}\Big)^{t-1}\tilde{\ell}(0).
\]
where $
V=\frac{1}{16}\Big(\frac{1}{2}-
\sqrt{\frac{8\log(n^2/\delta)}{m}}\Big)\max\Big\{s+\frac{2-2s}{n},\lambda_{\min}(\mathbf{X}_{+}^\top\mathbf{X}_+)\wedge\lambda_{\min}(\mathbf{X}_{-}^\top\mathbf{X}_-)\Big\}>0$,
$\mathbf{X}_+:=(\boldsymbol{x}_1,\cdots,\boldsymbol{x}_\frac{n}{2})$ and $\mathbf{X}_-:=(\boldsymbol{x}_{\frac{n}{2}+1},\cdots,\boldsymbol{x}_n)$.

\end{theorem}

\begin{proof}[\bfseries\color{blue} Proof of Theorem \ref{thm: binary global GD linear}]\ \\
From Lemma \ref{lemma: correct classificationc linear}, we know that neurons adjust their directions in the first step and 
$\mathcal{L}(\boldsymbol{\theta}(1))\leq \tilde{\ell}(0)$, so we only need to prove the following by induction.

For $t=1$, it holds naturally. Assume it hold for $s\leq t$, then for $s=t+1$ we have the following estimates.

From Lemma \ref{lemma: hessian upper bound GD trajectory linear} and the selection of $c\leq\frac{1}{2}$, we know:
\begin{align*}
\eta_t\sup\limits_{\boldsymbol{\theta}\in\overline{\boldsymbol{\theta}(t)\boldsymbol{\theta}(t+1)}}\left\|\nabla^2 \mathcal{L}(\boldsymbol{\theta})\right
\|\leq 2\eta_t\mathcal{L}(\boldsymbol{\theta}(t))\leq 2c\leq 1.
\end{align*}

Combining Lemma \ref{lemma: loss quadratic upper bound 1} and Lemma \ref{lemma: gradient lower bound global GD linear}, we have the estimate:
\begin{align*}
     \mathcal{L}(\boldsymbol{\theta}(t+1))
    \leq&
    \mathcal{L}(\boldsymbol{\theta}(t))-\eta_t\Big(1-\frac{1}{2}\Big)\left\|\nabla \mathcal{L}(\boldsymbol{\theta}(t))\right\|^2
    \\\leq&
    \mathcal{L}(\boldsymbol{\theta}(t))-\frac{\eta_t}{2}V\mathcal{L}(\boldsymbol{\theta}(t))^2
    \leq\mathcal{L}(\boldsymbol{\theta}(t))\Big(1-\frac{\eta_t}{2}V\mathcal{L}(\boldsymbol{\theta}(t))\Big)
    \\\leq&
\mathcal{L}(\boldsymbol{\theta}(t))(1-\frac{V c}{2})\leq\cdots
\\\leq&
\mathcal{L}(\boldsymbol{\theta}(1))\Big(1-\frac{V c}{2}\Big)^t.
\end{align*}

Now we have proved the theorem for specific numbers $g_a=\frac{1}{2}$, $g_b=1$ and $h=1$ in Assumption \ref{def: exponential type loss}, such as the logistic loss $\ell(\boldsymbol{y}_1,\boldsymbol{y}_2)=\log(1+e^{-\boldsymbol{y}_1^\top\boldsymbol{y}_2})$ (Notice that $\tilde{\ell}(0)=\log 2$.)

In the same way, our result also holds for any other $g_a$, $g_b$ and $h$ in Assumption \ref{def: exponential type loss}. So we complete our proof.

\end{proof}

\newpage

\section{Early Stage Convergence for PRM problem}
\label{appendix: results PRM}

\subsection{Problem Setting and Result}

In this section, we study the similar fast training phenomenon in the early stage for population risk minimization (PRM) problem. 
Our results for ERM problem (Theorem \ref{thm: binary quadratic} and \ref{thm: one-hot}) could not be applied effectively to PRM problem, because  $m=\Omega(\log(n/\delta))$ becomes infinite when $n$ goes to infinity. However, we know that the loss landscape of PRM problem is often better than ERM problem under some distribution, so we may still be able to achieve similar results by loss landscape analysis.

We consider the following PRM problem with the teacher-student model, which has been studied in quite a few recent works \citep{li2017convergence, zhong2017recovery, tian2017analytical, safran2018spurious, safran2021effects}.
\begin{equation}\label{probelm: PRM}
    \min\limits_{\boldsymbol{\theta}}:\mathcal{L}(\boldsymbol{\theta})=\mathbb{E}_{\boldsymbol{x}\sim\mathcal{N}(\mathbf{0},\mathbf{I}_d)}\Bigg[\frac{1}{2}\Big(\sum_{k=1}^m \sigma(\boldsymbol{w}_k^\top\boldsymbol{x})-\sum_{k=1}^M\sigma(\boldsymbol{v}_k^\top\boldsymbol{x})\Big)^2\Bigg],
\end{equation}
where $\sigma(z)=\text{ReLU}(z)$, $\boldsymbol{\theta}=(\boldsymbol{w}_1^\top,\cdots,\boldsymbol{w}_m^\top)^\top\in\mathbb{R}^{m d}$, the first network is the student model (with width $m$), and the second network is the teacher model (with width $M$).
Without loss of generality, we assume $\left\|\boldsymbol{v}_k\right\|={1}/{M}$. If $M\geq d$, we assume $\boldsymbol{v}_i=\boldsymbol{e}_i/M$ for $i\in[d]$; If $M\leq d$, we assume $\boldsymbol{v}_i=\boldsymbol{e}_i/M$ for $i\in[M]$. These assumptions are similar to \cite{safran2018spurious, safran2021effects}.

We use Gradient Descent (GD) starting from random initialization to solve the problem (\ref{probelm: PRM}). As space is limited, we only state our result for the slightly difficult case $M\geq d$ here, and the other case $M\leq d$ is the same.

\begin{theorem}\label{thm: PRM M>d}
Assume $M\geq d\geq2$, let $\boldsymbol{\theta}(t)$ be trained by Gradient Descent with random initialization $\boldsymbol{w}_k(0)\overset{\text{i.i.d.}}{\sim}\mathbb{U}(\frac{d\kappa}{mM}\sqrt{\frac{d-1}{d}}\mathbb{S}^{d-1})$, where $\kappa\leq1$ controls the initialization scale. If the learning rate $\eta=\mathcal{O}(\frac{\kappa d}{m^2M})$, then
in the first $\Theta(\frac{m}{\kappa})$ 
iterations,
loss will descend $\Omega(\frac{(1-\kappa)^3d^3}{M^3}+\frac{\kappa(1-\kappa)d^2}{M^2})$.
\end{theorem}

\subsection{Proof sketch of Theorem \ref{thm: PRM M>d}}

In Theorem \ref{thm: PRM M>d}, we study the population risk minimization problem under the teacher-student model. The proof technique of Theorem \ref{thm: binary quadratic} and \ref{thm: one-hot} does not apply to PRM problem due to $m=\Omega(\log(n/\delta))=\infty$, so we need some different technique. 

For simplicity, we use $f(\boldsymbol{x};\boldsymbol{\theta})$ and $f^*(\boldsymbol{x})$ to denote the student model and the teacher model respectively. $\ell({\boldsymbol{x}};\boldsymbol{\theta})=\frac{1}{2}(f(\boldsymbol{x};\boldsymbol{\theta})-f^*(\boldsymbol{x}))^2$. The main difficulty to derive loss descent is still to construct non-trivial gradient lower bound in the early stage. 
Our main method is to construct an efficient descent direction $\mathbf{\Phi}(t)$ s.t. $
\left<\mathbf{\Phi}(t),\nabla_{\boldsymbol{
\theta}} \mathcal{L}(\boldsymbol{\theta}(t))\right>\geq{B}(t)>0$, where $\frac{{B}(t)}{\left\|\mathbf{\Phi}(t)\right\|}$ should have effective lower bound. Inspired by the homogeneity of ReLU activation $\left<\boldsymbol{\theta},\nabla_{\boldsymbol{\theta}} f(\boldsymbol{x};\boldsymbol{\theta})\right>=f(\boldsymbol{x};\boldsymbol{\theta})$, in the early stage, we have the approximate that $\left<-\boldsymbol{\theta},\nabla_{\boldsymbol{\theta}} \ell(\boldsymbol{x};\boldsymbol{\theta})\right>=f(\boldsymbol{x};\boldsymbol{\theta})(f^*(\boldsymbol{x})-f(\boldsymbol{x};\boldsymbol{\theta}))\approx f^*(\boldsymbol{x})f(\boldsymbol{x};\boldsymbol{\theta})$. So we conjecture that $\mathbf{\Phi}(t):=-\boldsymbol{\theta}(t)$ has this property naturally.

\begin{lemma}[Informal Lemma \ref{lemma: gradient lower bound}] For any $\boldsymbol{\theta}\ne0$, s.t. $\sum\limits_{j=1}^m\left\|\boldsymbol{w}_j\right\|<\frac{d}{\pi M}\sqrt{\frac{d-1}{d}}$, we have
\[
\left<-\frac{\boldsymbol{w}_k}{\left\|\boldsymbol{w}_k\right\|},\frac{\partial \mathcal{L}(\boldsymbol{\theta})}{\partial\boldsymbol{w}_k}\right>
\geq
\Big(\frac{d}{2\pi M}\sqrt{\frac{d-1}{d}}-\frac{1}{2}(\sum_{j=1}^m\left\|\boldsymbol{w}_j\right\|)\Big),\ \forall k\in[m].
\]
\end{lemma}
Other techniques such as hitting time estimate, parameter estimate and Hessian upper bound analysis can be achieved in the similar way, we do not repeat them here. Please refer to Appendix \ref{appendix: proof PRM}.

\newpage

\section{Proof of Theorem \ref{thm: PRM M>d}}\label{appendix: proof PRM}
\subsection{Preparation}
\textbf{Learning Rate.} We use the following learning rate:
\begin{equation}\label{equ: learning rate PRM}
\begin{aligned}
\eta\leq&\min\Bigg\{\frac{2\pi d \kappa\sqrt{\frac{d-1}{d}}}{(\pi+1)m M\Big(1+\frac{d}{\pi M}\sqrt{\frac{d-1}{d}}\Big)},\frac{1}{\frac{1}{2}+m(m-1)\Big(\frac{\kappa+(\frac{1}{\pi}-\kappa)\sqrt{\frac{d-1}{d}}}{2\pi\kappa}+\frac{1}{2}\Big)+\frac{m^2M}{2\pi d\kappa}}
\Bigg\}
\\=&\mathcal{O}\Big(\frac{d\kappa}{m^2 M}\Big).
\end{aligned}
\end{equation}

\textbf{Hitting time.} We define the hitting time as
\begin{equation}\label{equ: hitting time PRM}
T:=\sup\big\{t\geq0:\sum\limits_{j=1}^m\left\|\boldsymbol{w}_j(t+1)\right\|<\frac{d}{\pi M}\sqrt{\frac{d-1}{d}}\big\}
\end{equation}

\subsection{Gradient lower bound}

\begin{lemma}[Connection between loss and gradient]
\label{lemma: loss homogeneity}
For any $\boldsymbol{\theta}$ and $i\in[m]$, we have:
\begin{gather*}
    \left<\boldsymbol{w}_i,\frac{\partial \mathcal{L}(\boldsymbol{\theta})}{\partial \boldsymbol{w}_i }\right>=
    \sum_{j=1}^m k(\boldsymbol{w}_i;\boldsymbol{w}_j)-\sum_{j=1}^M k(\boldsymbol{w}_i;\boldsymbol{v}_j);
\\
\left<\boldsymbol{\theta},\nabla \mathcal{L}(\boldsymbol{\theta})\right>=
\sum_{i=1}^m\sum_{j=1}^m k(\boldsymbol{w}_i;\boldsymbol{w}_j)-\sum_{i=1}^m\sum_{j=1}^M k(\boldsymbol{w}_i;\boldsymbol{v}_j),
\end{gather*}
where $k(\boldsymbol{w};\boldsymbol{v})=\frac{1}{2\pi}\left\|\boldsymbol{w}\right\|\left\|\boldsymbol{v}\right\|\Big(\sin(\theta_{\boldsymbol{w},\boldsymbol{v}})+(\pi-\theta_{\boldsymbol{w},\boldsymbol{v}})\cos(\theta_{\boldsymbol{w},\boldsymbol{v}})\Big)$.
\end{lemma}

\begin{proof}[Proof of Lemma \ref{lemma: loss homogeneity}]\ \\
From some prior work \citep{cho2009kernel, safran2018spurious, safran2021effects},
$\mathcal{L}(\boldsymbol{\theta})$ has the closed form:
\[
\mathcal{L}(\boldsymbol{\theta})=\frac{1}{2}\sum_{i,j=1}^m k(\boldsymbol{w}_i;\boldsymbol{w}_j) -\sum_{i=1}^m\sum_{j=1}^{M}k(\boldsymbol{w}_i;\boldsymbol{v}_j)+\frac{1}{2}\sum_{i,j=1}^M k(\boldsymbol{v}_i;\boldsymbol{v}_j),
\]
where
\[
\begin{aligned}
k(\boldsymbol{w};\boldsymbol{v})
&=
\mathbb{E}_{\boldsymbol{x}\sim\mathcal{N}(\boldsymbol{0},\boldsymbol{I}_d)}\Big[ \sigma(\boldsymbol{w}^\top \boldsymbol{x})\sigma(\boldsymbol{v}^\top\boldsymbol{x})\Big]
\\&=\frac{1}{2\pi }\left\|\boldsymbol{w}\right\|\left\|\boldsymbol{v}\right\|\Big(\sin(\theta_{\boldsymbol{w},\boldsymbol{v}})+(\pi-\theta_{\boldsymbol{w},\boldsymbol{v}})\cos(\theta_{\boldsymbol{w},\boldsymbol{v}})\Big)
\end{aligned}
\]
and $\theta_{\boldsymbol{w},\boldsymbol{v}}=\arccos\Big(\frac{\left<\boldsymbol{w},\boldsymbol{v}\right>}{\left\|\boldsymbol{w}\right\|\left\|\boldsymbol{v}\right\|}\Big)$. The similar result has shown in \cite{cho2009kernel}. And we can calculate the gradient of $f$:
\[
\begin{gathered}
 \frac{\partial k(\boldsymbol{w};\boldsymbol{v})}{\partial \boldsymbol{w}}=\frac{1}{2\pi} \left\|\boldsymbol{v}\right\|\Big(\sin(\theta_{\boldsymbol{w},\boldsymbol{v}})\Bar{\boldsymbol{w}}+(\pi-\theta_{\boldsymbol{w},\boldsymbol{v}})\Bar{\boldsymbol{v}}\Big),\ \boldsymbol{v\ne w};
 \\
 \frac{\partial k(\boldsymbol{w};\boldsymbol{w})}{\partial \boldsymbol{w}}=\boldsymbol{w}.
 \end{gathered}
\]
Then it is easy to verify that:

\begin{align*}
\left<\boldsymbol{w},\frac{\partial k(\boldsymbol{w};\boldsymbol{v})}{\partial \boldsymbol{w}}\right>
&=\frac{1}{2\pi } \left\|\boldsymbol{v}\right\|\Big(\sin(\theta_{\boldsymbol{w},\boldsymbol{v}})\left<\boldsymbol{w},\Bar{\boldsymbol{w}}\right>+(\pi-\theta_{\boldsymbol{w},\boldsymbol{v}})\left<\boldsymbol{w},\Bar{\boldsymbol{v}}\right>\Big)
\\&=\frac{1}{2\pi}\left\|\boldsymbol{w}\right\|\left\|\boldsymbol{v}\right\|\Big(\sin(\theta_{\boldsymbol{w},\boldsymbol{v}})+(\pi-\theta_{\boldsymbol{w},\boldsymbol{v}})\cos(\theta_{\boldsymbol{w},\boldsymbol{v}})\Big)
\\&=k(\boldsymbol{w};\boldsymbol{v}),\ \boldsymbol{v}\ne\boldsymbol{w};
\end{align*}
\begin{align*}
    \left<\boldsymbol{w},\frac{\partial k(\boldsymbol{w};\boldsymbol{w})}{\partial \boldsymbol{w}}\right>=\left\|\boldsymbol{w}\right\|^2=2k(\boldsymbol{w};\boldsymbol{w}).
\end{align*}
So for any $i\in[m]$, we have:
\begin{align*}
    \left<\boldsymbol{w}_i,\frac{\partial \mathcal{L}(\boldsymbol{\theta})}{\partial \boldsymbol{w}_i }\right>
    &=\frac{1}{2}\left<\boldsymbol{w}_i,\frac{\partial k(\boldsymbol{w}_i,\boldsymbol{w}_i)}{\partial \boldsymbol{w}_i }\right>
    +\sum_{\substack{j=1 \\ j\ne i}}^m
    \left<\boldsymbol{w}_i,\frac{\partial k(\boldsymbol{w}_i,\boldsymbol{w}_j)}{\partial \boldsymbol{w}_i }\right>-\sum_{j=1}^M
    \left<\boldsymbol{w}_i,\frac{\partial k(\boldsymbol{w}_i,\boldsymbol{v}_j)}{\partial \boldsymbol{w}_i }\right>
    \\&=
    \sum_{j=1}^m k(\boldsymbol{w}_i;\boldsymbol{w}_j)-\sum_{j=1}^M k(\boldsymbol{w}_i;\boldsymbol{v}_j);
\end{align*}
\begin{align*}
\left<\boldsymbol{\theta},\nabla \mathcal{L}(\boldsymbol{\theta})\right>
&=\sum_{i=1}^m \sum_{j=1 }^m k(\boldsymbol{w}_i;\boldsymbol{w}_j)-\sum_{i=1}^m\sum_{j=1}^M k(\boldsymbol{w}_i;\boldsymbol{v}_j).
\end{align*}

\end{proof}

\begin{lemma}\label{lemma: function analysis A}
Consider the function:
\[
f(\theta)=\sin\theta+(\pi-\theta)\cos\theta,\ \theta\in[0,\pi].
\]
Then $f(\theta)$ is decreasing on $[0,\pi]$, $0\leq f(\theta)\leq\pi$, and $f(\theta)$ is convex on $[\frac{\pi}{2},\pi]$.
\end{lemma}
\begin{proof}[Proof of Lemma \ref{lemma: function analysis A}]\ \\
\begin{gather*}
    f'(\theta) = (\theta-\pi)\sin\theta,
    \\
    f''(\theta)=(\theta-\pi)\cos\theta+\sin\theta.
\end{gather*}
It is obviously that $f'(\theta)<0$ for $\theta\in(0,\pi)$, so $f$ is decreasing and $0\leq f(\theta)\leq\pi$. 

From $\tan\theta\leq0,\cos\theta\leq0$ for $\theta\in(\frac{\pi}{2},\pi]$, we\ have:
\[
\begin{gathered}
f''(\theta)=\cos\theta\Big(\theta-\pi+\tan\theta\Big)\geq(\pi-\theta)\cos\theta\geq0,\ \theta\in(\frac{\pi}{2},\pi],
\end{gathered}
\]
so $f(\theta)$ is convex for $\theta\in[\frac{\pi}{2},\pi]$.
\end{proof}

\begin{lemma}[Lower bound of cross terms]\label{lemma: cross term lower bound}For any $\boldsymbol{w}$ we have:
\[
\sum_{j=1}^M k(\boldsymbol{w}_i;\boldsymbol{v}_j)
\geq
\frac{d}{2\pi M}\sqrt{\frac{d-1}{d}}\left\|\boldsymbol{w}\right\|.
\]
where $k(\cdot;\cdot)$ is defined in Lemma \ref{lemma: loss homogeneity}.
\end{lemma}

\begin{proof}[Proof of Lemma \ref{lemma: cross term lower bound}]\ \\
For convenience, we define the function:
\[
\phi(\theta)=\sin\theta+(\pi-\theta)\cos(\theta).
\]
Then we have:
\begin{align*}
\sum_{j=1}^M k(
\boldsymbol{w};\boldsymbol{v}_j)
=&\sum_{k=1}^M\frac{1}{2\pi}\left\|\boldsymbol{w}\right\|\left\|\boldsymbol{v}_i\right\|\Big(\sin(\theta_{\boldsymbol{w},\boldsymbol{v}_i})+(\pi-\theta_{\boldsymbol{w},\boldsymbol{v}_i})\cos(\theta_{\boldsymbol{w},\boldsymbol{v}_i})\Big)
\\=&
\frac{\left\|\boldsymbol{w}\right\|}{2\pi M}\sum_{i=1}^M\phi(\theta_{\boldsymbol{w},\boldsymbol{v}_i})
\\\overset{\rm Lemma\ \ref{lemma: function analysis A}}{\geq}&
\frac{\left\|\boldsymbol{w}\right\|}{2\pi M}\sum_{i=1}^d\phi(\theta_{\boldsymbol{w},\boldsymbol{e}_i}).
\end{align*}
We denote $\boldsymbol{w}=(w_1\cdots,w_d)$, then we consider $\boldsymbol{w}^*=(-|w_1|,\cdots,-|w_d|)$. From Lemma \ref{lemma: function analysis A} we know:
\[
\phi(\theta_{\boldsymbol{w},\boldsymbol{e}_i})
\geq
\phi(\theta_{\boldsymbol{w}^*,\boldsymbol{e}_i}).\]
So we have the upper bound estimate:
\begin{align*}
   \sum_{j=1}^M k(
\boldsymbol{w};\boldsymbol{v}_j)
\geq&
\frac{\left\|\boldsymbol{w}\right\|}{2\pi M}\sum_{i=1}^d\phi(\theta_{\boldsymbol{w},\boldsymbol{e}_i})
\geq \frac{\left\|\boldsymbol{w}\right\|}{2\pi M}\sum_{i=1}^d\phi(\theta_{\boldsymbol{w}^*,\boldsymbol{e}_i})
\\\overset{\rm Lemma\ \ref{lemma: jensen inequation},\ref{lemma: function analysis A}}{\geq}&
\frac{\left\|\boldsymbol{w}\right\|d}{2\pi M}\phi\Big(\frac{\left\|\boldsymbol{w}\right\|}{d}\sum_{i=1}^d \theta_{\boldsymbol{w}^*,\boldsymbol{e}_i}\Big)
\\=&
\frac{\left\|\boldsymbol{w}\right\|d}{2\pi M}\phi\Big(\frac{1}{d}\sum_{i=1}^d \arccos\Big(\frac{-|w_i|}{\left\|\boldsymbol{w}\right\|}\Big)\Big)
\\\overset{\rm Lemma\ \ref{lemma: jensen inequation},\ref{lemma: function analysis B}}{\geq}&
\frac{\left\|\boldsymbol{w}\right\|d}{2\pi M}\phi\circ \arccos\Big(-\frac{1}{d}\sum_{i=1}^d\frac{|w_i|}{\left\|\boldsymbol{w}\right\|}\Big)
\\\geq&
\frac{\left\|\boldsymbol{w}\right\|d}{2\pi M}\phi\circ \arccos\Big(-\frac{\sqrt{d}\left\|\boldsymbol{w}\right\|}{d\left\|\boldsymbol{w}\right\|}\Big)
\\=&
\frac{\left\|\boldsymbol{w}\right\|d}{2\pi M}\Big(\frac{\sqrt{d-1}}{\sqrt{d}}-(\arccos(-\frac{1}{\sqrt{d}})-\pi)\frac{1}{\sqrt{d}}\Big)
\\\geq&
\frac{\left\|\boldsymbol{w}\right\|d}{2\pi M}\sqrt{\frac{d-1}{d}}
.
\end{align*}
\end{proof}

\begin{lemma}[Gradient Lower Bound]\label{lemma: gradient lower bound} For any $\boldsymbol{\theta}\ne0$, if the condition $\sum\limits_{j=1}^m\left\|\boldsymbol{w}_j\right\|<\frac{d}{\pi M}\sqrt{\frac{d-1}{d}}$ holds, we have the lower bound of gradient:
\begin{gather*}
\left<-\boldsymbol{w}_k,\frac{\partial \mathcal{L}(\boldsymbol{\theta})}{\partial\boldsymbol{w}_k}\right>
\geq
\Bigg(\frac{d}{2\pi M}\sqrt{\frac{d-1}{d}}-\frac{1}{2}\Big(\sum_{j=1}^m\left\|\boldsymbol{w}_j\right\|\Big)\Bigg)\left\|\boldsymbol{w}_k\right\|,\ \forall k\in[m];
\\
    \left\|\frac{\partial \mathcal{L}(\boldsymbol{\theta})}{\partial\boldsymbol{w}_k}\right\|\geq\frac{d}{2\pi M}\sqrt{\frac{d-1}{d}}-\frac{1}{2}\sum_{j=1}^m\left\|\boldsymbol{w}_j\right\|,\ \forall k\in[m];
    \\
    \left\|\nabla \mathcal{L}(\boldsymbol{\theta})\right\|\geq
    \sqrt{m}\Big(\frac{d}{2\pi M}\sqrt{\frac{d-1}{d}}-\frac{1}{2}\sum_{j=1}^m\left\|\boldsymbol{w}_j\right\|\Big).
\end{gather*}
\end{lemma}

\begin{proof}[Proof of Lemma \ref{lemma: gradient lower bound}]\ \\
From Lemma \ref{lemma: loss homogeneity}, we have:
\[
\left<\boldsymbol{w}_i,\frac{\partial \mathcal{L}(\boldsymbol{\theta})}{\partial \boldsymbol{w}_i }\right>=
\sum_{j=1}^m k(\boldsymbol{w}_i;\boldsymbol{w}_j)-\sum_{j=1}^M k(\boldsymbol{w}_i;\boldsymbol{v}_j).
\]
From Lemma \ref{lemma: cross term lower bound}, we know:
\[
\sum_{j=1}^M k(\boldsymbol{w}_i;\boldsymbol{v}_j)\geq
\frac{d}{2\pi M}\sqrt{\frac{d-1}{d}}\left\|\boldsymbol{w}_k\right\|.
\]
And we can estimate that:
\begin{align*}
    &\sum_{j=1}^m k(\boldsymbol{w}_i;\boldsymbol{w}_j)
    \\\leq&
    \frac{1}{2}\left\|\boldsymbol{w}_i\right\|^2+\sum_{\substack{j=1\\ j\ne i}}^m\frac{1}{2\pi }\left\|\boldsymbol{w}_k\right\|\left\|\boldsymbol{w}_j\right\|\pi
    \\\leq&
    \frac{1}{2}\Big(\sum_{j=1}^m\left\|\boldsymbol{w}_j\right\|\Big)\left\|\boldsymbol{w}_i\right\|.
\end{align*}
So under the condition $\Big(\sum\limits_{j=1}^m\left\|\boldsymbol{w}_j\right\|\Big)<\frac{d}{\pi M}\sqrt{\frac{d-1}{d}}$, we have:
\begin{align*}
    &-\left<\boldsymbol{w}_i,\frac{\partial \mathcal{L}(\boldsymbol{\theta})}{\partial \boldsymbol{w}_i }\right>
    \\\geq&
    \sum_{j=1}^M k(\boldsymbol{w}_i;\boldsymbol{v}_j)-\Big(
    2k(\boldsymbol{w}_i;\boldsymbol{w}_i)+\sum_{\substack{j=1\\ j\ne i}}^m k(\boldsymbol{w}_i;\boldsymbol{w}_j)\Big)
    \\\geq&
    \frac{d}{2\pi M}\sqrt{\frac{d-1}{d}}\left\|\boldsymbol{w}_i\right\|-\frac{1}{2}\Big(\sum_{j=1}^m\left\|\boldsymbol{w}_j\right\|\Big)\left\|\boldsymbol{w}_i\right\|
    \\>&0.
\end{align*}
So we have:
\begin{align*}
    &\left\|\frac{\partial \mathcal{L}(\boldsymbol{\theta})}{\partial \boldsymbol{w}_i }\right\|
    \geq\frac{1}{\left\|\boldsymbol{w}_i\right\|}\left|-\left<\boldsymbol{w}_i,\frac{\partial \mathcal{L}(\boldsymbol{\theta})}{\partial \boldsymbol{w}_i }\right>\right|
    \\\geq&
    \frac{d}{2\pi M}\sqrt{\frac{d-1}{d}}-\frac{1}{2}\Big(\sum_{j=1}^m\left\|\boldsymbol{w}_j\right\|\Big)>0,
\end{align*}
and
\begin{align*}
    &\left\|\nabla \mathcal{L}(\boldsymbol{\theta})\right\|=\Bigg(\sum_{i=1}^m\left\|\frac{\partial \mathcal{L}(\boldsymbol{\theta})}{\partial \boldsymbol{w}_i }\right\|^2\Bigg)^{\frac{1}{2}}
    \\\geq&
    \sqrt{m}\Big(\frac{d}{2\pi M}\sqrt{\frac{d-1}{d}}-\frac{1}{2}\sum_{j=1}^m\left\|\boldsymbol{w}_j\right\|\Big)>0.
\end{align*}

\end{proof}

\subsection{Hitting time estimate}

\begin{lemma}[Gradient Upper Bound]\label{lemma: gradient estimate}

\begin{gather*}
     \left\|\frac{\partial \mathcal{L}(\boldsymbol{\theta})}{\partial \boldsymbol{w}_k}\right\|\leq
    \frac{\pi+1}{2\pi}\Big(1+\sum_{j=1}^m\left\|\boldsymbol{w}_j\right\|\Big),
    \\
 \left\|\nabla \mathcal{L}(\boldsymbol{\theta})\right\|\leq
\frac{(\pi+1)\sqrt{m}}{2\pi}\Big(1+\sum_{j=1}^m\left\|\boldsymbol{w}_j\right\|\Big)
.
\end{gather*}

\end{lemma}

\begin{proof}[Proof of Lemma \ref{lemma: gradient estimate}]\ \\
From the closed form of gradient in the proof of Lemma \ref{lemma: loss homogeneity}:
\[
\frac{\partial \mathcal{L}(\boldsymbol{\theta})}{\partial \boldsymbol{w}_k}=\frac{1}{2}\frac{\partial k(\boldsymbol{w}_k;\boldsymbol{w}_k)}{\partial \boldsymbol{w}_k}+\sum_{\substack{j=1\\ j\ne k}}^m \frac{\partial k(\boldsymbol{w}_k;\boldsymbol{w}_j)}{\partial \boldsymbol{w}_k}-\sum_{j=1}^M \frac{\partial k(\boldsymbol{w}_k;\boldsymbol{v}_j)}{\partial \boldsymbol{w}_k},
\]
we can estimate the three parts respectively:
\begin{align*}
    \left\|\frac{1}{2}\frac{\partial k(\boldsymbol{w}_k;\boldsymbol{w}_k)}{\partial \boldsymbol{w}_k}\right\|
    \leq
    \frac{1}{2}\left\|\boldsymbol{w}_k\right\|;
\end{align*}
\begin{align*}
    &\left\|\sum_{\substack{j=1\\ j\ne k}}^m \frac{\partial k(\boldsymbol{w}_k;\boldsymbol{w}_j)}{\partial \boldsymbol{w}_k}\right\|
    \leq
    \sum_{\substack{j=1\\ j\ne k}}^m \left\|\frac{\partial k(\boldsymbol{w}_k;\boldsymbol{w}_j)}{\partial \boldsymbol{w}_k}\right\|
    \\\leq&
    \sum_{\substack{j=1\\j\ne k}}^m\frac{1}{2\pi }\left\|\boldsymbol{w}_j\right\|\Bigg(\left\|\sin(\theta_{\boldsymbol{w}_k,\boldsymbol{w}_j})\Bar{\boldsymbol{w}}_k\right\|
    +\left\|(\pi-\theta_{\boldsymbol{w}_k,\boldsymbol{w}_j})\Bar{\boldsymbol{w}}_j\right\|
    \Bigg)
    \\\leq&
    \frac{\pi +1}{2\pi }\Big(\sum_{\substack{j=1\\ j\ne k}}^m\left\|\boldsymbol{w}_j\right\|\Big);
\end{align*}
\begin{align*}
    &\left\|\sum_{j=1}^M \frac{\partial k(\boldsymbol{w}_k;\boldsymbol{v}_j)}{\partial \boldsymbol{w}_k}\right\|\leq
    \sum_{j=1}^M \left\|\frac{\partial k(\boldsymbol{w}_k;\boldsymbol{v}_j)}{\partial \boldsymbol{w}_k}\right\|
    \\\leq&
    \sum_{j=1}^ M\frac{1}{2\pi M}\Bigg(\left\|\sin(\theta_{\boldsymbol{w}_k,\boldsymbol{v}_j})\Bar{\boldsymbol{w}}_k\right\|
    +\left\|(\pi-\theta_{\boldsymbol{w}_k,\boldsymbol{v}_j})\Bar{\boldsymbol{v}}_j\right\|\Bigg)
    \\\leq&
    \frac{\pi+1}{2\pi}.
\end{align*}

Then we have the upper bound:
\begin{align*}
    \left\|\frac{\partial \mathcal{L}(\boldsymbol{\theta})}{\partial \boldsymbol{w}_k}\right\|\leq&
    \frac{1}{2}\left\|\boldsymbol{w}_k\right\|+
    \frac{\pi +1}{2\pi }\Big(\sum_{\substack{j=1\\ j\ne k}}^m\left\|\boldsymbol{w}_j\right\|\Big)+\frac{\pi+1}{2\pi }
    \\\leq&
    \frac{\pi+1}{2\pi }\Big(1+\sum_{j=1}^m\left\|\boldsymbol{w}_j\right\|\Big).
\end{align*}
\[
 \left\|\nabla \mathcal{L}(\boldsymbol{\theta})\right\|\leq
\frac{(\pi+1)\sqrt{m}}{2\pi}\Big(1+\sum_{j=1}^m\left\|\boldsymbol{w}_j\right\|\Big).
\]

\end{proof}

\begin{lemma}[Estimate of hitting time]\label{lemma: hitting time estimate PRM}
\[
T\geq T^*=\Theta\Big(\frac{d(1-\pi\kappa)}{\eta m M}\Big)=\Theta\Big(\frac{m(1-\pi\kappa)}{\kappa}\Big).
\]
\end{lemma}
\begin{proof}[Proof of Lemma \ref{lemma: hitting time estimate PRM}]\ \\
From the estimation
\begin{align*}
    \left\|\boldsymbol{w}_k(t+1)\right\|
    \leq&
    \left\|\boldsymbol{w}_k(0)\right\|+\eta\sum_{s=0}^t\left\|\frac{\partial \mathcal{L}(\boldsymbol{\theta}(s))}{\partial\boldsymbol{w}_k}\right\|
    \\\overset{\rm Lemma\ \ref{lemma: gradient estimate}}{\leq}&
    \left\|\boldsymbol{w}_k(0)\right\|+\eta\sum_{s=0}^t\frac{\pi+1}{2\pi}\Big(1+\sum_{j=1}^m\left\|\boldsymbol{w}_j(s)\right\|\Big)
    \\\leq&
     \left\|\boldsymbol{w}_k(0)\right\|+\eta(t+1)\frac{\pi+1}{2\pi}\Big(1+\sum_{j=1}^m\left\|\boldsymbol{w}_j(s)\right\|\Big),
\end{align*}
we have:
\begin{align*}
\sum_{k=1}^m\left\|\boldsymbol{w}_k(t+1)\right\|\leq&
 \sum_{k=1}^m\left\|\boldsymbol{w}_k(0)\right\|+\eta(t+1)\frac{(\pi+1)m}{2\pi}\Big(1+\sum_{j=1}^m\left\|\boldsymbol{w}_j(s)\right\|\Big)    
 \\\leq&
  \sum_{k=1}^m\left\|\boldsymbol{w}_k(0)\right\|+\eta(t+1)\frac{(\pi+1)m}{2\pi}\Big(1+\frac{d}{\pi M}\sqrt{\frac{d-1}{d}}\Big)
  \\\leq&
\frac{d\kappa}{ M}\sqrt{\frac{d-1}{d}}+\eta(t+1)\frac{(\pi+1)m}{2\pi }\Big(1+\frac{d}{\pi M}\sqrt{\frac{d-1}{d}}\Big).   
\end{align*}
So if we define $T^*$ as:
\[T^*=\sup\Big\{t\geq0:\frac{d\kappa}{ M}\sqrt{\frac{d-1}{d}}+\eta(t+1)\frac{(\pi+1)m}{2\pi}\Big(1+\frac{d}{\pi M}\sqrt{\frac{d-1}{d}}\Big)\leq\frac{d}{
\pi M}\sqrt{\frac{d-1}{d}}\Big\},\]
From the definition of $T$ (\ref{equ: gamma definition})
we can get $T\geq T^*$ and
\begin{align*}
    T^*+1=&\frac{\frac{d}{\pi M}\sqrt{\frac{d-1}{d}}(1-\pi\kappa)}{\frac{\eta(\pi+1)m}{2\pi }\Big(1+\frac{d}{\pi M}\sqrt{\frac{d-1}{d}}\Big)}
    \\=&
    \frac{2d\sqrt{\frac{d-1}{d}}(1-\pi\kappa)}{(\pi +1)\eta m M\Big(1+\frac{d}{
    \pi M}\sqrt{\frac{d-1}{d}}\Big)}.
\end{align*}

\end{proof}

\subsection{Early stage convergence}

\begin{lemma}[Quadratic differentiability along GD trajectory]\label{lemma: quadratic diff along GD trajectory}
Let $\boldsymbol{\theta}(t)$ be trained by Gradient Descent with random initialization $\boldsymbol{w}_k(0)\sim\mathbb{U}\Big(\frac{d\kappa}{mM}\sqrt{\frac{d-1}{d}}\mathbb{S}^{d-1}\Big)$. If the learning rate is satisfied to
\[
\eta\leq\frac{2\pi d \kappa\sqrt{\frac{d-1}{d}}}{(\pi+1)m M\Big(1+\frac{d}{\pi M}\sqrt{\frac{d-1}{d}}\Big)}=\mathcal{O}\Big(\frac{d\kappa}{mM}\Big),
\]
then for any $0\leq t\leq T$, we have:
\[
\boldsymbol{\theta}(t)+\alpha\Big(\boldsymbol{\theta}(t+1)-\boldsymbol{\theta}(t)\Big)\ne\boldsymbol{0},\ \forall\alpha\in[0,1];
\]
\[\left\|\boldsymbol{w}_k(t)\right\|< \left\|\boldsymbol{w}_k(t+1)\right\|<2\left\|\boldsymbol{w}_k(t)\right\|,\ \forall k\in[m].
\]
\end{lemma}

\begin{proof}[Proof of Lemma \ref{lemma: quadratic diff along GD trajectory}]\ \\
First, we have the norm of random initialization:
\[
\left\|\boldsymbol{w}_k(0)\right\|=\frac{d\kappa}{m M}\sqrt{\frac{d-1}{d}},k\in[m].
\]
From Lemma \ref{lemma: gradient lower bound}, we have:
\begin{align*}
    \left\|\boldsymbol{w}_k(t+1)\right\|^2=&
    \left\|\boldsymbol{w}_k(t)+\boldsymbol{w}_k(t+1)-\boldsymbol{w}_k(t)\right\|^2
    \\=&
    \left\|\boldsymbol{w}_k(t)\right\|^2-2\eta\left<\boldsymbol{w}_k(t),\frac{\partial \mathcal{L}(\boldsymbol{\theta}(t))}{\boldsymbol{w}_k(t)}\right>+\eta^2\left\|\frac{\partial \mathcal{L}(\boldsymbol{\theta}(t))}{\boldsymbol{w}_k(t)}\right\|^2
    \\>&
    \left\|\boldsymbol{w}_k(t)\right\|^2+2\eta\Big(\frac{d}{2\pi M}\sqrt{\frac{d-1}{d}}-\frac{1}{2}\Big(\sum_{j=1}^m\left\|\boldsymbol{w}_j\right\|\Big)\Big)
    \\>& \left\|\boldsymbol{w}_k(t)\right\|^2.
\end{align*}
From Lemma \ref{lemma: gradient estimate}, we have 
\begin{align*}
    \left\|\boldsymbol{w}_k(t+1)\right\|=&
    \left\|\boldsymbol{w}_k(t)-\eta\frac{\partial \mathcal{L}(\boldsymbol{\theta}(t))}{\boldsymbol{w}_k(t)}\right\|
    \\\leq&
    \left\|\boldsymbol{w}_k(t)\right\|+\eta\left\|\frac{\partial \mathcal{L}(\boldsymbol{\theta}(t))}{\boldsymbol{w}_k(t)}\right\|
    \\\leq&
    \left\|\boldsymbol{w}_k(t)\right\|+\eta\frac{\pi+1}{2\pi}\Big(1+\sum_{j=1}^m\left\|\boldsymbol{w}_j(t)\right\|\Big)
    \\<&
    \left\|\boldsymbol{w}_k(t)\right\|+\eta\frac{\pi+1}{2\pi}\Big(1+\frac{d}{\pi M}\sqrt{\frac{d-1}{d}}\Big)
    \\\leq&
     \left\|\boldsymbol{w}_k(t)\right\|+\left\|\boldsymbol{w}_k(0)\right\|
     \\<& 2\left\|\boldsymbol{w}_k(t)\right\|.
\end{align*}
Combining the two estimates above, we have:
\[
\boldsymbol{w}(t)+\alpha\Big(\boldsymbol{w}(t+1)-\boldsymbol{w}(t)\Big)\ne\boldsymbol{0},\ \forall\alpha\in[0,1],\]
Then we have:
\[
\boldsymbol{\theta}(t)+\alpha\Big(\boldsymbol{\theta}(t+1)-\boldsymbol{\theta}(t)\Big)\ne\boldsymbol{0},\ \forall\alpha\in[0,1].\]

\end{proof}

\begin{lemma}[Hessian upper bound]\label{lemma: hessian upper bound PRM}
$T^*$ is defined in Lemma \ref{lemma: hitting time estimate PRM}. Use the learning rate $\eta$ in Lemma \ref{lemma: quadratic diff along GD trajectory}.
Consider all parameters along the trajectory of GD
\[
\mathcal{S}(T^*):=\bigcup\limits_{t\leq T^*-1}\Big\{\boldsymbol{\theta}:\boldsymbol{\theta}\in\overline{\boldsymbol{\theta}(t)\boldsymbol{\theta}({t+1})}\Big\},
\]
then for any $\boldsymbol{\theta}\in\mathcal{S}(T^*)$, we have:
\[
\left\|\nabla^2 \mathcal{L}(\boldsymbol{\theta})\right\|
\leq \frac{1}{2}+m(m-1)\Big(\frac{\kappa+(\frac{1}{\pi}-\kappa)\sqrt{\frac{d-1}{d}}}{2\pi\kappa}+\frac{1}{2}\Big)+\frac{m^2M}{2\pi d\kappa}.
\]
\end{lemma}
\begin{proof}[Proof of Lemma \ref{lemma: hessian upper bound PRM}]\ \\
With the help of Lemma 7 in \cite{safran2018spurious}, we know that for any $\boldsymbol{\theta}\ne\boldsymbol{0}$,
\[
\left\|\nabla^2 \mathcal{L}(\boldsymbol{\theta})\right\|\leq
\frac{1}{2}+m(m-1)\Big(\frac{\left\|\boldsymbol{w}_{\rm max}\right\|}{2\pi\left\|\boldsymbol{w}_{\rm min}\right\|}+\frac{1}{2})+\frac{m M\left\|\boldsymbol{v}_{\rm max}\right\|}{2\pi\left\|\boldsymbol{w}_{\rm min}\right\|}.
\]

From the proof of Lemma \ref{lemma: hitting time estimate PRM}, we have the estimate of $\left\|\boldsymbol{w}_{\max}(t)\right\|$:
\begin{align*}
   \max_{0\leq t\leq T^*+1}\left\|\boldsymbol{w}_{\max}(t)\right\|\leq&\left\|\boldsymbol{w}_{\max}(0)\right\|+\eta(T^*+1)\frac{(\pi+1)}{2\pi }\Big(1+\frac{d}{\pi M}\sqrt{\frac{d-1}{d}}\Big) \\\leq&
   \frac{d\kappa}{ m M}\sqrt{\frac{d-1}{d}}+\eta(T^*+1)\frac{(\pi+1)}{2\pi }\Big(1+\frac{d}{\pi M}\sqrt{\frac{d-1}{d}}\Big)\\\leq&
   \frac{d\kappa}{ m M}\sqrt{\frac{d-1}{d}}+
   \frac{2d\sqrt{\frac{d-1}{d}}(1-\pi\kappa)}{(\pi +1)m M\Big(1+\frac{d}{
    \pi M}\sqrt{\frac{d-1}{d}}\Big)}\frac{(\pi+1)}{2\pi }\Big(1+\frac{d}{\pi M}\sqrt{\frac{d-1}{d}}\Big)
    \\\leq&
      \frac{d\kappa}{ m M}+
   \frac{d(1-\pi\kappa)}{\pi m M}\sqrt{\frac{d-1}{d}}.
\end{align*}
Then we have:
\begin{align*}
&\frac{1}{2}+m(m-1)\Big(\frac{\kappa+(\frac{1}{\pi}-\kappa)\sqrt{\frac{d-1}{d}}}{2\pi\kappa}+\frac{1}{2}\Big)+\frac{m^2M}{2\pi d\kappa}
\\\geq&\max_{0\leq t\leq T^*+1}\Bigg(
\frac{1}{2}+m(m-1)\Big(\frac{\left\|\boldsymbol{w}_{\rm max}(t)\right\|}{2\pi\frac{d\kappa}{m M}}+\frac{1}{2}\Big)+\frac{m}{2\pi\frac{d\kappa}{m M}}\Bigg)
\\\geq&\max_{0\leq t\leq T^*+1}\Bigg(
\frac{1}{2}+m(m-1)\Big(\frac{\left\|\boldsymbol{w}_{\rm max}(t)\right\|}{2\pi\left\|\boldsymbol{w}_{\rm min}(t)\right\|}+\frac{1}{2}\Big)+\frac{m M\left\|\boldsymbol{v}_{\rm max}\right\|}{2\pi\left\|\boldsymbol{w}_{\rm min}(t)\right\|}\Bigg)
\\\geq&
\max_{0\leq t\leq T^*}\sup_{\alpha\in[0,1]}
\left\|\nabla^2 \mathcal{L}(\boldsymbol{\theta}(t)+\alpha(\boldsymbol{\theta}(t+1)-\boldsymbol{\theta}(t)))\right\|.
\end{align*}

\end{proof}

\begin{proof}[\bfseries\color{blue} Proof of Theorem \ref{thm: PRM M>d}]\ \\
From the choosing of $\eta$ (\ref{equ: learning rate PRM}),
we know that Lemma \ref{lemma: quadratic diff along GD trajectory} holds, so $\mathcal{L}(\boldsymbol{\theta}(t))$ is quadratic differential along GD trajectory if $\boldsymbol{\theta}(0)\ne\boldsymbol{0}$. Combining Lemma \ref{lemma: hessian upper bound PRM}, we can construct the quadratic upper bound for $0\leq t\leq T^*$:
\begin{align*}
&\mathcal{L}(\boldsymbol{\theta}(t+1))-\mathcal{L}(\boldsymbol{\theta}(t))-\left<\nabla \mathcal{L}(\boldsymbol{\theta}(t)),\boldsymbol{\theta}(t+1)-\boldsymbol{\theta}(t)\right>
\\
=&\int_{0}^1 \left<\nabla \mathcal{L}\big(\boldsymbol{\theta}(t)+\alpha(\boldsymbol{\theta}(t+1)-\boldsymbol{\theta}(t))\big)-\nabla \mathcal{L}(\boldsymbol{\theta}(t))
,\Theta(t+1)-\Theta(t)\right>{\mathrm d}\alpha
\\
\leq&
\int_{0}^1 \left\|\nabla \mathcal{L}\big(\boldsymbol{\theta}(t)+\alpha(\boldsymbol{\theta}(t+1)-\boldsymbol{\theta}(t))\big)-\nabla \mathcal{L}(\boldsymbol{\theta}(t))
\right\|\left\|\boldsymbol{\theta}(t+1)-\boldsymbol{\theta}(t)\right\|{\mathrm d}\alpha
\\\leq&
\frac{1}{2}\Big(\sup_{\theta\in\mathcal{S}(T^*)}\left\|\nabla^2\mathcal{L}(\boldsymbol{\theta})\right\|\Big)\left\|\boldsymbol{\theta}(t+1)-\boldsymbol{\theta}(t)\right\|^2.
\end{align*}
From the choose of $\eta$ (\ref{equ: learning rate PRM}) and Lemma \ref{lemma: hessian upper bound PRM}, we know the learning rate is satisfied to
\[\eta
\leq\frac{1}{2\sup_{\theta\in\mathcal{S}(T^*)}\left\|\nabla^2\mathcal{L}(\boldsymbol{\theta})\right\|},\]
so we have the following loss descent for $0\leq t\leq T^*$:
\begin{align*}
    &\mathcal{L}(\boldsymbol{\theta}(t+1))
    \\\leq&
    \mathcal{L}(\boldsymbol{\theta}(t))-\frac{\eta(1+\frac{1}{2})}{2}\left\|\nabla \mathcal{L}(\boldsymbol{\theta}(t))\right\|^2
    \\\overset{\rm Lemma\ \ref{lemma: gradient lower bound}}{\leq}&
    \mathcal{L}(\boldsymbol{\theta}(t))-\frac{3\eta}{4}\Big(\frac{d}{2\pi M}\sqrt{\frac{d-1}{d}}-\frac{1}{2}\sum_{j=1}^m\left\|\boldsymbol{w}_j(t)\right\|\Big)^2
    \\\leq&
    \mathcal{L}(\boldsymbol{\theta}(t))-\frac{3\eta m}{4}\Big(\frac{d}{2\pi M}\sqrt{\frac{d-1}{d}}-\frac{1}{2}\sum_{j=1}^m\left\|\boldsymbol{w}_j(0)\right\|-\frac{1}{2}\sum_{j=1}^m\sum_{s=0}^{t-1}\eta\left\|\frac{\partial \mathcal{L}(\boldsymbol{\theta}(s))}{\partial \boldsymbol{w}_j}\right\|\Big)^2
    \\\leq&
    \mathcal{L}(\boldsymbol{\theta}(t))-\frac{3\eta m}{4}\Bigg(\frac{d(1-\pi\kappa)}{2\pi M}\sqrt{\frac{d-1}{d}}-\frac{\eta t m(\pi +1)}{4\pi}\Big(1+\frac{d}{\pi M}\sqrt{\frac{d-1}{d}}\Big)\Bigg)^2
    .
\end{align*}
For convenience, we denote:
\begin{gather*}
    A=\frac{d(1-\pi\kappa)}{2\pi M}\sqrt{\frac{d-1}{d}},\quad B =\frac{\eta  m(\pi +1)}{4\pi}\Big(1+\frac{d}{\pi M}\sqrt{\frac{d-1}{d}}\Big),\quad C=\frac{A}{B}.
\end{gather*}
Then we have the cumulative estimate:
\begin{align*}
 \mathcal{L}(\boldsymbol{\theta}(T^*+1))
\leq& \mathcal{L}(\boldsymbol{\theta}(0))-\frac{3\eta m}{4}\sum_{t=0}^{T^*}\Bigg(\frac{d(1-\pi\kappa)}{2\pi M}\sqrt{\frac{d-1}{d}}-\frac{\eta t m(\pi +1)}{4\pi }\Big(1+\frac{d}{\pi M}\sqrt{\frac{d-1}{d}}\Big)\Bigg)^2
\\=&
 \mathcal{L}(\boldsymbol{\theta}(0))-\frac{3\eta m}{4}\sum_{t=0}^{T^*}\Big(A-B t\Big)^2
 \\=&
  \mathcal{L}(\boldsymbol{\theta}(0))-\frac{3\eta m B^2}{4}\sum_{t=0}^{C}\Big(C- t\Big)^2
 \\\overset{\text{Lemma \ref{lemma: quadratic sum}}}{=}&
 \mathcal{L}(\boldsymbol{\theta}(0))-\frac{\eta m B^2C^3}{4}
 \\=&
  \mathcal{L}(\boldsymbol{\theta}(0))-\frac{\eta mA^3}{4B}
  \\=&
  \mathcal{L}(\boldsymbol{\theta}(0))-\frac{(1-\pi\kappa)^3}{8(\pi+1)\pi^2\Big(1+\frac{d}{\pi M\sqrt{\frac{d-1}{d}}})}\Big(\sqrt{\frac{d-1}{d}}\Big)^3\Big(\frac{d}{M}\Big)^3.
\end{align*}

It is easy to see that:
\[
\mathcal{L}(\boldsymbol{0})=\frac{1}{2}.
\]

From Lemma \ref{lemma: loss homogeneity} and Lemma \ref{lemma: gradient lower bound}, we know that:
\begin{align*}
    \mathcal{L}(\boldsymbol{\theta}(0))-\mathcal{L}(\boldsymbol{0})
    &=\frac{1}{2}\left<\boldsymbol{\theta}(0),\nabla \mathcal{L}(\boldsymbol{\theta}(0))\right>=\frac{1}{2}\sum_{k=1}^m\left<\boldsymbol{w}_i,\frac{\partial \mathcal{L}(\boldsymbol{\theta}(0))}{\partial \boldsymbol{w}_i }\right>
    \\\leq&
    -\Big(\sum_{k=1}^m\left\|\boldsymbol{w}_i\right\|\Big)\Bigg(
    \frac{d}{4\pi M}\sqrt{\frac{d-1}{d}}-\frac{1}{4}\Big(\sum_{j=1}^m\left\|\boldsymbol{w}_j\right\|\Big)\Bigg)
    \\=&-\frac{d^2\kappa(1-\pi\kappa)}{4\pi M^2}\Big(\sqrt{\frac{d-1}{d}}\Big)^2.
\end{align*}
So we have:
\begin{align*}
    &\mathcal{L}(\boldsymbol{0})-\mathcal{L}(\boldsymbol{\theta}(T^*+1))
    \\=&
    \Big(\mathcal{L}(\boldsymbol{0})-\mathcal{L}(\boldsymbol{\theta}(0))\Big)+\Big(\mathcal{L}(\boldsymbol{\theta}(0))-\mathcal{L}(\boldsymbol{\theta}(T^*+1))\Big)
    \\\geq&
    \frac{\kappa(1-\pi\kappa)}{4\pi }\Big(\sqrt{\frac{d-1}{d}}\Big)^2\Big(\frac{d}{M}\Big)^2+\frac{(1-\pi\kappa)^3}{8(\pi+1)\pi^2\Big(1+\frac{d}{\pi M\sqrt{\frac{d-1}{d}}})}\Big(\sqrt{\frac{d-1}{d}}\Big)^3\Big(\frac{d}{M}\Big)^3
    \\=&\Omega\Big(\kappa(1-\kappa)\frac{d^2}{M^2}+(1-\kappa)^3\frac{d^3}{M^3}\Big).
\end{align*}

\end{proof}

\newpage

\section{Some Basic Inequalities}\label{app: inequalities}
In this section, we will state some basic inequalities used in our proof.

\begin{lemma}[Hoeffding's Inequality]\label{lemma: hoeffding} Let $X_1,\cdots,X_n$ are independent random variables, and $X_i\in[a_i,b_i]$ for any $i\in[n]$. Define $\bar{X}=\frac{1}{n}\sum_{i=1}^n X_i$. Then for any $\epsilon>0$, we have the following probability inequalities:
\begin{gather*}
    \mathbb{P}\Big(\bar{X}-\mathbb{E}[\bar{X}]\geq \epsilon\Big)\leq\exp\Big(-\frac{2n^2\epsilon^2}{\sum_{i=1}^n(b_i-a_i)^2}\Big),\\
      \mathbb{P}\Big(\bar{X}-\mathbb{E}[\bar{X}]\leq- \epsilon\Big)\leq\exp\Big(-\frac{2n^2\epsilon^2}{\sum_{i=1}^n(b_i-a_i)^2}\Big).
\end{gather*}
\end{lemma}

\begin{lemma}\label{lemma: basic exp}
For any $x\in(0,1)$, we have:
\begin{gather*}
    e^{\frac{1}{1+x}}\leq (1+x)^{\frac{1}{x}}\leq e,
    \\
    \Big(\frac{1}{e}\Big)^{\frac{1}{1-x}}\leq(1-x)^{\frac{1}{x}}\leq\frac{1}{e}.
\end{gather*}

\end{lemma}

\begin{lemma}\label{lemma: alpha increase} If $0<\eta<\frac{1}{C}$, then we have:
\begin{gather*}
\frac{(1+\alpha\eta)^t-(1-\alpha\eta)^t}{2}\leq\frac{(1+C\eta)^t-(1-C\eta)^t}{2},\ \forall\alpha\in[1,C],\forall t\geq 1;
\\
\frac{(1+\alpha\eta)^t+(1-\alpha\eta)^t}{2}\leq\frac{(1+C\eta)^t+(1-C\eta)^t}{2},\ \forall\alpha\in[1,C],\forall t\geq 1;
\\
\frac{(1+\alpha\eta)^{t-1}-(1-\alpha\eta)^{t-1}}{2}\leq\frac{(1+\alpha\eta)^t-(1-\alpha\eta)^t}{2},\ \forall\alpha\in[1,C],\forall t\geq 1;
\\
\frac{(1+\alpha\eta)^{t-1}+(1-\alpha\eta)^{t-1}}{2}\leq\frac{(1+\alpha\eta)^t+(1-\alpha\eta)^t}{2},\ \forall\alpha\in[1,C],\forall t\geq 1
.
\end{gather*}
\end{lemma}

\begin{lemma}[Jensen Inequation]\label{lemma: jensen inequation} For any convex function $f(\cdot)$, we have:
\begin{gather*}
f\Big(\mathbb{E}[X]\Big)\leq\mathbb{E}\Big[f(X)\Big];\\
f(\frac{1}{n}\sum_{i=1}^n x_i)\leq\frac{1}{n}\sum_{i=1}^n f(x_i).    
\end{gather*}

\end{lemma}

\begin{lemma}\label{lemma: function analysis B} The function
$\arccos (x)$ is decreasing and convex on $[-1,0]$.
\end{lemma}

\begin{lemma}\label{lemma: quadratic sum}
For $n\in\mathbb{N}_{+}$, we have:
\[
\sum_{i=0}^n (n-i)^2=\frac{n(n+1)(2n+1)}{6}\sim \frac{(n+1)^3}{3},\ n\to\infty.
\]
\end{lemma}





\newpage

\end{document}